\documentclass[a4paper,12pt]{article}
\usepackage[utf8]{inputenc}
\usepackage{fullpage}
\usepackage{amsmath, amssymb, amstext, amsfonts, amsthm, bbm, array, enumerate,scalerel}
\usepackage{mathrsfs}%\mathscr schoen, \mathcal.. nicht so schoen
\usepackage{dsfont}
\usepackage{nicefrac}
\usepackage{shuffle}
\usepackage{mathabx}
\usepackage{mathtools}
\usepackage[dvipsnames]{xcolor}
\usepackage[english]{babel}
\usepackage{tikz}
\usepackage{stmaryrd}
\usepackage{subcaption}
\usepackage{tikz-3dplot}
\usepackage{tikz}
\usetikzlibrary{positioning}
\usetikzlibrary{patterns}
\usetikzlibrary{plotmarks}
\usepackage{enumitem}
\usepackage{bbm}

\newtheorem{theorem}{Theorem}[section]
\newtheorem{lemma}[theorem]{Lemma}
\newtheorem{corollary}[theorem]{Corollary}
\newtheorem{proposition}[theorem]{Proposition}

\theoremstyle{definition}
\newtheorem{definition}[theorem]{Definition}
\newtheorem{example}[theorem]{Example}

\theoremstyle{remark}
\newtheorem{remark}[theorem]{Remark}

\numberwithin{equation}{section}
\numberwithin{theorem}{section}

\DeclareMathOperator{\supp}{supp}

\newcommand{\Pol}{\textup{Pol}}

\DeclareMathOperator{\id}{id}
\DeclareMathOperator{\linspan}{span}

\newcommand{\img}{\text{Im}}

\newcommand{\vertiii}[1]{{\left\vert\kern-0.25ex\left\vert\kern-0.25ex\left\vert #1 
    \right\vert\kern-0.25ex\right\vert\kern-0.25ex\right\vert}}

\usepackage{scalerel,stackengine}
\stackMath
\newcommand\reallywidehat[1]{%
\savestack{\tmpbox}{\stretchto{%
  \scaleto{%
    \scalerel*[\widthof{\ensuremath{#1}}]{\kern.1pt\mathchar"0362\kern.1pt}%
    {\rule{0ex}{\textheight}}%WIDTH-LIMITED CIRCUMFLEX
  }{\textheight}% 
}{2.4ex}}%
\stackon[-6.9pt]{#1}{\tmpbox}%
}

\stackMath
\newcommand\reallywidecheck[1]{%
\savestack{\tmpbox}{\stretchto{%
  \scaleto{%
    \scalerel*[\widthof{\ensuremath{#1}}]{\kern-.6pt\bigwedge\kern-.6pt}%
    {\rule[-\textheight/2]{1ex}{\textheight}}%WIDTH-LIMITED BIG WEDGE
  }{\textheight}% 
}{0.5ex}}%
\stackon[1pt]{#1}{\scalebox{-1}{\tmpbox}}%
}

\makeatletter
\newcommand{\labeltext}[3][]{%
	\@bsphack%
	\csname phantomsection\endcsname% in case hyperref is used
	\def\tst{#1}%
	\def\labelmarkup{\emph}% How to markup the label itself
	\def\refmarkup{}%
	\ifx\tst\empty\def\@currentlabel{\refmarkup{#2}}{\label{#3}}%
	\else\def\@currentlabel{\refmarkup{#1}}{\label{#3}}\fi%
	\@esphack%
	\labelmarkup{#2}% visible printed text.
}
\makeatother

\RequirePackage[colorlinks,citecolor=blue,urlcolor=blue, linkcolor=blue]{hyperref}

\newcommand\mydots{\hbox to 1em{.\hss.\hss.}}
\begin{document}
	
	\title{Global universal approximation of functional input maps on weighted spaces}

	\date{}

	\author{Christa Cuchiero\thanks{Vienna University, Department of Statistics and Operations Research, Data Science @ Uni Vienna, Kolingasse 14-16, 1090 Vienna, Austria, \mbox{christa.cuchiero@univie.ac.at}
		} \quad Philipp Schmocker\thanks{ETH Zurich, Department of Mathematics, Rämistrasse 101, 8092 Zurich, Switzerland, \mbox{philipp.schmocker@math.ethz.ch}} \quad  Josef Teichmann\thanks{ETH Zurich, Department of Mathematics, Rämistrasse 101, 8092 Zurich, Switzerland, \mbox{jteichma@math.ethz.ch} \newline
			The first author gratefully acknowledges financial support through grant Y 1235 of the START-program.\newline
			The third author gratefully acknowledges financial support by SNF and ETH Foundation.\newline
			The authors are grateful to the anonymous referees whose comments helped to improve the article.}
}

	\maketitle

	\begin{abstract}
    	We introduce so-called \emph{functional input neural networks} defined on a possibly infinite dimensional weighted space with values also in a possibly infinite dimensional output space. To this end, we use an additive family to map the input weighted space to the hidden layer, on which a non-linear scalar activation function is applied to each neuron, and finally return the output via some linear readouts. Relying on Stone-Weierstrass theorems on weighted spaces, we can prove a \emph{global} universal approximation result on weighted spaces for continuous functions going beyond the usual approximation on compact sets. This then applies in particular to approximation of (non-anticipative) path space functionals via functional input neural networks. As a further application of the weighted Stone-Weierstrass theorem we prove a global universal approximation result for linear functions of the signature. We also introduce the viewpoint of Gaussian process regression in this setting and emphasize that the reproducing kernel Hilbert space of the signature kernels are Cameron-Martin spaces of certain Gaussian processes. This paves a way towards uncertainty quantification for signature kernel regression.
	\end{abstract}
	
	
	\vspace{0.2cm}
	
 
	\tableofcontents
	
	\section{Introduction}
	\label{intro}
	
	We introduce a generalization of neural networks to infinite dimensional spaces which we call \emph{functional input neural networks (FNNs).} These neural networks can be applied as supervised learning tools in machine learning (see \cite{mitchell97,montavon12,gbc16}), when the input and possibly also the output spaces are infinite dimensional.
	
	The study of neural networks on finite-dimensional Euclidean spaces was originally initiated by Warren McCulloch and Walter Pitts in their seminal work \cite{mcculloch43}. The idea was to mimick the functionality of the human brain with a system consisting of various connections and neurons, where data is fed in, further transformed roughly as in neural connections and finally returned  as output. Mathematically, such a system can be described by a concatenation of affine and non-linear maps, where the affine maps represent the connections between the different neurons and the non-linear maps the transformation of the input data. In the well-known universal approximation theorems (UATs), first proved by George Cybenko \cite{cybenko89} and Kurt Hornik et al.~\cite{hornik89,hornik91}, it was then shown that such neural networks can approximate any continuous function arbitrarily well, uniformly on compact subsets of $\mathbb{R}^d$.
	
    One of the goals of the present work consists in replacing compact subsets of $\mathbb{R}^d$ by some infinite dimensional space, in particular spaces of functions or path spaces (being the reason for ``functional input maps'' in the title), and going beyond uniform approximation on compact sets. More precisely, we provide \emph{global UATs} on non locally compact input spaces. This naturally leads to the setting of \emph{weighted spaces}, rigorously introduced in Section \ref{sec:1}. These weighted spaces do not only qualify as infinite dimensional input spaces, but also allow to formulate a probabilistic theory of so-called generalized Feller processes together with their semigroups, as considered in \cite{roeckner06, doersek10, cuchiero20, cuchiero20b, blessing22}. The latter is for instance important for the analysis of approximation properties of neural S(P)DEs (see, e.g., \cite{salvi2022neural, kidger2021neural, cohen2021arbitrage, gierjatowicz2020robust}), or signature SDEs (see, e.g., \cite{arribas2020sig, cuchiero2022signature, primavera22}) as the corresponding paths usually do not stay in compact sets almost surely.

	In order to show the global universal approximation property of FNNs between infinite dimensional spaces, we use the structure of the weighted spaces and rely on a \emph{weighted version of the classical Stone-Weierstrass theorem} (see, e.g., \cite{stone48,buck58}) proved in Section~\ref{SecSW}. Our formulation of this weighted Stone-Weierstrass theorem is inspired by Leopoldo Nachbin's article \cite{nachbin65}, even though our setting differs in several important respects (definition of function spaces, weighted topologies, criteria for density, etc.). The crucial ingredient for global approximation is a weight function that controls the functions to be approximated outside of large compact sets and in turn allows to prove density of \emph{point separating moderate growth algebras} (introduced in Definition~\ref{DefPtSepModGrowth}) among all continuous functions which can be dominated by this weight function. Another Stone-Weierstrass theorem was proved in \cite{buck58,todd65} for the space of continuous bounded functions over a non-compact space using the so-called strict topology (see also \cite{giles71}), which can be seen as the projective limit of the weighted spaces introduced in this paper. 
	
    The weighted Stone-Weierstrass result is then used to get an UAT for FNNs, which consist of an additive family as hidden layer maps and a non-linear activation function applied to each hidden layer. The activation functions are defined on $\mathbb{R}$, allowing us to trace the universal approximation property on infinite dimensional spaces back to the one-dimensional case. By imposing classical conditions on the activation function like the discriminatory property, the sigmoidal form, or conditions on its Fourier transform (similar to the non-polynomial case), we first prove a global UAT for neural networks on $\mathbb{R}$ (see Proposition \ref{PropAct}), which can be lifted to FNNs between infinite dimensional spaces. In particular, Condition (A3) of Proposition \ref{PropAct} stating that the Fourier transform of the tempered distribution induced by the activation function has a support with zero as inner point, is a new formulation yielding another version of UAT for continuous functions on $\mathbb{R}$ and is thus interesting in its own right. Another UAT of neural networks beyond compact subsets of the Euclidean space has been recently established in \cite{vannuland24}.
    
    Let us remark that there is of course an extensive literature on neural networks with infinite dimensional input and output spaces. Early instances can be found in \cite{chen95,mhaskar97b,stinchcombe99,rossi05} for learning non-linear functionals. For more recent results we refer to \cite{kratsios20,kratsios23} for approximation on non-Euclidean spaces, to \cite{benth21} for approximation on Fr\'echet spaces, and to \cite{song23} for approximation rates for non-linear functionals on $L^p$-spaces. Moreover, echo-state network architectures were considered in \cite{grigoryeva18,gonon23} and so-called metric hypertransformers in \cite{acciaio2022metric}, which are both shown to be universal for adapted maps between suitably defined discrete-time path spaces. In addition, for learning the solution of a partial differential equation, we refer to the works on physics-informed neural networks in \cite{raissi19}, neural operators in \cite{stuart21}, DeepONets in \cite{lu21,lanthaler22}, and the references therein.

    The crucial novelty of the current work are global UATs for (generalizations of) continuous functions beyond compact sets. While classical UATs on compact domains ensure the existence of an approximation over a fixed compact set, our results yield a global approximation. These UATs on non-compact domains are highly relevant in areas like stochastic analysis or mathematical finance, where the model space is given by a generically non-compact set of paths. An important class of functional input maps in mathematical finance are so-called non-anticipative functionals introduced in \cite{dupire09,contfournie13}. We translate our general UAT to this setup and illustrate in Section~\ref{SecNE} the numerical performance of FNNs when approximating the running maximum of a standard Brownian motion. 
    
    Apart from neural networks, there are many other families which can serve as universal approximators on function spaces. One well-known example is the \emph{signature of a path} which actually serves as a linear regression basis. It plays a central role in rough path theory (see, e.g., \cite{lyons07, friz10, friz2020course}) and lately also in econometrics and mathematical finance (see, e.g., \cite{levin13, kiraly19, perez19, BHLPW:20, kalsi20, arribas2020sig, cuchiero20sign, NSSBWL:21, BHRS:21, MH:21, AGTZ:22, chevyrev22, cuchiero2022signature, primavera22, neufeld22} and the references therein). Similarly as for neural networks, UATs for linear functions of the signature have been proved on compact sets of paths (see, e.g., \cite{perez19,cuchiero2022signature} for continuous semimartingales and \cite{primavera22} for c\`adl\`ag paths). Note that in the context of (semi-)martingales the Wiener-Ito chaos decomposition with iterated stochastic integrals instead of the signature (see \cite{ito51} for Brownian motion and \cite{ito56} for processes with stationary increments) can be seen as global UAT for the $L^2$-norm with respect to the Wiener measure (see also the $L^p$-version with (random) neural networks of \cite{neufeld22}).
    
    By choosing an appropriate weight function and by applying the weighted Stone-Weierstrass theorem we obtain a global universal approximation result for linear functions of the signature of continuous geometric rough paths (see Section \ref{SecSM}). Let us remark that in \cite{chevyrev22} also a global approximation result is obtained, however not with the ``true'' signature but with a bounded normalization of it, as there the strict topology from \cite{giles71} on continuous \emph{bounded} functions (see \cite[Definition 8]{chevyrev22}) is used. The disadvantage of this normalization is that the tractability properties of the signature, like the analytic computation of its expected value, are lost. In Section \ref{SecKernels} we also introduce the viewpoint of Gaussian process regression in this setting and emphasize that the reproducing kernel Hilbert space of the ``true'' signature kernels are Cameron-Martin spaces of certain Gaussian processes, which are first time introduced in the literature. This can be useful in the context of uncertainty quantification for signature kernel regression.
    
    In the following subsection we provide three specific examples of FNNs that can be used to approximate $\alpha$-Hölder continuous functions, non-anticipative path functionals as well as monetary risk measures.

	\subsection{Motivational examples}
	\label{SubsecMotEx}
	
	In order to give  an overview of the applicability of our results, we consider three particular examples. The first one is within functional data analysis, which is a branch of statistics where every sample corresponds to a continuous function (see \cite{ramsay05}). Hence, we learn a continuous map $f: C^\alpha(S) \rightarrow Y$ with some growth conditions, where $(S,d_S)$ is a compact metric space, $C^\alpha(S)$ denotes the space of $\alpha$-H\"older continuous real-valued functions on $(S,d_S)$, and $(Y,\Vert \cdot \Vert_Y)$ is a Banach space. Then, our results show that the map $f: C^\alpha(S) \rightarrow Y$ can be approximated by learning a FNN $\varphi: C^\alpha(S) \rightarrow Y$ of the form
	\begin{equation*}
		\varphi(x) = \sum_{n=1}^N y_n \rho\left( \int_S x(s) \nu_n(ds) + b_n \right),
	\end{equation*}
	for $x \in C^\alpha(S)$, where $y_1,\ldots,y_N \in Y$ and $b_1,\ldots,b_N \in \mathbb{R}$, and where $N \in \mathbb{N}$ is the number of neurons. Moreover, $\rho: \mathbb{R} \rightarrow \mathbb{R}$ is an activation function and $\nu_1,\ldots,\nu_N$ are finite signed Radon measures defined on $S$. If $S$ is, e.g.,~a subset of the Euclidean space, the Radon measures can be replaced by classical neural networks playing the role of a density function. This together with a numerical integration allows us to implement the functional neural network on a computer.
	
	In the second example, we learn a non-anticipative functional $f(t,x)$ whose value at time $t \in [0,T]$ only depends on the path $x \in C^\alpha([0,T];\mathbb{R}^d)$ up to time $t \in [0,T]$ (see \cite{dupire09,cont10,contfournie13}). For this purpose, we define the \emph{stopped path} $x^t$ of $x$ at time $t \in [0,T]$ as $x^t(s) = x(\min(s,t))$, for $s \in [0,T]$, and consider the \emph{space of stopped $\alpha$-H\"older continuous paths} $\Lambda^\alpha_T = \left\lbrace (t,x^t): (t,x) \in [0,T] \times C^\alpha([0,T];\mathbb{R}^d) \right\rbrace$, which is equipped with the metric $d_\infty((t,x),(s,y)) = \vert t-s \vert + \sup_{u \in [0,T]} \left\Vert x^t(u) - y^s(u) \right\Vert$. Measurable maps from $\Lambda^\alpha_T$ to $\mathbb{R}$ are then called non-anticipative functionals. If they are additionally continuous and satisfy certain growth conditions, we can approximate them via FNNs $\varphi: \Lambda^\alpha_T \rightarrow \mathbb{R}$ of the form
	\begin{equation*}
		\varphi(t,x) = \sum_{n=1}^N y_n \rho\left( a_n t + \sum_{i=1}^d \int_0^t \phi_{n,i}(s) x_i(s) ds + b_n \right),
	\end{equation*}
	for $(t,x) \in \Lambda^\alpha_T$, where $y_1,\ldots,y_N, a_1,\ldots,a_N, b_1,\ldots,b_N \in \mathbb{R}$, where $N \in \mathbb{N}$ is the number of neurons, where $\rho: \mathbb{R} \rightarrow \mathbb{R}$ denotes an activation function, and where $\phi_{n,i}: \mathbb{R} \rightarrow \mathbb{R}$ are classical neural networks. Compared to the first example, the classical neural networks act here as density functions instead of Radon measures.
	
	As third example we consider monetary risk measures in the sense of \cite[Section 4]{follmer2016stochastic}, i.e.~maps $R: L^{\infty}(\Omega, \mathcal{F}, \mathbb{P}) \to \mathbb{R}$, where $(\Omega, \mathcal{F}, \mathbb{P})$ denotes a probability space. A monetary risk measure is defined to be monotone, i.e.~$R(x) \leq R(\widetilde{x})$ for any $x, \widetilde{x} \in L^\infty(\Omega, \mathcal{F}, \mathbb{P})$ with $x \geq \widetilde{x}$, and cash-invariant, i.e.~$R(x+m)=R(x)-m$ for any $x \in L^{\infty}(\Omega, \mathcal{F}, \mathbb{P})$ and $m \in \mathbb{R}$. We define such a risk measure via an acceptance set $\mathcal{A} \subseteq L^{\infty}(\Omega, \mathcal{F}, \mathbb{P})$ such that
    \begin{equation*}
        R(x) := \inf\left\lbrace m \in \mathbb{R}: x+m \in \mathcal{A} \right\rbrace,
    \end{equation*}
    with $\inf\emptyset := \infty$, where some assumptions on $\mathcal{A}$ are needed to ensure that $R$ is monotone. We can now approximate the indicator function $\mathds{1}_\mathcal{A}: L^{\infty}(\Omega, \mathcal{F}, \mathbb{P}) \to \{0,1\}$ (or smooth approximation thereof) via a FNN $\varphi: L^{\infty}(\Omega, \mathcal{F}, \mathbb{P}) \to \mathbb{R}$ of the form
    \begin{equation*}
        \varphi(x)=\sum_{n=1}^N y_n \rho\left( \mathbb{E}[xz_n] + b_n \right),
    \end{equation*}
    for $x \in L^{\infty}(\Omega,\mathcal{F},\mathbb{P})$, where $y_1, \ldots, y_N, b_1, \ldots b_N \in \mathbb{R}$, $N \in \mathbb{N}$, $\rho: \mathbb{R} \to \mathbb{R}$, and $z_1, \ldots z_N \in L^{1}(\Omega, \mathcal{F}, \mathbb{P})$. An approximation of the risk measure $R$ can then be computed by solving
    \begin{equation*}
        \inf\left\lbrace m \in \mathbb{R}: \varphi(x+m) = 1 \right\rbrace,
    \end{equation*}
	where $\inf\emptyset := \infty$. These three examples provide a first glance over the broad range of applications for FNNs. Note that the input space can violate the vector space structure, as illustrated by the second example, and can be the dual of another Banach space, as in the third example. \\
	
	The remainder of the article is structured as follows. In the next subsection we introduce relevant notation used throughout the article. Section~\ref{sec:1} is dedicated to notions related to weighted spaces and the generalization of continuous functions, called $\mathcal{B}_{\psi}$-functions, defined thereon. In Section~\ref{SecSW}, we prove the weighted Stone-Weierstrass theorem in this setting, which we then use in Section~\ref{SecUATWS} 
	to lift the global universal approximation property of classical neural networks to FNNs.
	In Section~\ref{SecSM}, we present a further application of the weighted Stone-Weierstrass theorem to prove a global UAT for linear functions of path signatures. Section \ref{SecKernels} then introduces a Gaussian process regression point of view and identifies the reproducing kernel Hilbert space of the signature kernel with the Cameron-Martin space of Gaussian processes taking values in $\mathcal{B}_{\psi}$-functions.
	Finally, we provide some numerical examples in Section~\ref{SecNE}.
	
	\subsection{Notation}
	\label{SecNotation}
	
	As usual, $\mathbb{N} := \lbrace 1, 2, 3, \ldots \rbrace$ and $\mathbb{N}_0 := \mathbb{N} \cup \lbrace 0 \rbrace$ denote the sets of natural numbers, while $\mathbb{R}$ and $\mathbb{C}$ (with imaginary unit $\mathbf{i} := \sqrt{-1} \in \mathbb{C}$) represent the sets of real and complex numbers, respectively. Furthermore, we define the functions $\lfloor r \rfloor := \max\left\lbrace n \in \mathbb{Z}: r \geq n \right\rbrace$ and $\lceil r \rceil := \min\left\lbrace n \in \mathbb{Z}: r \leq n \right\rbrace$, for $r \in \mathbb{R}$. Moreover, for $d \in \mathbb{N}$, we denote by $\mathbb{R}^d$ the $d$-dimensional Euclidean space equipped with the norm $\Vert x \Vert = \sqrt{\sum_{i=1}^d x_i^2}$, for $x \in \mathbb{R}^d$. In addition, $\mathcal{F}_X$ denotes the Borel $\sigma$-algebra of a topological space $(X,\tau_X)$.
	
	Furthermore, for a Hausdorff topological space $(X,\tau_X)$ and a Banach space $(Y,\Vert \cdot \Vert_Y)$, we denote by $C^0(X;Y)$ the vector space of continuous maps $f: X \rightarrow Y$, whereas $C^0_b(X;Y) \subseteq C^0(X;Y)$ denotes the vector subspace of bounded maps. If $(X,\tau_X)$ is additionally compact, every map $f \in C^0(X;Y)$ is bounded. In this case, the supremum norm $\Vert f \Vert_{C^0(X;Y)} := \sup_{x \in X} \Vert f(x) \Vert_Y$ turns $C^0(X;Y)$ into a Banach space.

    Moreover, for an open subset $U \subseteq \mathbb{R}$, we denote by $C^\infty_c(U;\mathbb{C})$ the vector space of compactly supported smooth functions $g: U \rightarrow \mathbb{C}$ with support $\supp(g) := \overline{\left\lbrace s \in U: g(s) \neq 0 \right\rbrace}$ contained in $U$. In addition, $\mathcal{S}(\mathbb{R};\mathbb{C})$ represents the Schwartz space of smooth functions $g: \mathbb{R} \rightarrow \mathbb{C}$ such that the seminorms $\max_{j=0,\ldots,n} \sup_{s \in \mathbb{R}} \left( 1+s^2 \right)^n \left\vert g^{(j)}(s) \right\vert$ are finite, for all $n \in \mathbb{N}_0$, which is equipped with the locally convex topology induced by these seminorms. Then, its dual space $\mathcal{S}'(\mathbb{R};\mathbb{C})$ consists of continuous linear functionals $T: \mathcal{S}(\mathbb{R};\mathbb{C}) \rightarrow \mathbb{C}$ called tempered distributions. For example, $\rho \in C^0(\mathbb{R};\mathbb{C})$ with $\sup_{z \in \mathbb{R}} \frac{\vert \rho(z) \vert}{(1+z^2)^n} < \infty$, for some $n \in \mathbb{N}_0$, induces the tempered distribution $\left( g \mapsto T_\rho(g) := \int_{\mathbb{R}} \rho(s) g(s) ds \right) \in \mathcal{S}'(\mathbb{R};\mathbb{C})$. Moreover, the support of any $T \in \mathcal{S}'(\mathbb{R};\mathbb{C})$ is defined as the complement of the largest open set $U \subseteq \mathbb{R}$ on which $T \in \mathcal{S}'(\mathbb{R};\mathbb{C})$ vanishes, i.e.~$T(g) = 0$ for all $g \in C^\infty_c(U;\mathbb{C})$. Furthermore, we define the Fourier transform of any $g \in L^1(\mathbb{R};\mathbb{C})$ as $\widehat{g}(\xi) = \int_{\mathbb{R}} e^{-i \xi s} g(s) ds$, for $\xi \in \mathbb{R}$, whereas the Fourier transform $\widehat{T} \in \mathcal{S}'(\mathbb{R};\mathbb{C})$ of any $T \in \mathcal{S}'(\mathbb{R};\mathbb{C})$ is defined by $\widehat{T}(g) = T(\widehat{g})$, for $g \in \mathcal{S}(\mathbb{R};\mathbb{C})$. For more details, we refer to \cite{folland92}.
	
	In addition, if $(X,\Vert \cdot \Vert_X)$ and $(Y,\Vert \cdot \Vert_Y)$ are normed vector spaces, we denote by $L(X;Y)$ the vector space of continuous linear maps $T: X \rightarrow Y$, which is equipped with the norm $\Vert T \Vert_{L(X;Y)} := \sup_{x \in X, \, \Vert x \Vert_X \leq 1} \Vert T(x) \Vert_Y$. If $Y = \mathbb{R}$, the space $X^* := L(X;\mathbb{R})$ is the dual space of $X$ consisting of continuous linear functionals $l: X \rightarrow \mathbb{R}$, which is endowed with the norm $\Vert l \Vert_{X^*} := \sup_{x \in X, \, \Vert x \Vert_X \leq 1} \vert l(x) \vert$. In this case, the dual pairing $(X^*,X) \ni (l,x) \mapsto \langle l,x \rangle_{X^* \times X} := l(x) \in \mathbb{R}$ is bilinear and continuous. Moreover, if $(X,\Vert \cdot \Vert_X)$ and $(Y,\Vert \cdot \Vert_Y)$ are Banach spaces, then $L(X;Y)$ and $X^*$ are also Banach spaces.
 
    Furthermore, a Banach space $(X,\Vert \cdot \Vert_X)$ is called a \emph{dual Banach space} if there exists a Banach space $(E,\Vert \cdot \Vert_E)$ such that $E^* \cong X$ are isometrically isomorphic. In this case, $(E,\Vert \cdot \Vert_E)$ is called a \emph{predual} for $(X,\Vert \cdot \Vert_X)$ and the isomorphism $\Phi: X \rightarrow E^*$ can be used to define the continuous dual pairing $(X,\Vert \cdot \Vert_X) \times (E,\Vert \cdot \Vert_E) \ni (x,e) \mapsto \langle x, e \rangle_{X \times E} := \Phi(x)(e) \in \mathbb{R}$. Hence, we can equip $X$ with a weak-$*$-topology generated by sets of the form $\left\lbrace x \in X: \langle x, e \rangle_{X \times E} \in U \right\rbrace$, for $e \in E$ and $U \subseteq \mathbb{R}$ open (see Appendix~\ref{AppBanachPredual}).

    Moreover, if $(S,d_S)$ is a metric space and $(Z,\Vert \cdot \Vert_Z)$ is a Banach space, a subset $A \subseteq C^0(S;Z)$ is called \emph{equicontinuous} if for every $\varepsilon > 0$ there exists some $\delta > 0$ such that for every $f \in A$ and $s,t \in S$ with $d_S(s,t) < \delta$ it holds that $\Vert f(s) - f(t) \Vert_Z < \varepsilon$. In addition, a subset $A \subseteq C^0(S;Z)$ is called \emph{pointwise bounded} (resp.~\emph{pointwise compact}) if for every $s \in S$ the set $\lbrace f(s): f \in A \rbrace$ is bounded (resp.~compact) in $Z$.

    One of our most important examples for a weighted space $X$, as introduced in Section~\ref{sec:1} below, are H\"older spaces defined as follows. For some $\alpha \in (0,1]$, a compact metric space $(S,d_S)$ with designated origin $0 \in S$, and a dual Banach space $(Z,\Vert \cdot \Vert_Z)$, we denote by $C^\alpha(S;Z)$ the space of $\alpha$-H\"older continuous functions $x: (S,d_S) \rightarrow (Z,\Vert \cdot \Vert_Z)$ such that
	\begin{equation*}
	    \Vert x \Vert_\alpha := \Vert x(0) \Vert_Z + \vert x \vert_\alpha < \infty,
	\end{equation*}
    where $\vert \cdot \vert_\alpha$ denotes the $\alpha$-H\"older seminorm of $x: S \rightarrow Z$ defined as
    \begin{equation*}
        \vert x \vert_\alpha := \sup_{s,t \in S \atop s \neq t} \frac{\Vert x(s) - x(t) \Vert_Z}{d_S(s,t)^\alpha}.
    \end{equation*}
    Then, the norm $\Vert \cdot \Vert_\alpha$ turns $C^\alpha(S;Z)$ into a Banach space (see \cite[Theorem~5.25~(ii)]{friz10} and \cite[Proposition~2.3(b)]{weaver99}). For $\alpha=1$, we can relate $C^\alpha(S;Z)$ to the notion of globally Lipschitz continuous functions $x: S \rightarrow Z$ considered in \cite{weaver99}. Moreover, in order to equip $C^\alpha(S;Z)$ also with weaker topologies, we assume that $(V,\Vert \cdot \Vert_V)$ is a predual for $(Z,\Vert \cdot \Vert_Z)$. Then, for $\alpha' \in [0,\alpha]$, we equip $C^\alpha(S;Z)$ with the $C^{\alpha'}$-topology $\tau_{\alpha'}$ generated by seminorms of the form
    \begin{equation*}
        \Vert x \Vert_{\alpha',v} := \vert \langle x(0), v \rangle_{Z \times V} \vert + \vert x \vert_{\alpha',v}
    \end{equation*}
    for $v \in V$, where $\vert x \vert_{\alpha',v}$ denotes the $(\alpha',v)$-H\"older seminorm of $x: S \rightarrow Z$ defined as
    \begin{equation*}
        \vert x \vert_{\alpha',v} := \sup_{s,t \in S \atop s \neq t} \frac{\vert \langle x(s) - x(t), v \rangle_{Z \times V} \vert}{d_S(s,t)^{\alpha'}}.
    \end{equation*}
    Hence, $(C^\alpha(S;Z),\tau_{\alpha'})$ forms a locally convex topological vector space. Note that for $\alpha' = 0$ the $C^0$-topology $\tau_0$ is equivalent to the $w^*$-uniform topology $\tau_\infty$ generated by seminorms of the form $\Vert x \Vert_{\infty,v} := \sup_{t \in S} \vert \langle x(t), v \rangle_{Z \times V} \vert$, for $v \in V$ (see Lemma~\ref{LemmaHoelderTop}~\ref{LemmaHoelderTop1}). Moreover, $C^\alpha(S;Z)$ is by Theorem~\ref{ThmHoelderPredual} a dual Banach space and can therefore be equipped with a weak-$*$-topology $\tau_{w^*}$. In addition, $C^\alpha_0(S;Z) \subseteq C^\alpha(S;Z)$ denotes the vector subspace of $\alpha$-H\"older continuous functions $x \in C^\alpha(S;Z)$ preserving the origin, i.e.~$x(0) = 0 \in Z$.
    
    The second important example are spaces of finite $p$-variation. For some $T > 0$, $p \in [1,\infty)$, and a dual Banach space $(Z,\Vert \cdot \Vert_Z)$, we denote by $C^{p-var}([0,T];Z)$ the space of continuous paths $x: [0,T] \rightarrow (Z,\Vert \cdot \Vert_Z)$ such that
    \begin{equation*}
        \Vert x \Vert_{p-var} := \Vert x(0) \Vert_Z + \vert x \vert_{p-var} < \infty.
    \end{equation*}
    Hereby, $\vert x \vert_{p-var}$ denotes the $p$-variation of $x: [0,T] \rightarrow Z$ defined as
    \begin{equation*}
        \vert x \vert_{p-var} := \left( \sup_{(t_i)_i \in \mathcal{D}([0,T])} \sum_i \Vert x(t_{i+1}) - x(t_i) \Vert_Z^p \right)^\frac{1}{p},
    \end{equation*}
    where $\mathcal{D}([0,T])$ is the set of all partitions of the form $0 = t_0 < \ldots < t_n = T$, with $n \in \mathbb{N}$. Then, the norm $\Vert \cdot \Vert_{p-var}$ turns $C^{p-var}([0,T];Z)$ into a Banach space (see \cite[Theorem~5.25~(i)]{friz10}). In order to equip $C^{p-var}([0,T];Z)$ also with weaker topologies, we assume that $(V,\Vert \cdot \Vert_V)$ is a predual for $(Z,\Vert \cdot \Vert_Z)$. Then, for $p' \in [p,\infty]$, we equip $C^{p-var}([0,T];Z)$ with the $C^{p'-var}$-topology $\tau_{p'-var}$ generated by seminorms of the form
    \begin{equation*}
        \Vert x \Vert_{p'-var,v} := \vert \langle x(0), v \rangle_{Z \times V} \vert + \vert x \vert_{p'-var,v}
    \end{equation*}
    for $v \in V$, where $\vert \cdot \vert_{p'-var,v}$ denotes the $(p',v)$-variation of $x: [0,T] \rightarrow Z$ defined as
    \begin{equation*}
        \vert x \vert_{p'-var,v} := \left( \sup_{(t_i)_i \in \mathcal{D}([0,T])} \sum_i \vert \langle x(t_{i+1}) - x(t_i), v \rangle_{Z \times V} \vert^p \right)^\frac{1}{p},
    \end{equation*}
    and where the $C^{\infty-var}$-topology $\tau_{\infty-var}$ is defined as the $w^*$-uniform topology $\tau_\infty$ generated by seminorms of the form $\Vert x \Vert_{\infty,v} := \sup_{t \in [0,T]} \vert \langle x(t), v \rangle_{Z \times V} \vert$, for $v \in V$. Hence, $(C^{p-var}([0,T];Z),\tau_{p'-var})$ forms a locally convex topological vector space. Moreover, we denote by $C^{p-var}_0([0,T];Z) \subseteq C^{p-var}([0,T];Z)$ the vector subspace of finite $p$-variation paths $x \in C^{p-var}([0,T];Z)$ preserving the origin, i.e.~$x(0) = 0 \in Z$.

    We shall also need to intersect spaces of finite $p$-variation with H\"older spaces. For some $T > 0$, $(p,\alpha) \in [1,\infty) \times (0,1)$ with $p \alpha < 1$, and a dual Banach space $(Z,\Vert \cdot \Vert_Z)$, we define the intersection $C^{p-var,\alpha}([0,T];Z) := C^{p-var}([0,T];Z) \cap C^\alpha([0,T];Z)$ that consists of $\alpha$-H\"older continuous paths $x: [0,T] \rightarrow (Z,\Vert \cdot \Vert_Z)$ with finite $p$-variation. For example, the sample paths of Brownian motion lie in $C^{p-var,\alpha}([0,T];\mathbb{R})$ for $p > 2$ and $\alpha \in (0,1/2)$. Moreover, note that the condition $p \alpha < 1$ ensures that $C^{p-var,\alpha}([0,T];Z) \subsetneq C^\alpha([0,T];Z)$ is a proper vector subspace. Otherwise, if $p \alpha \geq 1$, we have $C^\alpha([0,T];Z) \subseteq C^{(1/\alpha)-var}([0,T];Z) \subseteq C^{p-var}([0,T];Z)$ (see \cite[p.~77 \& Proposition~5.3]{friz10}) and thus $C^{p-var,\alpha}([0,T];Z) = C^{p-var}([0,T];Z) \cap C^\alpha([0,T];Z) = C^\alpha([0,T];Z)$. Then, we endow $C^{p-var,\alpha}([0,T];Z)$ with the norm
    \begin{equation*}
        \Vert x \Vert_{p-var,\alpha} := \Vert x(0) \Vert_Z + \vert x \vert_{p-var} + \vert x \vert_\alpha.
    \end{equation*}
    In addition, for $p \in [1,\infty)$ and $\alpha \in [0,\infty)$, we define $C^{p-var,0}([0,T];Z) := C^{p-var}([0,T];Z)$ and $C^{\infty-var,\alpha}([0,T];Z) := C^\alpha([0,T];Z)$, respectively. In order to equip $C^{p-var,\alpha}([0,T];Z)$ also with weaker topologies, we assume that $(V,\Vert \cdot \Vert_V)$ is a predual for $(Z,\Vert \cdot \Vert_Z)$. Then, for $(p',\alpha') \in (p,\infty) \times [0,\alpha)$ with $p' \alpha' < 1$, we equip $C^{p-var,\alpha}([0,T];Z)$ with the $C^{p'-var,\alpha'}$-topology $\tau_{p'-var,\alpha'}$ generated by seminorms of the form
    \begin{equation*}
        \Vert x \Vert_{p'-var,\alpha',v} := \vert \langle x(0), v \rangle_{Z \times V} \vert + \vert x \vert_{p'-var,v} + \vert x \vert_{\alpha',v},
    \end{equation*}
    for $v \in V$. Hence, $(C^{p-var,\alpha}([0,T];Z),\tau_{p'-var,\alpha'})$ defines a locally convex topological vector space. For $p' = \infty$ and $\alpha' \in [0,1)$, we define $\tau_{\infty-var,\alpha'} := \tau_{\alpha'}$ that coincides for $\alpha' = 0$ with $\tau_\infty$ (see Lemma~\ref{LemmaHoelderTop}~\ref{LemmaHoelderTop1}). Moreover, for $p' \in (p,\infty)$, we observe that $\tau_{p'-var,0}$ is identical to $\tau_{p'-var}$ (see Lemma~\ref{LemmaPVarTop}~\ref{LemmaPVarTop1}). In addition, $C^{p-var,\alpha}([0,T];Z)$ is by Theorem~\ref{ThmPVarPredual} a dual Banach space and can therefore be equipped with a weak-$*$-topology $\tau_{w^*}$. Furthermore, we denote by $C^{p-var,\alpha}_0([0,T];Z) \subseteq C^{p-var,\alpha}([0,T];Z)$ the vector subspace of paths $x \in C^{p-var,\alpha}([0,T];Z)$ preserving the origin, i.e.~$x(0) = 0 \in Z$.
    
    Furthermore, if the functions are real-valued, we use the abbreviations $C^0(S) := C^0(S;\mathbb{R})$, $C^\alpha(S) := C^\alpha(S;\mathbb{R})$, $C^{p-var}([0,T]) := C^{p-var}([0,T];\mathbb{R})$, etc.	
	
	\section{Weighted spaces and functions thereon}
	\label{sec:1}
	
	For the universal approximation results as well as the weighted Stone-Weierstrass theorems on infinite dimensional spaces, we use a weighted space as input space and a Banach space as output space. On the weighted input space, one can in turn introduce a weighted function space, which was under (slightly) different conditions also studied in \cite{bernstein24,nachbin65,summers68,prolla71b,bierstedt73,prolla77,mhaskar97,triebel06}. 
 
    Our weighted setting is in particular inspired by \cite{roeckner06,doersek10,cuchiero20b,cuchiero20}, where weighted spaces have been used for Kolmogorov equations, splitting schemes of (stochastic) partial differential equations, and generalized Feller processes. To make the current article self-contained, we shall recall all necessary definitions and provide several examples.
	
    \subsection{Definition and examples of weighted spaces}
	
	In order to define a weighted space we shall assume that $(X,\tau_X)$ is a completely regular Hausdorff topological space (i.e., a Tychonoff space, resp.~$T_{3.5}$-space, see \cite[p.~117]{kelley75}).
	
	\begin{definition}
		\label{DefWeight}
		A function $\psi: X \rightarrow (0,\infty)$ is called an \emph{admissible weight function (on $(X,\tau_X)$)} if every pre-image $K_R := \psi^{-1}((0,R]) = \left\lbrace x \in X: \psi(x) \leq R \right\rbrace$ is compact with respect to $\tau_X$, for all $R > 0$. In this case, we call the pair $(X,\psi)$ a \emph{weighted space}.
	\end{definition}
	
	\begin{remark}
	The definition of a weighted space has the following important consequences.
	    \begin{enumerate}
            \label{RemWeight}
            \item\label{RemWeight1} If $(X,\psi)$ is a weighted space, then $(X,\tau_X)$ is $\sigma$-compact (as $X = \bigcup_{R \in \mathbb{N}} K_R$) and $\psi: X \rightarrow (0,\infty)$ is necessarily lower semicontinuous (as the sublevel sets $K_R := \psi^{-1}((0,R])$ are closed) as well as bounded from below by a strictly positive constant $C > 0$. Indeed, assume by contradiction that such a constant $C > 0$ does not exist. Then, for every $n \in \mathbb{N}$, the pre-image $K_{1/n} := \psi^{-1}((0,1/n])$ is non-empty and it holds that $K_{1/(n+1)} \subseteq K_{1/n}$. Hence, by Cantor's intersection theorem, it follows that $\bigcap_{n \in \mathbb{N}} K_{1/n} \neq \emptyset$, which contradicts the assumption that $\psi(x) > 0$ for all $x \in X$.
            
        	\item\label{RemWeight2} If $(X,\tau_X)$ is a separable locally convex topological vector space and $\psi: X \rightarrow (0,\infty)$ is convex, then $\psi: X \rightarrow (0,\infty)$ is continuous on convex subsets $E \subseteq X$ if and only if $E$ is locally compact (see \cite[Remark~2.2]{cuchiero20}).
         
        	\item\label{RemWeight3} Note that $(X,\tau_X)$ is a Banach space only if $X$ is finite-dimensional. Indeed, if $X$ were an infinite-dimensional Banach space, the compact pre-images $K_R := \psi^{-1}((0,R])$ would have empty interior. Moreover, since the complete pseudo-metrizable space $(X,\psi)$ is by Baire's category theorem a Baire space, any countable union of closed sets with empty interior also has empty interior. However, the countable union $X = \bigcup_{R \in \mathbb{N}} K_R$ has interior $X$, which yields a contradiction.
         
        	\item\label{RemWeight4} One of our main examples for a weighted space (see Example \ref{ExWeightedSpaces}~\ref{ExWeightedSpaceDual} below) are spaces $X$ which are the dual of some Banach space, equipped with a weak-$*$-topology. Again, by Baire's category theorem, these spaces are metrizable if and only if they are finite dimensional. Moreover, in infinite dimensions completeness with respect to the weak-$*$-topology also fails, since the completion would consist of \emph{all} linear functions and not only of continuous ones.
    	\end{enumerate}
	\end{remark}
	
	In the following, we present some examples of weighted spaces $(X,\psi)$, where the compactness of the pre-images $K_R = \psi^{-1}((0,R])$ needs to be verified with a suitable criterion (e.g.~the Banach-Alaoglu theorem or the Arzel\`a-Ascoli theorem).
	
	\begin{example}
		\label{ExWeightedSpaces}
        The examples of admissible weight functions $\psi: X \rightarrow (0,\infty)$ that we present in the following are all of the form $\psi(x) = \eta\left( \Vert x \Vert \right)$ for a norm $\Vert \cdot \Vert$ on $X$ and a continuous, non-decreasing, and unbounded function $\eta: [0,\infty) \rightarrow (0,\infty)$. In view of the so-called \emph{moderate growth} condition defined below (see Definition~\ref{DefPtSepModGrowth}), we consider in particular the function $\eta(r) = \exp\left( \beta r^\gamma \right)$, for some $\beta > 0$ and $\gamma \geq 1$.
		
		\begin{enumerate}
			\item\label{ExWeightedSpaceRd} For $X = \mathbb{R}^d$ equipped with the norm $\Vert x \Vert = \big( \sum_{i=1}^d x_i^2 \big)^{1/2}$, we choose the weight function $\psi(x) = \eta\left( \Vert x \Vert \right)$. Then, the pre-image $K_R = \psi^{-1}((0,R])$ is closed and bounded, thus compact in $\mathbb{R}^d$, for all $R > 0$. Hence, $(\mathbb{R}^d,\psi)$ is a weighted space.
			
			\item\label{ExWeightedSpaceDual} If $(X,\Vert \cdot \Vert_X)$ is a dual Banach space equipped with the weak-$*$-topology, we choose the weight function $\psi(x) = \eta\left( \Vert x \Vert_X \right)$. Then, the pre-image $K_R = \psi^{-1}((0,R])$ is by the Banach-Alaoglu theorem compact with respect to the weak-$*$-topology, for all $R > 0$. This shows that $(X,\psi)$ is a weighted space.
			
			\item\label{ExWeightedSpaceHoelderPredual} For $\alpha > 0$, a compact metric space $(S,d_S)$, and a dual Banach space $(Z,\Vert \cdot \Vert_Z)$, let $X = C^\alpha(S;Z)$ be the space of $\alpha$-H\"older continuous functions $x: S \rightarrow Z$ introduced in Section~\ref{SecNotation}. Then, $C^\alpha(S;Z)$ is by Theorem~\ref{ThmHoelderPredual} again a dual Banach space and can be equipped with the corresponding weak-$*$-topology. In this case, $\psi(x) = \eta\left( \Vert x \Vert_\alpha \right)$ is by the Banach-Alaoglu theorem admissible, which shows that $(C^\alpha(S;Z),\psi)$ is a weighted space. If $(S,d_S)$ is additionally pointed, the same reasoning applies to the space of $\alpha$-H\"older continuous functions preserving the origin $C^\alpha_0(S;Z)$ (see Remark~\ref{RemHoelder0}~\ref{RemHoelder0Predual}).
			
			\item\label{ExWeightedSpaceHoelderWeaker} Opposed to \ref{ExWeightedSpaceHoelderPredual}, let $X = C^\alpha(S;Z)$ be the space of $\alpha$-H\"older continuous functions $x: S \rightarrow Z$, but now equipped with the $w^*$-uniform topology $\tau_\infty$. Then, we choose $\psi(x) = \eta\left( \Vert x \Vert_\alpha \right)$ and observe that the pre-image $K_R = \psi^{-1}((0,R])$ is closed with respect to $\tau_\infty$. Hence, by using Theorem~\ref{ThmHoelderEmbedding} together with Lemma~\ref{LemmaHoelderTop}~\ref{LemmaHoelderTop1}, we conclude that $K_R = \psi^{-1}((0,R])$ is compact with respect to $\tau_\infty$, for all $R > 0$, which shows that $(C^\alpha(S;Z),\psi)$ is a weighted space. Note that instead of the $w^*$-uniform topology $\tau_\infty$ we could also use any $C^{\alpha'}$-topology $\tau_{\alpha'}$, $\alpha' \in [0,\alpha)$, due to the compact embedding $(C^\alpha(S;Z),\Vert \cdot \Vert_\alpha) \hookrightarrow (C^{\alpha'}(S;Z),\tau_{\alpha'})$ in Theorem~\ref{ThmHoelderEmbedding}.
			
			\item\label{ExWeightedSpaceHoelderStoppedPaths} For $\alpha > 0$ and $T > 0$, let $X = \Lambda^\alpha_T$ be the space of stopped $\alpha$-H\"older continuous paths (as introduced in Section~\ref{SubsecMotEx}), i.e.
			\begin{equation*}
				\Lambda^\alpha_T = \left\lbrace (t,x^t): t \in [0,T], \, x \in C^\alpha([0,T];\mathbb{R}^d) \right\rbrace \cong \left( [0,T] \times C^\alpha([0,T];\mathbb{R}^d) \right) / \sim,
			\end{equation*}
			where $(t,x) \sim (s,y)$ if and only if $t = s$ and $x^t = y^s$. We equip $\Lambda^\alpha_T$ with the metric $d_\infty((t,x),(s,y)) = \vert t-s \vert + \sup_{u \in [0,T]} \left\Vert x^t(u) - y^s(u) \right\Vert$ and define the weight function $\psi(t,x) = \eta\left( \Vert x \Vert_{C^\alpha([0,T];\mathbb{R}^d)} \right)$. Since $A_R := \left\lbrace x \in C^\alpha([0,T];\mathbb{R}^d): (t,x) \in K_R \right\rbrace$ is equicontinuous, pointwise bounded, and closed in $C^0([0,T];\mathbb{R}^d)$, we can apply the Arzel\`a-Ascoli theorem to conclude that $A_R$ is compact in $C^0([0,T];\mathbb{R}^d)$. Hence, $[0,T] \times A_R$ is by Tychonoff's theorem compact in the product topology of $[0,T] \times C^0([0,T];\mathbb{R}^d)$, which implies that $K_R$ is compact with respect to the quotient topology induced by $d_\infty$. Thus, $(\Lambda_T^\alpha,\psi)$ is a weighted space.
			
			\item\label{ExWeightedSpacePVarPredual} For $T > 0$, $(p,\alpha) \in [1,\infty) \times (0,1)$ with $p \alpha < 1$, and a dual Banach space $(Z,\Vert \cdot \Vert_Z)$, let $X = C^{p-var,\alpha}([0,T];Z)$ be the space of $\alpha$-H\"older continuous paths $x: [0,T] \rightarrow Z$ with finite $p$-variation (see Section~\ref{SecNotation}). Then, $C^{p-var,\alpha}([0,T];Z)$ is by Theorem~\ref{ThmPVarPredual} again a dual Banach space and can be equipped with the corresponding weak-$*$-topology. Hence, $\psi(x) = \eta\left( \Vert x \Vert_{p-var,\alpha} \right)$ is by the Banach-Alaoglu theorem admissible, which shows that $(C^{p-var,\alpha}([0,T];Z),\psi)$ is a weighted space. Similarly, we can consider paths in $C^{p-var,\alpha}_0([0,T];Z)$ preserving the origin, where the same reasoning applies (see Remark~\ref{RemPVar0}).

            \item\label{ExWeightedSpacePVarWeaker} Opposed to \ref{ExWeightedSpacePVarPredual}, let $X = C^{p-var,\alpha}([0,T];Z)$ be the space of $\alpha$-H\"older continuous paths $x: [0,T] \rightarrow Z$ with finite $p$-variation, but now equipped with the $w^*$-uniform topology $\tau_\infty$. Then, we choose $\psi(x) = \eta\left( \Vert x \Vert_{p-var,\alpha} \right)$ and observe that the pre-image $K_R = \psi^{-1}((0,R])$ is closed with respect to $\tau_\infty$. Hence, by using Theorem~\ref{ThmPVarEmbedding} together with Lemma~\ref{LemmaPVarTop}~\ref{LemmaPVarTop1}, we conclude that $K_R = \psi^{-1}((0,R])$ is compact with respect to $\tau_\infty$, for all $R > 0$. This shows that $(C^{p-var,\alpha}([0,T];Z),\psi)$ is a weighted space. Note that instead of the $w^*$-uniform topology $\tau_\infty$ we could also use any $C^{p'-var,\alpha'}$-topology $\tau_{p'-var,\alpha'}$, with $(p',\alpha') \in (p,\infty] \times [0,\alpha)$ and $p' \alpha' < 1$, due to the compact embedding $(C^{p-var,\alpha}([0,T];Z),\Vert \cdot \Vert_{p-var,\alpha}) \hookrightarrow (C^{p'-var,\alpha'}([0,T];Z),\tau_{p'-var,\alpha'})$ in Theorem~\ref{ThmPVarEmbedding}.
			
			\item\label{ExWeightedSpaceLpPredual} For a measure space $(\Omega,\mathcal{F},\mu)$ and $p \in (1,\infty]$, let $X = L^p(\Omega) := L^p(\Omega,\mathcal{F},\mu)$ be the space of (equivalence classes of) $\mathcal{F}$-measurable functions $x: \Omega \rightarrow \mathbb{R}$ such that the norm $\Vert x \Vert_{L^p(\Omega)} = \left( \int_{\Omega} \vert x(\omega) \vert^p \mu(d\omega) \right)^{1/p}$ is finite. For $p = \infty$, we assume that $\mu: \mathcal{F} \rightarrow [0,\infty)$ is $\sigma$-finite and consider $\Vert x \Vert_{L^\infty(\Omega)} := \inf\left\lbrace c > 0: \mu\left(\left\lbrace \omega \in \Omega: \vert x(\omega) \vert > c \right\rbrace\right) = 0 \right\rbrace$. Then, $L^p(\Omega) \cong L^q(\Omega)^*$ with $1/p + 1/q = 1$ (and convention $1/\infty := 0$) is a dual Banach space and can be equipped with the corresponding weak-$*$-topology (which coincides with the weak topology for all $p \in (1,\infty)$). Hence, $\psi(x) = \eta\left( \Vert x \Vert_{L^p(\Omega)} \right)$ is by the Banach-Alaoglu theorem admissible, which shows that $(L^p(\Omega),\psi)$ is a weighted space.

            \item\label{ExWeightedSpaceMeasures} For a weighted space $(\Omega,\psi_\Omega)$, let $X = \mathcal{M}_{\psi_\Omega}(\Omega)$ be the Banach space of signed Radon measures $x: \mathcal{F}_\Omega \rightarrow \mathbb{R}$ satisfying $\int_\Omega \psi_\Omega(\omega) \vert x \vert(d\omega) < \infty$ with the norm $\Vert x \Vert_{\mathcal{M}_{\psi_\Omega}(\Omega)} := \sup\big\lbrace \left\vert \int_\Omega f(\omega) x(d\omega) \right\vert: f \in \mathcal{B}_{\psi_\Omega}(\Omega), \, \Vert f \Vert_{\mathcal{B}_{\psi_\Omega}(\Omega)} \leq 1 \big\rbrace$, where $\vert x \vert: \mathcal{F}_\Omega \rightarrow \mathbb{R}$ denotes the total variation measure, and where $(\mathcal{B}_{\psi_\Omega}(\Omega),\Vert \cdot \Vert_{\mathcal{B}_{\psi_\Omega}(\Omega)})$ is introduced in Definition~\ref{DefBpsi} below. Then, by the Riesz representation theorem in \cite[Theorem~2.4]{doersek10}, $\mathcal{M}_{\psi_\Omega}(\Omega) \cong \mathcal{B}_{\psi_\Omega}(\Omega)^*$ is a dual Banach space and can be equipped with the corresponding weak-$*$-topology. Hence, $\psi(x) = \eta\big( \Vert x \Vert_{\mathcal{M}_{\psi_\Omega}(\Omega)} \big)$ is by the Banach-Alaoglu theorem admissible.
		\end{enumerate}
	\end{example}
	
	\begin{remark}
        \label{RemWeakerTop}
    	The above examples give rise to the following observations:
    	\begin{enumerate}
            \item\label{RemWeakerTop1} Note that in all the infinite dimensional examples \ref{ExWeightedSpaceDual} to \ref{ExWeightedSpaceMeasures} the idea is always to use a weaker topology than the norm topology which renders the closed unit ball compact. In the above cases this is either achieved with the weak-$*$-topology as in Example \ref{ExWeightedSpaceDual}, \ref{ExWeightedSpaceHoelderPredual}, \ref{ExWeightedSpacePVarPredual}, \ref{ExWeightedSpaceLpPredual}, and \ref{ExWeightedSpaceMeasures}, or via the following compact embeddings:
              \begin{enumerate}
                  \item $(C^\alpha(S;Z),\Vert \cdot \Vert_\alpha) \hookrightarrow (C^{\alpha'}(S;Z),\tau_{\alpha'})$ for $0 \leq \alpha' < \alpha$ as in Example \ref{ExWeightedSpaceHoelderWeaker} and \ref{ExWeightedSpaceHoelderStoppedPaths},
                  \item $(C^{p-var,\alpha}([0,T];Z),\Vert \cdot \Vert_{p-var,\alpha}) \hookrightarrow (C^{p'-var,\alpha'}([0,T];Z),\tau_{p'-var,\alpha'})$ for $(p,\alpha) \in [1,\infty) \times (0,1)$ with $p \alpha < 1$ and $(p',\alpha') \in  (p,\infty] \times [0,\alpha)$ with $p' \alpha' < 1$ as in Example~\ref{ExWeightedSpacePVarWeaker},
                  \item versions of the Rellich–Kondrachov embedding $W^{1,p}(\mathbb{R}^d) \hookrightarrow L^p(\mathbb{R}^d)$, where $W^{1,p}(\mathbb{R}^d)$ denotes the Sobolov space of functions whose weak derivatives up to order $1$ are in $L^p$,
                  \item $L^p$-embeddings of Sobolev–Slobodeckij spaces (see, e.g., \cite{flandoli1995martingale}).
              \end{enumerate}
        	 
        	 \item\label{RemWeakerTop2} Consider the setting of Example \ref{ExWeightedSpaces}~\ref{ExWeightedSpaceHoelderPredual}. Then, as shown in Theorem~\ref{ThmHoelderPredual}, the weak-$*$-topology $\tau_{w^*}$, the $w^*$-uniform topology $\tau_\infty$, and every $C^{\alpha'}$-topology $\tau_{\alpha'}$, $0 \leq \alpha' < \alpha$, coincide on any $\Vert \cdot \Vert_\alpha$-bounded subset of $C^\alpha(S;Z)$, but are globally different. The same reasoning applies to Example \ref{ExWeightedSpaces}~\ref{ExWeightedSpacePVarPredual}, see Theorem~\ref{ThmPVarPredual}.
          
          \item\label{RemWeakerTop3} Note that the $p$-variation spaces cannot be compactly embedded into each other. For example, the sequence $(x_n)_{n \in \mathbb{N}} \subseteq C^{1-var}([0,1])$ defined by $x_n(t) = t^n$ has uniformly bounded $\Vert \cdot \Vert_{C^{1-var}}$-norm but converges pointwise to a discontinuous function. This is the reason why we consider the intersection $C^{p-var,\alpha}([0,T];Z) := C^{p-var}([0,T];Z) \cap C^\alpha([0,T];Z)$.   
          
          \item\label{RemWeakerTop4} Instead of $L^p$-spaces and Sobolev–Slobodeckij spaces, we could also consider Besov spaces $B^s_{p,q}(\Omega)$ with $\Omega \subseteq \mathbb{R}^d$, $p,q \in (1,\infty]$, and $s \in \mathbb{R}$ (see \cite[Section~1.2.5]{triebel92}). Indeed, if $\Omega = \mathbb{R}^d$ and $p,q \in (1,\infty]$ (with dual exponents $p',q' \in [1,\infty)$), then $B^s_{p,q}(\mathbb{R}^d) \cong B^{-s}_{p',q'}(\mathbb{R}^d)^*$ is by \cite[Theorem~2.11.2~(i)]{triebel10} a dual Banach space and carries a weak-$*$-topology (identical to the weak topology if $p,q \in (1,\infty)$). Hence, $B^s_{p,q}(\mathbb{R}^d)$ together with $\psi(x) = \eta\big( \Vert x \Vert_{B^s_{p,q}(\mathbb{R}^d)} \big)$ is a weighted space. Otherwise, if $\Omega \subseteq \mathbb{R}^d$ is bounded, $p',q' \in [1,\infty)$, and $s' \in (-\infty,s)$ with $s-d/p > s'-d/p'$, we can alternatively use the compact embedding $B^s_{p,q}(\Omega) \hookrightarrow B^{s'}_{p',q'}(\Omega)$ in \cite[Theorem~1.97]{triebel06} to obtain a weighted space $(B^s_{p,q}(\Omega),\psi)$, where $B^s_{p,q}(\Omega)$ is equipped with $\Vert \cdot \Vert_{B^{s'}_{p',q'}(\Omega)}$. Note that for $p=q$ we obtain the Sobolev–Slobodeckij spaces.
          
          Let us remark that Besov spaces have particular relevance in the theory of rough paths and in turn also for \emph{neural rough differential equations} (see, e.g., \cite{morrill2021neural} in the context of long time series prediction).
    	\end{enumerate}
	\end{remark}
	
    \subsection{Function space $\mathcal{B}_\psi(X;Y)$}
    \label{SecWeightedFunctions}
    
    For a given weighted space $(X,\psi)$ with admissible weight function $\psi: X \rightarrow (0,\infty)$ and a Banach space $(Y,\Vert \cdot \Vert_Y)$, we need to introduce an appropriate weighted function space that can be used for the subsequent global approximation theorems. Besides the vector space of bounded and continuous maps $C^0_b(X;Y)$, we define the vector space
	\begin{equation*}
		B_\psi(X;Y) = \left\lbrace f: X \rightarrow Y: \sup_{x \in X} \frac{\Vert f(x) \Vert_Y}{\psi(x)} < \infty \right\rbrace
	\end{equation*}
	of maps $f: X \rightarrow Y$, whose growth is controlled by the growth of the weight function $\psi: X \rightarrow (0,\infty)$. We equip $B_\psi(X;Y)$ with the weighted norm $\Vert \cdot \Vert_{\mathcal{B}_\psi(X;Y)}$ given by
	\begin{equation}
		\label{EqDefWeightedNorm}
		\Vert f \Vert_{\mathcal{B}_\psi(X;Y)} = \sup_{x \in X} \frac{\Vert f(x) \Vert_Y}{\psi(x)},
	\end{equation}
	for $f \in B_\psi(X;Y)$. Since $\psi: X \rightarrow (0,\infty)$ is bounded from below by a strictly positive constant (see Remark~\ref{RemWeight}~\ref{RemWeight1}), the embedding $C^0_b(X;Y) \hookrightarrow B_\psi(X;Y)$ is continuous.
	
	\begin{definition}
		\label{DefBpsi}
		We define $\mathcal{B}_\psi(X;Y)$ as the closure of $C^0_b(X;Y)$ with respect to $\Vert \cdot \Vert_{\mathcal{B}_\psi(X;Y)}$, which is a Banach space equipped with the weighted norm defined in \eqref{EqDefWeightedNorm}. If $Y=\mathbb{R}$, we shall only write $\mathcal{B}_\psi(X)$.
	\end{definition}
	
	Note that if $X$ is compact, the weight function $\psi(x) = 1$ is admissible. However, on general spaces, $\psi: X \rightarrow (0,\infty)$ grows on the compact pre-images $K_R := \psi^{-1}((0,R])$, which means that the elements of $\mathcal{B}_\psi(X;Y)$ are typically unbounded, but their growth is controlled by the growth of the weight function $\psi: X \rightarrow (0,\infty)$.
	
	\begin{remark}
	    \label{RemBanachOutput}
	    For simplicity, we always assume that the output space is a Banach space $(Y,\Vert \cdot \Vert_Y)$. However, all the following results (including the weighted vector-valued Stone-Weierstrass theorem in Theorem~\ref{ThmStoneWeierstrassBpsi}) still hold true with a locally convex topological vector space $(Y,\tau_Y)$ as output space. In this case, the topology on $Y$ and $\mathcal{B}_\psi(X;Y)$ is generated by families of seminorms (see \cite{schaefer99} for more details).
	\end{remark}
	
	The following result characterizes functions in $\mathcal{B}_\psi(X;Y)$, which is a slight generalization of \cite[Theorem~2.7]{doersek10} to this vector-valued setting. The reverse implication in \ref{LemmaBpsiEquivChar2} only holds true for the real-valued case $Y = \mathbb{R}$. It relies on the Tietze extension theorem for the Stone-\v{C}ech compactification of $X$, see \cite[Th\'eor\`eme~IX\S4.2]{bourbaki74}.
	
	\begin{lemma}
		\label{LemmaBpsiEquivChar}
        The following holds true for a function $f: X \rightarrow Y$:
        \begin{enumerate}
            \item\label{LemmaBpsiEquivChar1} If $f \in \mathcal{B}_\psi(X;Y)$ then $f\vert_{K_R} \in C^0(K_R;Y)$, for all $R > 0$, and
                \begin{equation}
                    \label{EqLemmaBpsiEquivChar1}
                    \lim_{R \rightarrow \infty} \sup_{x \in X \setminus K_R} \frac{\Vert f(x) \Vert_Y}{\psi(x)} = 0.
                \end{equation}
            \item\label{LemmaBpsiEquivChar2} Let $Y = \mathbb{R}$. If $f \in C^0(K_R)$, for all $R > 0$, and
                \begin{equation}
                    \label{EqLemmaBpsiEquivChar2}
                    \lim_{R \rightarrow \infty} \sup_{x \in X \setminus K_R} \frac{\vert f(x) \vert}{\psi(x)} = 0,
                \end{equation}
                then $f \in \mathcal{B}_\psi(X)$. In particular, $f \in \mathcal{B}_\psi(X)$ for every $f \in C^0(X)$ satisfying \eqref{EqLemmaBpsiEquivChar2}.
        \end{enumerate}
	\end{lemma}
    \begin{proof}
        For \ref{LemmaBpsiEquivChar1}, let $f \in \mathcal{B}_\psi(X;Y)$ and fix some $\varepsilon > 0$. Then, there exists by definition of $\mathcal{B}_\psi(X;Y)$ some $g \in C^0_b(X;Y)$ such that $\Vert f - g \Vert_{\mathcal{B}_\psi(X;Y)} \leq \varepsilon/2$. Hence, by choosing $R \geq 2 \varepsilon^{-1} \sup_{x \in X} \Vert g(x) \Vert_Y$, it follows for every $x \in X \setminus K_R$ that
        \begin{equation*}
            \frac{\Vert f(x) \Vert_Y}{\psi(x)} \leq \frac{\Vert f(x) - g(x) \Vert_Y}{\psi(x)} + \frac{\Vert g(x) \Vert_Y}{\psi(x)} \leq \frac{\varepsilon}{2} + \frac{\sup_{x \in X} \Vert g(x) \Vert_Y}{R} \leq \frac{\varepsilon}{2} + \frac{\varepsilon}{2} = \varepsilon,
        \end{equation*}
        which shows that
        \begin{equation*}
            \sup_{x \in X \setminus K_R} \frac{\Vert f(x) \Vert_Y}{\psi(x)} \leq \varepsilon.
        \end{equation*}
        Since $\varepsilon > 0$ was chosen arbitrarily, we obtain \eqref{EqLemmaBpsiEquivChar1}. Moreover, with $g \in C^0_b(X;Y)$ from above, we observe for every $R > 0$  that 
        \begin{equation*}
            \sup_{x \in K_R} \Vert f(x) - g(x) \Vert_Y \leq R \sup_{x \in K_R} \frac{\Vert f(x) - g(x) \Vert_Y}{\psi(x)} \leq R \Vert f - g \Vert_{\mathcal{B}_\psi(X;Y)} < \frac{\varepsilon}{2} R.
        \end{equation*}
        This shows that $f\vert_{K_R} \in C^0(K_R;Y)$ is continuous as a uniform limit of continuous functions, for all $R > 0$. On the other hand, \ref{LemmaBpsiEquivChar2} follows from \cite[Theorem~2.7]{doersek10}.
    \end{proof}
	
	In addition, by following \cite[Theorem~2.8]{doersek10}, there exists for every $f \in \mathcal{B}_\psi(X)$ with $\sup_{x \in X} f(x) > 0$ some $x_0 \in X$ such that $f(x)/\psi(x) \leq f(x_0)/\psi(x_0)$, for all $x \in X$, which shows the analogy to functions vanishing at infinity on locally compact spaces.
	
\section{Weighted Stone-Weierstrass theorems}
\label{SecSW}

    In this section, we generalize the Stone-Weierstrass theorems to this weighted setting with non-compact input space. For this purpose, we shortly recall the classical Stone-Weierstrass theorem formulated for compact domains, which we will need for our results.

    \subsection{Classical formulation}
    
    For a subset $E \subseteq \mathbb{R}$, we denote by $\Pol(E) \subseteq C^0(E)$ the vector subspace of polynomial functions of the form $E \ni x \mapsto \sum_{n=0}^N c_n x^n \in \mathbb{R}$, with $N \in \mathbb{N}_0$ and $ c_0,\ldots,c_N \in \mathbb{R}$.

    \begin{theorem}[Weierstrass, {\cite{weierstrass85}}]
    	\label{ThmWstrass}
    	For $X \subsetneq \mathbb{R}$ compact, $\Pol(X)$ is dense in $C^0(X)$.
    \end{theorem}
    
    Subsequently, the Weierstrass theorem (Theorem~\ref{ThmWstrass}) was generalized by Marshall Harvey Stone in \cite{stone48} to more general spaces by using the notion of subalgebras. Hereby, a vector space $\mathcal{A}$ of maps $a: X \rightarrow \mathbb{R}$ is called a \emph{subalgebra} if $\mathcal{A}$ is closed under multiplication, i.e.~for every $a_1, a_2 \in \mathcal{A}$ it holds that $a_1 \cdot a_2 \in \mathcal{A}$. Moreover, $\mathcal{A}$ is called \emph{point separating} if for every distinct $x_1,x_2 \in X$ there exists some $a \in \mathcal{A}$ with $a(x_1) \neq a(x_2)$. In addition, $\mathcal{A}$ \emph{vanishes nowhere} if for every $x \in X$ there exists some $a \in \mathcal{A}$ with $a(x) \neq 0$.
    
    \begin{theorem}[Stone-Weierstrass on $C^0(X)$, {\cite{stone48} or \cite[p.~122]{rudin91}}]
    	\label{ThmStone}
    	Let $(X,\tau_X)$ be a compact Hausdorff topological space and let $\mathcal{A} \subseteq C^0(X)$ be a subalgebra. Then, $\mathcal{A}$ is dense in $C^0(X)$ if and only if $\mathcal{A}$ is point separating and vanishes nowhere.
    \end{theorem}
    
    Next, we state the vector-valued Stone-Weierstrass theorem of R.~Creighton Buck in \cite{buck58} (see also \cite[Theorem~3]{prolla94}). For a given subalgebra $\mathcal{A} \subseteq C^0(X)$, a vector subspace $\mathcal{W} \subseteq C^0(X;Y)$ is called an \emph{$\mathcal{A}$-submodule} if $a \cdot w \in \mathcal{W}$ for all $a \in \mathcal{A}$ and $w \in \mathcal{W}$, where $X \ni x \mapsto (a \cdot w)(x) := a(x) w(x) \in Y$. For further details on vector-valued approximation, we also refer to \cite{prolla77,prolla94}.
    
    \begin{theorem}[Stone-Weierstrass on $C^0(X;Y)$, {\cite[p.~103]{buck58} or \cite[Theorem~3]{prolla94}}]
    	\label{ThmStoneWeierstrass}
    	Let $(X,\tau_X)$ be a compact Hausdorff space and let $\mathcal{A} \subseteq C^0(X)$ be a point separating subalgebra that vanishes nowhere. Moreover, let $\mathcal{W} \subseteq C^0(X;Y)$ be an $\mathcal{A}$-submodule such that $\mathcal{W}(x) := \lbrace w(x): w \in \mathcal{W} \rbrace$ is dense in $Y$, for all $x \in X$. Then, $\mathcal{W}$ is dense in $C^0(X;Y)$.
    \end{theorem}
    
    Later, the vector-valued Stone-Weierstrass theorem was extended by Erret Bishop in \cite{bishop61} to a measure theoretic version. In addition, Silvio Machado derived in \cite{machado77} a quantitative version, which relies on Zorn's lemma, see also the proofs of \cite{brosowski81,ransford84}.

    \subsection{Weighted real-valued Stone-Weierstrass theorem}
	
	We now formulate a generalized version of the classical Stone-Weierstrass theorem in this weighted setting, first for the case when the output space is $\mathbb{R}$. To this end we need an additional condition related to Nachbin's definition of localisability (see \cite[Definition~4]{nachbin65}). 
 
    Note that Nachbin's setting in \cite{nachbin65} differs in several respects, in particular in the definition of weighted topologies. It allows for multiple upper semicontinuous weights $v: X \rightarrow [0,\infty)$, e.g.~$v(x) = \psi(x)^{-1}$, collected in a so-called Nachbin family and can therefore represent the compact-open topology, the strict topology, and other topologies.
	
	\begin{definition}
		\label{DefPtSepModGrowth}
		For a weighted space $(X,\psi)$, a vector space $\mathcal{A}$ of maps $a: X \rightarrow \mathbb{R}$ is called \emph{point separating and nowhere vanishing of $\psi$-moderate growth} if there exists a vector subspace $\widetilde{\mathcal{A}} \subseteq \mathcal{A}$ such that
        \begin{enumerate}
            \item[\labeltext{(M1)}{M1}] $\widetilde{\mathcal{A}}$ is point separating,
            \item[\labeltext{(M2)}{M2}] $\widetilde{\mathcal{A}}$ vanishes nowhere, and
            \item[\labeltext{(M3)}{M3}] for every $\widetilde{a} \in \widetilde{\mathcal{A}}$ there is some $\lambda > 0 $ such that $\exp\left(\lambda \left\vert \widetilde{a}(\cdot) \right\vert\right) := \left( x \mapsto \exp\left(\lambda \left\vert \widetilde{a}(x) \right\vert\right) \right) \in \mathcal{B}_\psi(X)$.
        \end{enumerate}
	\end{definition}

	\begin{remark}
		\label{RemStoneWstrassBdedFct}
		We do not assume a priori that $\mathcal{A}$ belongs to a given weighted function space, e.g.~in Theorem~\ref{ThmStoneWeierstrassBpsi} we only assume that $\mathcal{A}$ is point separating and nowhere vanishing of $\psi_w$-moderate growth but not that $\mathcal{A} \subseteq \mathcal{B}_{\psi_w}(X)$, where $\psi_w: X \rightarrow (0,\infty)$ is a tilted weight function (see Section~\ref{SecVecWeightedSW}). Moreover, $\mathcal{A}$ is point separating and nowhere vanishing of $\psi$-moderate growth under \emph{any} admissible weight $\psi: X \rightarrow (0,\infty)$ if $\mathcal{A}$ is point separating, vanishes nowhere, and consists only of bounded maps. This is related to the so-called ``bounded approximation problem'' in \cite[Theorem~1]{nachbin65}.
	\end{remark}

    For the proof we shall apply the following lemmas on Banach space-valued real-analytic functions, which are of interest on their own right. For more background on Banach space-valued real-analytic functions, we refer to \cite[Chapter~IX]{dieudonne69}.

    \begin{lemma}
        \label{lem:exp_analytic}
        Let $\widetilde{a}: X \rightarrow \mathbb{R}$ such that $\exp\left(\lambda \left\vert \widetilde{a}(\cdot) \right\vert\right) \in \mathcal{B}_\psi(X)$ for some $\lambda > 0$. Then, the maps $\mathbb{R} \ni \lambda \mapsto \cos\left( \lambda \widetilde{a}(\cdot) \right) \in \mathcal{B}_\psi(X)$ and $\mathbb{R} \ni \lambda \mapsto \sin\left( \lambda \widetilde{a}(\cdot) \right) \in \mathcal{B}_\psi(X)$ are real-analytic.
    \end{lemma}
    \begin{proof}
        Let $U := \lbrace z \in \mathbb{C}: \vert \img(z) \vert < \lambda/2 \rbrace$. Then, by applying Taylor's theorem to the holomorphic function $\cos: \mathbb{C} \rightarrow \mathbb{C}$ and by using that $\vert \sin(z) \vert \leq \left\vert \frac{e^{\mathbf{i} z} - e^{-\mathbf{i} z}}{2\mathbf{i}} \right\vert \leq \frac{\vert e^{\mathbf{i} z} \vert + \vert e^{-\mathbf{i} z} \vert}{2} \leq e^{\vert \img(z) \vert}$, it holds for every $z_0,z_1 \in \mathbb{C}$ that
        \begin{equation}
            \label{eq:lem:exp_analytic:proof1}
            \begin{aligned}
                \vert \cos(z_1) - \cos(z_0) \vert & = \left\vert (z_1-z_0) \int_0^1 \cos'(z_0+t(z_1-z_0)) dt \right\vert \\
                & \leq \vert z_1-z_0 \vert \int_0^1 \vert \sin(z_0+t(z_1-z_0)) \vert dt \\
                & \leq \vert z_1-z_0 \vert \int_0^1 e^{(1-t) \vert \img(z_0) \vert + t \vert \img(z_1) \vert\vert} dt \\
                & \leq \vert z_1-z_0 \vert e^{\max\left( \vert \img(z_0) \vert, \vert \img(z_1) \vert \right)}.
            \end{aligned}
        \end{equation}
        Moreover, for every fixed $x \in X$, we apply Taylor's theorem to the holomorphic function $\mathbb{C} \ni \lambda \mapsto \sin(\lambda \widetilde{a}(x)) \in \mathbb{C}$ to conclude for every $\lambda_0,\lambda_1 \in U$ with $\lambda_0 \neq \lambda_1$ that 
        \begin{equation}
            \label{eq:lem:exp_analytic:proof2}
            \frac{\sin(\lambda_1 \widetilde{a}(x)) - \sin(\lambda_0 \widetilde{a}(x))}{\lambda_1 - \lambda_0} = \widetilde{a}(x) \int_0^1 \cos\left( (\lambda_0 + t(\lambda_1-\lambda_0)) \widetilde{a}(x) \right) dt.
        \end{equation}
        Hence, by using the identity~\eqref{eq:lem:exp_analytic:proof2}, the inequality \eqref{eq:lem:exp_analytic:proof1} with $z_1 := (\lambda_0 + t(\lambda_1-\lambda_0)) \widetilde{a}(x) \in \mathbb{C}$ and $z_0 := \lambda_0 \widetilde{a}(x) \in \mathbb{C}$, and the inequality $\vert \widetilde{a}(x) \vert^2 = \frac{8}{\lambda^2} \frac{(\frac{\lambda}{2} \vert \widetilde{a}(x) \vert)^2}{2!} \leq \frac{8}{\lambda^2} \sum_{n=0}^\infty \frac{(\frac{\lambda}{2} \vert \widetilde{a}(x) \vert)^n}{n!} = \frac{8}{\lambda^2} \exp(\frac{\lambda}{2} \vert \widetilde{a}(x) \vert)$, it follows for every $\lambda_0,\lambda_1 \in U$ with $\lambda_0 \neq \lambda_1$ that
        \begin{equation*}
            \begin{aligned}
                & \left\vert \frac{\sin(\lambda_1 \widetilde{a}(x)) - \sin(\lambda_0 \widetilde{a}(x))}{\lambda_1 - \lambda_0} - \widetilde{a}(x) \cos\left( \lambda_0 \widetilde{a}(x) \right) \right\vert \\
                & \quad\quad \leq \left\vert \widetilde{a}(x) \right\vert \int_0^1 \left\vert \cos\left( (\lambda_0 + t(\lambda_1-\lambda_0)) \widetilde{a}(x) \right) - \cos\left( \lambda_0 \widetilde{a}(x) \right) \right\vert dt \\
                & \quad\quad \leq \left\vert \widetilde{a}(x) \right\vert \int_0^1 \vert t(\lambda_1-\lambda_0) \widetilde{a}(x) \vert e^{\vert \widetilde{a}(x) \vert \max\left( \vert \img(\lambda_0 + t(\lambda_1-\lambda_0)) \vert, \vert \img(\lambda_0) \vert \right)} dt \\
                & \quad\quad \leq \left\vert \widetilde{a}(x) \right\vert^2 \vert \lambda_1-\lambda_0 \vert e^{\frac{\lambda}{2} \vert \widetilde{a}(x) \vert} \leq \frac{8 \vert \lambda-\lambda_0 \vert}{\lambda^2} e^{\lambda \vert \widetilde{a}(x) \vert}.
            \end{aligned}
        \end{equation*}
        Thus, we obtain for every $\lambda_0,\lambda_1 \in U$ with $\lambda_0 \neq \lambda_1$ that
        \begin{equation*}
            \left\Vert \frac{\sin(\lambda_1 \widetilde{a}(\cdot)) \!-\! \sin(\lambda_0 \widetilde{a}(\cdot))}{\lambda_1 \!-\! \lambda_0} - \widetilde{a}(\cdot) \cos\left( \lambda_0 \widetilde{a}(\cdot) \right) \right\Vert_{\mathcal{B}_\psi(X)} \!\leq\! \frac{8 \vert \lambda\!-\!\lambda_0 \vert}{\lambda^2} \left\Vert \exp\left( \lambda \left\vert \widetilde{a}(\cdot) \right\vert \right) \right\Vert_{\mathcal{B}_\psi(X)} \overset{\lambda \rightarrow \lambda_0}{\longrightarrow} 0,
        \end{equation*}
        which shows that $U \ni \lambda \mapsto \sin\left( \lambda \widetilde{a}(\cdot) \right) \in \mathcal{B}_\psi(X)$ is holomorphic, implying that $\mathbb{R} \ni \lambda \mapsto \sin\left( \lambda \widetilde{a}(\cdot) \right) \in \mathcal{B}_\psi(X)$ is real-analytic. By a similar argument, $U \ni \lambda \mapsto \cos\left( \lambda \widetilde{a}(\cdot) \right) \in \mathcal{B}_\psi(X)$ is holomorphic, ensuring that $\mathbb{R} \ni \lambda \mapsto \cos\left( \lambda \widetilde{a}(\cdot) \right) \in \mathcal{B}_\psi(X)$ is real-analytic.
    \end{proof}
    
    \begin{lemma}
        \label{lem:analytic_continuation}
        Let $0 \in I \subseteq \mathbb{R}$ be an open interval, let $E$ be a Banach space, and assume that $f: I \rightarrow E$ is a real-analytic map. Let $F \subseteq E$ be a closed vector subspace and assume that $f^{(k)}(0) \in F$ for all $k \in \mathbb{N}_0$. Then, $f(I) \subseteq F$.
    \end{lemma}
    \begin{proof}
        Let $\lambda \in I$ be arbitrary. By real-analyticity, there exists for every $t \in I$ some $r_t > 0$ such that $f\vert_{(-r_t,r_t)}$ can be represented as power series with convergence radius $r_t$. Since $(t-r_t,t+r_t)_{t \in [0,\lambda]}$ is an open cover of the compact set $[0,\lambda]$, there exists a finite subcover $(t_i-r_{t_i},t_i+r_{t_i})_{i=1,\ldots,n}$ of $[0,\lambda]$, for some $0 = t_0 < \ldots < t_n = \lambda$. Hence, by using that $f^{(k)}(t_0) = f^{(k)}(0) \in F$ for all $k \in \mathbb{N}_0$ and that $F \subseteq E$ is closed, we have $f^{(k)}(t) \in F$ for all $k \in \mathbb{N}_0$ and $\vert t \vert < r_{t_0}$. Since $t_1 < r_{t_0}$, it holds in particular that $f^{(k)}(t_1) \in F$ for all $k \in \mathbb{N}_0$. Now, we can argue by induction to conclude that $f(\lambda) \in F$, completing the proof.
    \end{proof}
    
    \begin{remark}
        We can replace ``open interval'' by ``open domain in $\mathbb{C}$'' and ``real-analytic'' by ``analytic'' for a complex Banach space $E$, then the conclusion remains verbatim the same.
    \end{remark}

	\begin{theorem}[Stone-Weierstrass on $\mathcal{B}_\psi(X)$]
		\label{ThmStoneWeierstrassBpsiR}
		Let $\mathcal{A} \subseteq \mathcal{B}_\psi(X)$ be a subalgebra that is point separating and nowhere vanishing of $\psi$-moderate growth. Then, $\mathcal{A}$ is dense in $\mathcal{B}_\psi(X)$.
	\end{theorem}
    \begin{remark}
    Condition (M3) is the analogue of the exponential moment condition for the uniqueness of the moment problem. Notice, however, that we avoid using characteristic functions but rather work with Lemma \ref{lem:analytic_continuation}.
    \end{remark}
	\begin{proof}
    	Since $\mathcal{B}_\psi(X)$ is defined as the closure of $C^0_b(X)$ with respect to $\Vert \cdot \Vert_{\mathcal{B}_\psi(X)}$, it suffices to show the approximation of every $f \in C^0_b(X)$ by an element $a \in \mathcal{A}$. Now, we first assume that $\mathcal{A}$ consists only of bounded maps, where the condition of $\mathcal{A}$ being point separating and nowhere vanishing of $\psi$-moderate growth reduces to $\mathcal{A}$ being point separating and nowhere vanishing. Let $f \in C^0_b(X)$, fix some $\varepsilon > 0$, and define the constants $M = \left( \inf_{x \in X} \psi(x) \right)^{-1} > 0$, $b = \sup_{x \in X} \vert f(x) \vert + \varepsilon/(4M)$, and $R \geq 4 b/\varepsilon$. Then, by using that the restriction $\mathcal{A}\vert_{K_R}$ is a point separating subalgebra of $C^0(K_R)$ that vanishes nowhere, we can apply the classical real-valued Stone-Weierstrass theorem (Theorem~\ref{ThmStone}) to conclude that there exists some $a \in \mathcal{A}$ such that
    	\begin{equation*}
    		\sup_{x \in K_R} \vert f(x) - a(x) \vert < \frac{\varepsilon}{4M},
    	\end{equation*}
    	which implies that $\vert a(x) \vert \leq \varepsilon/(4M) + \sup_{x \in X} \vert f(x) \vert = b$, for all $x \in K_R$. Let $g \in C^0_b(\mathbb{R})$ be the function defined by $g(s) = \max(\min(s,b),-b)$, for $s \in \mathbb{R}$, which implies that $g(a(x)) = a(x)$, for all $x \in K_R$. Then, we conclude that
    	\begin{equation}
    	\label{EqThmStoneWeierstrassBpsiRProof1}
    	\begin{aligned}
    		\Vert f - g \circ a \Vert_{\mathcal{B}_\psi(X)} & \leq M \sup_{x \in K_R} \vert f(x) - a(x) \vert + \sup_{x \in X \setminus K_R} \frac{\vert f(x) \vert}{\psi(x)} + \sup_{x \in X \setminus K_R} \frac{\vert g(a(x)) \vert}{\psi(x)} \\
    		& < M \frac{\varepsilon}{4M} + \frac{b}{R} + \frac{b}{R} \leq \frac{3\varepsilon}{4}.
    	\end{aligned}
    	\end{equation}
    	Now, for $c = \sup_{x \in X} \vert a(x) \vert$, we apply Weierstrass' theorem (Theorem~\ref{ThmWstrass}) to obtain a polynomial $p \in \Pol(\mathbb{R})$ with $\sup_{\vert s \vert \leq c} \vert g(s) - p(s) \vert < \varepsilon/(4M)$, where we can assume without loss of generality that $p(0) = 0$ as $g(0) = 0$. Hence, it follows that
    	\begin{equation}
    	\label{EqThmStoneWeierstrassBpsiRProof2}
    	    \Vert g \circ a - p \circ a \Vert_{\mathcal{B}_\psi(X)} \leq M \sup_{x \in X} \vert g(a(x)) - p(a(x)) \vert \leq M \sup_{\vert s \vert \leq c} \vert g(s) - p(s) \vert < M \frac{\varepsilon}{4M} = \frac{\varepsilon}{4}.
    	\end{equation}
    	Thus, \eqref{EqThmStoneWeierstrassBpsiRProof1} and \eqref{EqThmStoneWeierstrassBpsiRProof2} imply for $p \circ a \in \mathcal{A}$ (as $p(0) = 0$ and $\mathcal{A}$ is a subalgebra) that
    	\begin{equation*}
    		\Vert f - p \circ a \Vert_{\mathcal{B}_\psi(X)} \leq \Vert f - g \circ a \Vert_{\mathcal{B}_\psi(X)} + \Vert g \circ a - p \circ a \Vert_{\mathcal{B}_\psi(X)} < \frac{3\varepsilon}{4} + \frac{\varepsilon}{4} = \varepsilon.
    	\end{equation*}
        Since $\varepsilon > 0$ and $f \in C^0_b(X)$ were chosen arbitrarily, $\mathcal{A} \subseteq C^0_b(X)$ is dense in $\mathcal{B}_\psi(X)$.
    
    	For the general case of a subalgebra $\mathcal{A} \subseteq \mathcal{B}_\psi(X)$ that is point separating and nowhere vanishing of $\psi$-moderate growth, let $\widetilde{\mathcal{A}} \subseteq \mathcal{A}$ be the point separating and nowhere vanishing vector subspace such that $\exp\left(\lambda \left\vert \widetilde{a}(\cdot) \right\vert\right) \in \mathcal{B}_\psi(X)$, for any $\widetilde{a} \in \widetilde{\mathcal{A}}$ with $\lambda > 0 $ possibly depending on $\widetilde{a} \in \widetilde{\mathcal{A}}$. Then by using the function $\mathbb{R} \ni \cos^*(s) := \cos(s)-1 \in \mathbb{R}$, we define
    	\begin{equation*}
    		\mathcal{A}_{\text{trig}} = \linspan\left( \left\lbrace \cos^*\left( \widetilde{a}(\cdot) \right): \widetilde{a} \in \widetilde{\mathcal{A}} \right\rbrace \cup \left\lbrace \sin\left( \widetilde{a}(\cdot) \right): \widetilde{a} \in \widetilde{\mathcal{A}} \right\rbrace \right).
    	\end{equation*}
    	Since $\mathcal{A} \subseteq \mathcal{B}_\psi(X)$, we conclude from Lemma~\ref{LemmaBpsiEquivChar}~\ref{LemmaBpsiEquivChar1} that $\widetilde{a}\vert_{K_R} \in C^0(K_R)$, for all $\widetilde{a} \in \widetilde{\mathcal{A}}$ and $R > 0$. Hence, by using that $\cos^*\left( \widetilde{a}(\cdot) \right)\vert_{K_R} \in C^0(K_R)$ and $\sin\left( \widetilde{a}(\cdot) \right)\vert_{K_R} \in C^0(K_R)$, and that the maps $\cos^*\left( \widetilde{a}(\cdot) \right): X \rightarrow \mathbb{R}$ and $\sin\left( \widetilde{a}(\cdot) \right): X \rightarrow \mathbb{R}$ are bounded, it follows from Lemma~\ref{LemmaBpsiEquivChar}~\ref{LemmaBpsiEquivChar2} that $\mathcal{A}_{\text{trig}} \subseteq \mathcal{B}_\psi(X)$. Moreover, by using the trigonometric identities
        \begin{equation*}
            \begin{aligned}
                \cos^*(s) \cos^*(t) & = \cos(s) \cos(t) - \cos(s) - \cos(t) + 1 \\
                & = \frac{1}{2} \big( \cos(s-t) + \cos(s+t) \big) - \cos(s) - \cos(t) + 1 \\
                & = \frac{1}{2} \big( \cos^*(s-t) + \cos^*(s+t) \big) - \cos^*(s) - \cos^*(t), \\
                \cos^*(s) \sin(t) & = \cos(s) \sin(t) - \sin(t) \\
                & = \frac{1}{2} \big( \sin(s+t) - \sin(s-t) \big) - \sin(t), \quad \text{and} \\
                \sin(s) \sin(t) & = \frac{1}{2} \big( \cos(s-t) - \cos(s+t) \big) \\
                & = \frac{1}{2} \big( \cos^*(s-t) - \cos^*(s+t) \big),
            \end{aligned}
        \end{equation*}
    	for any $s,t \in \mathbb{R}$, we observe that $\mathcal{A}_{\text{trig}}$ is a subalgebra of $\mathcal{B}_\psi(X)$. In addition, for any fixed $x_0 \in X$ there exists some $\widetilde{a} \in \widetilde{\mathcal{A}}$ such that $\widetilde{a}(x_0) \neq 0$. Hence, for suitable $t \in \mathbb{R}$, we use that $\widetilde{\mathcal{A}}$ is a vector space to conclude that the map $\sin\left(t \widetilde{a}(\cdot)\right) \in \mathcal{A}_{\text{trig}}$ satisfies $a(x_0) = \sin\left( t \widetilde{a}(x_0) \right) \neq 0$, showing that $\mathcal{A}_{\text{trig}}$ vanishes nowhere. Furthermore, for any distinct $x_1,x_2 \in X$ there exists some $\widetilde{a} \in \widetilde{\mathcal{A}}$ such that $\widetilde{a}(x_1) \neq \widetilde{a}(x_2)$. Thus, for some $t \neq 0$ small enough, $t \widetilde{a}(x_1)$ and $t \widetilde{a}(x_2)$ are distinct points in $(-\frac{\pi}{2},\frac{\pi}{2})$ and the map $\sin\left( t \widetilde{a}(\cdot) \right) \in \mathcal{A}_{\text{trig}}$ therefore satisfies $a(x_1) = \sin\left( t \widetilde{a}(x_1) \right) \neq \sin\left( t \widetilde{a}(x_2) \right) = a(x_2)$ by injectivity of the sine. This proves that $\mathcal{A}_{\text{trig}}$ is point separating.
        
        Now, we show that $\mathcal{A}_{\text{trig}}$ belongs to the closure of $\mathcal{A}$ with respect to $\Vert \cdot \Vert_{\mathcal{B}_\psi(X)}$. Indeed, let $\widetilde{a} \in \widetilde{\mathcal{A}}$, $\lambda > 0$ such that $ \exp\left(\lambda \left \vert \widetilde{a}(\cdot) \right \vert \right) \in \mathcal{B}_\psi(X) $, and fix some $\varepsilon > 0$. Then, by using Lemma~\ref{LemmaBpsiEquivChar}~\ref{LemmaBpsiEquivChar1}, there exists some $R > 8/\varepsilon$ such that
    	\begin{equation}
    	    \label{EqThmStoneWeierstrassBpsiRProof4}
    	    \sup_{x \in X \setminus K_R} \frac{\exp\left( \lambda \left\vert \widetilde{a}(x) \right\vert\right)}{\psi(x)} < \frac{\varepsilon}{4}.
    	\end{equation}
    	Now, for $c = \lambda \sup_{x \in K_R} \left\vert \widetilde{a}(x) \right\vert$, we use the Taylor polynomial of $\cos^*: \mathbb{R} \rightarrow \mathbb{R}$ that is given by $p_n(s) = \sum_{k=1}^n \frac{(-1)^k}{(2k)!} s^{2k}$ to conclude that $\sup_{\vert s \vert \leq c} \big\vert \cos^*(s) - p_n(s) \big\vert \leq \frac{c^{2n+1}}{(2n+1)!}$. Hence, by choosing $n \in \mathbb{N}_0$ large enough such that $\frac{c^{2n+1}}{(2n+1)!} \leq \frac{\varepsilon}{2M}$, by using $\vert p_n(s) \vert \leq \exp(\vert s \vert)$ for any $s \in \mathbb{R}$, and \eqref{EqThmStoneWeierstrassBpsiRProof4}, we obtain for $p_n \circ \lambda \widetilde{a} \in \mathcal{A}$ (as $p_n(0) = 0$ and $\mathcal{A}$ is a subalgebra) that
    	\begin{equation*}
        	\begin{aligned}
            	& \left\Vert \cos^*\left( \lambda \widetilde{a}(\cdot) \right)- p_n\left( \lambda \widetilde{a}(\cdot) \right) \right\Vert_{\mathcal{B}_\psi(X)} \\
            	& \quad\quad \leq M \sup_{x \in K_R} \left\vert \cos^*\left( \lambda \widetilde{a}(x)\right) - p_n\left(\lambda \widetilde{a}(x)\right) \right\vert + \sup_{x \in X \setminus K_R} \frac{\left\vert \cos^*\left( \lambda \widetilde{a}(x)\right) \right\vert}{\psi(x)} + \sup_{x \in X \setminus K_R} \frac{\left\vert p_n\left(\lambda \widetilde{a}(x)\right) \right\vert}{\psi(x)} \\
            	& \quad\quad \leq M \sup_{\vert s \vert \leq c} \left\vert \cos^*(s) - p_n(s) \right\vert + \frac{\sup_{x \in X} \left\vert \cos^*\left( \lambda \widetilde{a}(x)\right) \right\vert}{R} + \sup_{x \in X \setminus K_R} \frac{\exp\left( \lambda \left\vert \widetilde{a}(x) \right\vert\right)}{\psi(x)} \\
            	& \quad\quad < M \frac{\varepsilon}{2M} + \frac{2}{R} + \frac{\varepsilon}{4} < \varepsilon.
        	\end{aligned}
    	\end{equation*}
    	Since $\varepsilon > 0$ was chosen arbitrarily, the map $\cos^*\left( \lambda \widetilde{a}(\cdot) \right): X \rightarrow \mathbb{R}$ belongs to the closure of $\mathcal{A}$ with respect to $\Vert \cdot \Vert_{\mathcal{B}_\psi(X)}$, which holds analogously true for $\sin\left( \lambda \widetilde{a}(\cdot) \right): X \rightarrow \mathbb{R}$. Hence, by using that $\mathbb{R} \ni \lambda \mapsto \cos^*\left( \lambda \widetilde{a}(\cdot) \right) := \cos\left( \lambda \widetilde{a}(\cdot) \right) - 1 \in \mathcal{B}_\psi(X)$ and $\mathbb{R} \ni \lambda \mapsto \sin \left( \lambda \widetilde{a}(\cdot) \right) \in \mathcal{B}_\psi(X)$ are real-analytic (see Lemma~\ref{lem:exp_analytic}), it follows from Lemma~\ref{lem:analytic_continuation} that $\cos^*\left( \lambda \widetilde{a}(\cdot) \right) \in \mathcal{B}_\psi(X)$ and $\sin\left( \lambda \widetilde{a}(\cdot) \right) \in \mathcal{B}_\psi(X)$ for all $\lambda \in \mathbb{R}$, whence by taking $\lambda = 1$, the entire subalgebra $\mathcal{A}_{\text{trig}} \subseteq \mathcal{B}_\psi(X)$ belongs to the closure of $\mathcal{A}$ with respect to $\Vert \cdot \Vert_{\mathcal{B}_\psi(X)}$.
        
        Finally, by applying the first step to the point separating subalgebra $\mathcal{A}_{\text{trig}} \subseteq \mathcal{B}_\psi(X)$ which vanishes nowhere and consists of bounded maps, we conclude that $\mathcal{A}_{\text{trig}}$ is dense in $\mathcal{B}_\psi(X)$. However, since $\mathcal{A}_{\text{trig}}$ is contained in the closure of $\mathcal{A}$ with respect to $\Vert \cdot \Vert_{\mathcal{B}_\psi(X)}$, it follows that $\mathcal{A}$ is also dense in $\mathcal{B}_\psi(X)$.	
    \end{proof}
	
	\begin{remark}
        Theorem~\ref{ThmStoneWeierstrassBpsiR} can be applied to Bernstein's weighted polynomial approximation problem in \cite{bernstein24}, i.e.~$\Pol(\mathbb{R})$ is dense in $\mathcal{B}_\psi(\mathbb{R})$ with $\psi(x) = \exp\left( \vert x \vert^\gamma \right)$ for $\gamma \geq 1$.
    \end{remark}
	
    \subsection{Weighted vector-valued Stone-Weierstrass theorem}
    \label{SecVecWeightedSW}
    
    In this section, we generalize the weighted Stone-Weierstrass theorem to the vector-valued case, where $(Y,\Vert \cdot \Vert_Y)$ is a Banach space. To this end, we introduce for every $w \in \mathcal{B}_\psi(X;Y)$ the tilted weight function $\psi_w: X \rightarrow (0,\infty)$ defined by $\psi_w(x) = \frac{\psi(x)}{1+\Vert w(x) \Vert_Y}$, for $x \in X$.
    
    \begin{lemma}
    	\label{LemmaTiltedWeights}
    	For every $w \in \mathcal{B}_\psi(X;Y)$, the weight function $\psi_w: X \rightarrow (0,\infty)$ is admissible.
    \end{lemma}
    \begin{proof}
        Let $w \in \mathcal{B}_\psi(X;Y)$ and fix some $R > 0$. First, we show that $K_{w,R} := \psi_w^{-1}((0,R]) = \left\lbrace x \in X: \psi_w(x) \leq R \right\rbrace$ lies in a compact set $K_C := \psi^{-1}((0,C]) = \left\lbrace x \in X: \psi(x) \leq C \right\rbrace$ for some $C > 0$. Indeed, by Lemma~\ref{LemmaBpsiEquivChar}~\ref{LemmaBpsiEquivChar1}, there exists $\widetilde{R} > 0$ such that for every $x \in X$ with $\psi(x) > \widetilde{R}$ it holds that $\frac{\Vert w(x) \Vert_Y}{\psi(x)} \leq \frac{1}{2R}$. Moreover, for every fixed $x \in K_{w,R} := \psi_w^{-1}((0,R])$, we have $\frac{\psi(x)}{1+\Vert w(x) \Vert_Y} \leq R$ and thus $\psi(x) \leq R(1+\Vert w(x) \Vert_Y)$. Hence, if $\psi(x) > \widetilde{R}$, we obtain
        \begin{equation*}
        	\psi(x) \leq R(1+\Vert w(x) \Vert_Y) = R + R \Vert w(x) \Vert_Y \leq R + \frac{1}{2}\psi(x),
        \end{equation*}
    	and thus $\psi(x) \leq 2R$. Overall, by combining these two cases, i.e.~$\psi(x) \leq \widetilde{R}$ and $\psi(x) > \widetilde{R}$, we obtain $\psi(x) \leq \max(\widetilde{R}, 2R) =: C$, for all $x \in K_{w,R}$, which shows that $K_{w,R} \subseteq K_C$.
    	
    	Now, we show that $K_{w,R}$ is closed, which then implies that $K_{w,R}$ is compact as closed subset of the compact set $K_C$. Indeed, by using Lemma~\ref{LemmaBpsiEquivChar}~\ref{LemmaBpsiEquivChar1}, $\Vert w \Vert_Y \big\vert_{K_C}: K_C \rightarrow \mathbb{R}$ is continuous, whence $\big( \Vert w \Vert_Y \big\vert_{K_C} \big)^{-1}([a,b])$ is closed, for all $a \leq b$. Therefore, the set
    	\begin{equation}
    		\label{EqLemmaTiltedWeightsProof1}
    		\bigcap_{n \in \mathbb{N}} \bigcup_{m \in \mathbb{N}_0} \left( K_{R \left( 1 + \frac{m+1}{n} \right)} \cap \left( \Vert w \Vert_Y \big\vert_{K_C} \right)^{-1}\left(\left[\frac{m}{n},\frac{m+1}{n}\right]\right) \right).
    	\end{equation}
    	is closed, too, as the inner union is a finite union of closed sets. However, since $x \in K_{w,R}$ if and only if $x$ lies in the inner union of \eqref{EqLemmaTiltedWeightsProof1}, for all $n \in \mathbb{N}$, $K_{w,R}$ coincides with \eqref{EqLemmaTiltedWeightsProof1}.
    \end{proof}
    
    \begin{theorem}[Stone-Weierstrass on $\mathcal{B}_\psi(X;Y)$]
    	\label{ThmStoneWeierstrassBpsi}
    	Let $\mathcal{A} \subseteq \mathcal{B}_\psi(X)$ be a subalgebra that is point separating and nowhere vanishing of $\psi_w$-moderate growth, for all $w \in \mathcal{W}$, where $\mathcal{W} \subseteq \mathcal{B}_\psi(X;Y)$ is an $\mathcal{A}$-submodule such that $\mathcal{W}(x) := \lbrace w(x): w \in \mathcal{W} \rbrace$ is dense in $Y$, for all $x \in X$. Then, $\mathcal{W}$ is dense in $\mathcal{B}_\psi(X;Y)$.
    \end{theorem}
    \begin{proof}
    	Let $\mathcal{B}_b(X) \subseteq \mathcal{B}_\psi(X)$ be the vector subspace of bounded maps in $\mathcal{B}_\psi(X)$. Then, we first show that the closure of $\mathcal{W}$ with respect to $\Vert \cdot \Vert_{\mathcal{B}_\psi(X;Y)}$ is a $\mathcal{B}_b(X)$-submodule. For this purpose, we fix some $g \in \mathcal{B}_b(X)$ and assume that $\overline{w}: X \rightarrow Y$ belongs to the closure of $\mathcal{W}$ with respect to $\Vert \cdot \Vert_{\mathcal{B}_\psi(X;Y)}$. Moreover, we fix some $\varepsilon > 0$ and define $c = \sup_{x \in X} \vert g(x) \vert < \infty$. Then, there exists some $w \in \mathcal{W}$ with $\Vert \overline{w} - w \Vert_{\mathcal{B}_\psi(X;Y)} < \varepsilon/(1+2c)$, which implies that
    	\begin{equation}
    	    \label{EqThmStoneWeierstrassBpsiProof1}
            \Vert g \cdot \overline{w} - g \cdot w \Vert_{\mathcal{B}_\psi(X;Y)} = \sup_{x \in X} \frac{\vert g(x) \vert \Vert \overline{w}(x)-w(x) \Vert_Y}{\psi(x)} \leq c \sup_{x \in X} \Vert \overline{w} - w \Vert_{\mathcal{B}_\psi(X;Y)} < \frac{\varepsilon}{2}.
    	\end{equation}
    	Now, we observe that $\psi_w: X \rightarrow (0,\infty)$ is by Lemma~\ref{LemmaTiltedWeights} admissible and claim that $\mathcal{A} \subseteq \mathcal{B}_{\psi_w}(X)$. To see this, we fix some $a \in \mathcal{A}$ and recall that $K_{w,R} := \psi_w^{-1}((0,R])$. Then, by using that $X \setminus K_{w,R} \subseteq X \setminus K_R$ for any $R > 0$ (as $\psi_w(x) \leq \psi(x)$ for any $x \in X$, and thus $K_R \subseteq K_{w,R}$ for any $R > 0$) and Lemma~\ref{LemmaBpsiEquivChar}~\ref{LemmaBpsiEquivChar1} (with $a \in \mathcal{B}_\psi(X)$ and $a \cdot w \in \mathcal{W} \subseteq \mathcal{B}_\psi(X;Y)$ as $\mathcal{W}$ is an $\mathcal{A}$-submodule of $\mathcal{B}_\psi(X;Y)$), it follows that
        \begin{equation*}
            \begin{aligned}
                \lim_{R \rightarrow \infty} \sup_{x \in X \setminus K_{w,R}} \frac{\vert a(x) \vert}{\psi_w(x)} & \leq \lim_{R \rightarrow \infty} \sup_{x \in X \setminus K_R} \frac{\vert a(x) \vert (1 + \Vert w(x) \Vert_Y)}{\psi(x)} \\
                & = \lim_{R \rightarrow \infty} \sup_{x \in X \setminus K_R} \frac{\vert a(x) \vert}{\psi(x)} + \lim_{R \rightarrow \infty} \sup_{x \in X \setminus K_R} \frac{\Vert a(x) w(x) \Vert_Y}{\psi(x)} = 0.
            \end{aligned}
        \end{equation*}
    	In addition, by using that for every fixed $R > 0$ there exists a constant $C > 0$ (depending on $R > 0$ and $w \in \mathcal{W}$) such that $K_{w,R} \subseteq K_C$ (see the proof of Lemma~\ref{LemmaTiltedWeights}), we conclude from $a\vert_{K_C} \in C^0(K_C)$ that $a\vert_{K_{w,R}} \in C^0(K_{w,R})$. Hence, Lemma~\ref{LemmaBpsiEquivChar}~\ref{LemmaBpsiEquivChar2} implies that $a \in \mathcal{B}_{\psi_w}(X)$, which completes the proof of the inclusion $\mathcal{A} \subseteq \mathcal{B}_{\psi_w}(X)$. Thus, by applying the weighted real-valued Stone-Weierstrass theorem (Theorem~\ref{ThmStoneWeierstrassBpsiR}) on $\mathcal{B}_{\psi_w}(X)$, there exists some $a \in \mathcal{A}$ with $\Vert g-a \Vert_{\mathcal{B}_{\psi_w}(X)} < \varepsilon/2$, which implies that
    	\begin{equation}
        	\label{EqThmStoneWeierstrassBpsiProof2}
        	\begin{aligned}
        		\Vert g \cdot w - a \cdot w \Vert_{\mathcal{B}_\psi(X;Y)} & = \sup_{x \in X} \frac{\Vert g(x) w(x) - a(x) w(x) \Vert_Y}{\psi(x)} \leq \sup_{x \in X} \frac{\vert g(x) - a(x) \vert (1 + \Vert w(x) \Vert_Y)}{\psi(x)} \\
        		& = \Vert g-a \Vert_{\mathcal{B}_{\psi_w}(X)} < \frac{\varepsilon}{2}.
        	\end{aligned}
    	\end{equation}
    	By combining \eqref{EqThmStoneWeierstrassBpsiProof1} with \eqref{EqThmStoneWeierstrassBpsiProof2}, we conclude for $a \cdot w \in \mathcal{W}$ (as $\mathcal{W}$ is an $\mathcal{A}$-submodule) that
        \begin{equation*}
            \Vert g \cdot \overline{w} - a \cdot w \Vert_{\mathcal{B}_\psi(X;Y)} \leq \Vert g \cdot \overline{w} - g \cdot w \Vert_{\mathcal{B}_\psi(X;Y)} + \Vert g \cdot w - a \cdot w \Vert_{\mathcal{B}_\psi(X;Y)} < \frac{\varepsilon}{2} + \frac{\varepsilon}{2} = \varepsilon.
        \end{equation*}
        Since $\varepsilon > 0$ was chosen arbitrarily, this shows that $g \cdot \overline{w}: X \rightarrow Y$ belongs to the closure of $\mathcal{W}$ with respect to $\Vert \cdot \Vert_{\mathcal{B}_\psi(X;Y)}$. Hence, by using that $g \in \mathcal{B}_b(X)$ and $\overline{w}: X \rightarrow Y$ (belonging to the closure of $\mathcal{W}$ with respect $\Vert \cdot \Vert_{\mathcal{B}_\psi(X;Y)}$) were also chosen arbitrarily, we have shown that the closure of $\mathcal{W}$ with respect to $\Vert \cdot \Vert_{\mathcal{B}_\psi(X;Y)}$ is a $\mathcal{B}_b(X)$-submodule.
    	
    	Finally, by using that $\mathcal{B}_\psi(X;Y)$ is defined as the closure of $C^0_b(X;Y)$ with respect to $\Vert \cdot \Vert_{\mathcal{B}_\psi(X;Y)}$, it suffices to show the approximation of any $f \in C^0_b(X;Y)$ by an element $w \in \mathcal{W}$. Let $f \in C^0_b(X;Y)$, fix some $\varepsilon > 0$, define $b = \sup_{x \in X} \Vert f(x) \Vert_Y$ as well as $M = \max\left(1, \left( \inf_{x \in X} \psi(x) \right)^{-1}\right)$ and choose $R > \max(4(b+1)/\varepsilon,1)$. Then, by using that the restricted subalgebra $\mathcal{A}\vert_{K_R} \subseteq C^0(K_R)$ and the restricted $\mathcal{A}\vert_{K_R}$-submodule $\mathcal{W}\vert_{K_R} \subseteq C^0(K_R;Y)$ satisfy the assumptions of the classical vector-valued Stone-Weierstrass theorem (Theorem~\ref{ThmStoneWeierstrass}), we can apply Theorem~\ref{ThmStoneWeierstrass} to conclude that there exists some $w \in \mathcal{W}$ such that
    	\begin{equation}
        	\label{EqThmStoneWeierstrassBpsiProof3}
        	\Vert f - w \Vert_{C^0(K_R;Y)} = \sup_{x \in K_R} \Vert f(x) - w(x) \Vert_Y < \frac{\varepsilon}{4M}.
    	\end{equation}
    	Now, for $c = \sup_{x \in K_R} \Vert w(x) \Vert_Y$, \eqref{EqThmStoneWeierstrassBpsiProof3} implies $c \leq \varepsilon/(4M) + b$. Moreover, we define $X \ni x \mapsto g(x) := \frac{\min(1+c,1+\Vert w(x) \Vert_Y)}{1+ \Vert w(x) \Vert_Y} \in \mathbb{R}$ satisfying $g(x) = 1$, for all $x \in K_R$. Then, by using Lemma~\ref{LemmaBpsiEquivChar}~\ref{LemmaBpsiEquivChar1}, we conclude that $\Vert w \Vert_Y \big\vert_{K_{\widetilde{R}}} \in C^0(K_{\widetilde{R}})$ and thus $g\vert_{K_{\widetilde{R}}} \in C^0(K_{\widetilde{R}})$, for all $\widetilde{R} > 0$. Hence, by using that $g: X \rightarrow \mathbb{R}$ is bounded together with Lemma~\ref{LemmaBpsiEquivChar}~\ref{LemmaBpsiEquivChar2}, we have $g \in \mathcal{B}_b(X)$. In addition, $g \cdot w: X \rightarrow Y$ is bounded with $\sup_{x \in X} \Vert g(x) w(x) \Vert_Y \leq 1+c \leq 1 + b + \varepsilon/(4M)$. Hence, by using \eqref{EqThmStoneWeierstrassBpsiProof3} and $M, R \geq 1$, we conclude for $g \cdot w: X \rightarrow Y$ (belonging by the previous step to the closure of $\mathcal{W}$ with respect to $\Vert \cdot \Vert_{\mathcal{B}_\psi(X;Y)}$) that
    	\begin{equation*}
        	\begin{aligned}
            	\Vert f - g \cdot w \Vert_{\mathcal{B}_\psi(X;Y)} & \leq M \sup_{x \in K_R} \Vert f(x) - w(x) \Vert_Y + \sup_{x \in X \setminus K_R} \frac{\Vert f(x) \Vert_Y}{\psi(x)} + \sup_{x \in X \setminus K_R} \frac{\Vert g(x) w(x) \Vert_Y}{\psi(x)} \\
            	& < M \frac{\varepsilon}{4M} + \frac{\sup_{x \in X} \Vert f(x) \Vert_Y}{R} + \frac{\sup_{x \in X} \Vert g(x) w(x) \Vert_Y}{R} \\
            	& \leq \frac{\varepsilon}{4} + \frac{b}{R} + \frac{1+b+\varepsilon/(4M)}{R} < \varepsilon.
        	\end{aligned}
    	\end{equation*}
    	Since $\varepsilon > 0$ and $f \in C^0_b(X;Y)$ were chosen arbitrarily, $\mathcal{W}$ is dense in $\mathcal{B}_\psi(X;Y)$.
    \end{proof}
    
    \begin{remark}
        \label{RemBanachOutput2}
        Theorem~\ref{ThmStoneWeierstrassBpsi} can be generalized to the setting with a locally convex topological vector space $(Y,\tau_Y)$ as output space (see \cite{schaefer99} and Remark~\ref{RemBanachOutput}).
    \end{remark}
    
    \begin{remark}
    	\label{RemStoneWstrassLocPTSep}
		To ensure that $\mathcal{A} \subseteq \mathcal{B}_\psi(X)$ is point separating and nowhere vanishing of $\psi_w$-moderate growth in Theorem~\ref{ThmStoneWeierstrassBpsi}, for all $w \in \mathcal{W}$, it is enough to assume that $\mathcal{A} \subseteq \mathcal{B}_\psi(X)$ is point separating, nowhere vanishing, and consists only of bounded functions (see Remark~\ref{RemStoneWstrassBdedFct}) or that $\mathcal{A} \subseteq \mathcal{B}_{\sqrt{\psi}}(X)$ is point separating and nowhere vanishing of $\sqrt{\psi}$-moderate growth and $\mathcal{W} \subseteq \mathcal{B}_{\sqrt{\psi}}(X;Y)$.
    \end{remark}
    
    The weighted vector-valued Stone-Weierstrass theorem (Theorem~\ref{ThmStoneWeierstrassBpsi}) extends Nachbin's weighted approximation result in \cite[Theorem~6]{nachbin65} to functions that are globally not continuous, however \cite[Theorem~6]{nachbin65} allows for more general topologies by using several weights of the form $v := 1/\psi$ (in their notation). Moreover, Jo\~{a}o B.~Prolla extended in \cite{prolla71} the measure theoretic version of Erret Bishop in \cite{bishop61} to weighted spaces, see also \cite{chen19}.
	
	\section{Universal approximation on weighted spaces}
	\label{SecUATWS}
	
	We now consider the \emph{global} universal approximation property of neural networks between infinite dimensional spaces, where we apply the weighted Stone-Weierstrass theorem. As above and in contrast to the classical literature (see \cite{mcculloch43,cybenko89,hornik91}), we assume that the input space $(X,\psi)$ and the hidden layer space $(\mathbb{R},\psi_1)$ are both weighted spaces, whereas the output space is a Banach space $(Y,\Vert \cdot \Vert_Y)$. In this setting, we call neural networks between these infinite dimensional spaces functional input neural networks (FNNs).
	
	\subsection{Functional input neural networks}
	
	First, we introduce the infinite dimensional analogue of affine maps, which are applied in classical neural networks on the hidden layer.
	
	\begin{definition}
		\label{DefAddFam}
		A subset $\mathcal{H} \subseteq \mathcal{B}_\psi(X)$ is called an \emph{additive family (on $X$)} if
		\begin{enumerate}
			\item[\labeltext{(H1)}{H1}] $\mathcal{H}$ is closed under addition, i.e.~for every $h_1, h_2 \in \mathcal{H}$ it holds that $h_1 + h_2 \in \mathcal{H}$,
			\item[\labeltext{(H2)}{H2}] $\mathcal{H}$ is point separating, i.e.~for distinct $x_1, x_2 \in X$ there is $h \in \mathcal{H}$ with $h(x_1) \neq h(x_2)$.
			\item[\labeltext{(H3)}{H3}] $\mathcal{H}$ contains the constant functions, i.e.~for every $b \in \mathbb{R}$ we have $\left( x \mapsto b \right) \in \mathcal{H}$.
		\end{enumerate}
	\end{definition}
	
	\begin{example}
		\label{ExAddFamFinDim}
		For $(X,\psi)$ with the Euclidean space $X = \mathbb{R}^m$, an additive family is given by $\mathcal{H} = \left\lbrace x \mapsto a^\top x + b: a \in \mathbb{R}^m, \, b \in \mathbb{R} \right\rbrace$, which shows the analogy to linear maps and bias vectors, whereas $\mathcal{H} = \left\lbrace x \mapsto a^\top x + b: a \in \mathbb{N}_0^m, \, b \in \mathbb{R} \right\rbrace$ is a smaller additive family.
	\end{example}
	
	\begin{definition}
	    \label{DefAct}
		For a given admissible weight function $\psi_1: \mathbb{R} \rightarrow (0,\infty)$, a function $\rho \in C^0(\mathbb{R})$ is called \emph{$\psi_1$-activating} if the set of neural networks
		\begin{equation*}
			\mathcal{NN}^\rho_{\mathbb{R},\mathbb{R}} = \left\lbrace \mathbb{R} \ni z \mapsto \sum_{n=1}^N w_n \rho\left( a_n z + b_n \right) \in \mathbb{R}: N \in \mathbb{N}, \, a_n \in \mathbb{N}_0, \, w_n, b_n \in \mathbb{R} \right\rbrace
		\end{equation*}
		is a dense subset of $\mathcal{B}_{\psi_1}(\mathbb{R})$.
	\end{definition}
	
	In the next proposition we derive sufficient conditions for a function $\rho \in C^0(\mathbb{R})$ to be $\psi_1$-activating. The proof can be found in Appendix \ref{sec:proof}.
	
	\begin{proposition}
	    \label{PropAct}
	    Let $\psi_1: \mathbb{R} \rightarrow (0,\infty)$ be an admissible weight function and assume that $\rho \in C^0(\mathbb{R})$ satisfies $\lim_{z \rightarrow \pm \infty} \frac{\vert \rho(az+b) \vert}{\psi_1(z)} = 0$, for all $a \in \mathbb{N}_0$ and $b \in \mathbb{R}$. Then, $\rho \in C^0(\mathbb{R})$ is $\psi_1$-activating if (at least) one of the following sufficient conditions holds true:
	    \begin{enumerate}
	        \item[\labeltext{(A1)}{A1}] $\rho \in C^0(\mathbb{R})$ is $\psi_1$-discriminatory, i.e.~for every signed Radon measure $\mu \in \mathcal{M}_{\psi_1}(\mathbb{R})$ satisfying $\int_{\mathbb{R}} \rho(a z + b) \mu(dz) = 0$, for all $a \in \mathbb{N}_0$ and $b \in \mathbb{R}$, it follows that $\mu = 0$.
	        \item[\labeltext{(A2)}{A2}] $\rho \in C^0(\mathbb{R})$ is sigmoidal, i.e.~$\lim_{z \rightarrow -\infty} \rho(z) = 0$ and $\lim_{z \rightarrow \infty} \rho(z) = 1$.
	        \item[\labeltext{(A3)}{A3}] $\rho \in C^0(\mathbb{R})$ induces the tempered distribution $\left( g \mapsto T_\rho(g) := \int_{\mathbb{R}} \rho(s) g(s) ds \right) \in \mathcal{S}'(\mathbb{R};\mathbb{C})$ whose Fourier transform $\widehat{T_\rho} \in \mathcal{S}'(\mathbb{R};\mathbb{C})$ has a support with $0 \in \mathbb{R}$ as inner point\footnote{We refer to Section \ref{SecNotation} for the definition of $\widehat{T_\rho} \in \mathcal{S}'(\mathbb{R};\mathbb{C})$ and its support.}, and assume additionally that $\lim_{(z,b) \rightarrow (\pm \infty,b_0)} \frac{\vert \rho(az+b) \vert}{\psi_1(z)} = 0$, for all $a \in \mathbb{N}_0$ and $b_0 \in \mathbb{R}$.
	    \end{enumerate}
	\end{proposition}

    \begin{remark}
    	Let us point out the following remarks concerning Proposition~\ref{PropAct}:
    	\begin{enumerate}
    		\item The condition $\lim_{z \rightarrow \pm \infty} \frac{\vert \rho(az+b) \vert}{\psi_1(z)} = 0$ is satisfied if, e.g.,~$\rho \in C^0(\mathbb{R})$ is bounded.
    		\item For \ref{A1} and \ref{A2}, we adapt the original UATs of George Cybenko in \cite{cybenko89} to this weighted setting. For \ref{A3}, we generalize Norbert Wiener's Tauberian theorem, i.e.~that $\linspan\left\lbrace \mathbb{R} \ni z \mapsto \rho(z+b) \in \mathbb{R}: b \in \mathbb{R} \right\rbrace$ is dense in $L^1(\mathbb{R})$ if and only if the Fourier transform $\widehat{\rho}$ (in the classical sense) does not have any zeros (see \cite{wiener32}). To this end we follow the distributional extension in \cite{korevaar65}.
            \item Since $\mathcal{M}_{\psi_1}(\mathbb{R})$ is a subset of the set of all finite signed Radon measures, the condition~\ref{A1} is actually less restrictive than the usual ``discriminatory'' property (up to the assertion that $a \in \mathbb{N}_0$ instead of $a \in \mathbb{R}$).
    		\item  With respect to related results in the literature on neural networks (on compact domains and) with parameters $a_1,\ldots,a_N \in \mathbb{R}$ instead of $\mathbb{N}_0$, note that the condition $\rho \in C^0(\mathbb{R})$ being non-polynomial (see, e.g., \cite{leshno93,chen95}) is no longer sufficient. Indeed, if $\rho \in C^0(\mathbb{R})$ is non-polynomial, the support of $\widehat{T_\rho} \in \mathcal{S}'(\mathbb{R};\mathbb{C})$ contains by \cite[Examples~7.16]{rudin91} a non-zero point, which can be reached by continuously shifting the parameter $a_n \in \mathbb{R}$. However, this is no longer possible with $a_n \in \mathbb{N}_0$.
    		\item An example satisfying \ref{A3} is the ReLU function $\rho(z) = \max(z,0)$. Indeed, its Fourier transform $\widehat{T_\rho} \in \mathcal{S}'(\mathbb{R};\mathbb{C})$ is given by $\widehat{T}_\rho(g) = i \pi g'(0) - \text{p.v.} \int_{\mathbb{R}} \frac{g'(s)-g'(0)}{s} ds$ for all $g \in \mathcal{S}(\mathbb{R};\mathbb{C})$ (see \cite[Exercise~9.4.11 \& Equation~9.19]{folland92}), where ``p.v.'' denotes the principal value of the integral $\int_{\mathbb{R}} \frac{g'(s)-g'(0)}{s} ds$ (see \cite[p.~324]{folland92}).
    	\end{enumerate}
    \end{remark}
	
	\begin{definition}
		\label{DefANN}
		For a given additive family $\mathcal{H} \subseteq \mathcal{B}_\psi(X)$, a function $\rho \in C^0(\mathbb{R})$, and a vector subspace $\mathcal{L} \subseteq Y$, we define a \emph{functional input neural network (FNN)} $\varphi: X \rightarrow Y$ as
		\begin{equation}
			\label{EqDefANN1}
			\varphi(x) = \sum_{n=1}^N y_n \rho(h_n(x)),
		\end{equation}
		for $x \in X$, where $N \in \mathbb{N}$ denotes the \emph{number of neurons}, where $h_1,\ldots,h_N \in \mathcal{H}$ are the \emph{hidden layer maps}, and where $y_1,\ldots,y_N \in \mathcal{L}$ represent the \emph{linear readouts}. Moreover, we denote by $\mathcal{NN}^{\mathcal{H},\rho,\mathcal{L}}_{X,Y}$ the set of FNNs of the form \eqref{EqDefANN1}.
	\end{definition}
	
	\begin{figure}[ht]
		\centering
		\begin{tikzpicture}[
			inputnode/.style={circle, draw=green!60, fill=green!5, very thick, minimum size=5mm},
			hiddennode/.style={circle, draw=blue!60, fill=blue!5, very thick, minimum size=5mm},
			outputnode/.style={circle, draw=red!60, fill=red!5, very thick, minimum size=5mm},
			node distance=7mm,
			]
			% Neurons
			\node[inputnode] (x1) {};
			\node[hiddennode] (y2) [right = 2cm of x1] {};
			\node[hiddennode] (y1) [above of = y2] {};
			\node[hiddennode] (y3) [below of = y2] {};
			\node[outputnode] (o1) [right = 2cm of y2] {};
			
			% Lines
			\draw[shorten >=0.1cm,shorten <=0.1cm, ->] (x1.east) -- (y1.west);	
			\draw[shorten >=0.1cm,shorten <=0.1cm, ->] (x1.east) -- (y2.west);
			\draw[shorten >=0.1cm,shorten <=0.1cm, ->] (x1.east) -- (y3.west);
			\draw[shorten >=0.1cm,shorten <=0.1cm, ->] (y1.east) -- (o1.west);
			\draw[shorten >=0.1cm,shorten <=0.1cm, ->] (y2.east) -- (o1.west);
			\draw[shorten >=0.1cm,shorten <=0.1cm, ->] (y3.east) -- (o1.west);
			\draw[shorten >=0.2cm,shorten <=0.2cm, ->] (y3) to[out=-120,in=-60,loop] ();
			
			% Text
			\draw[] (0,1.0) node[anchor=center, align=center] {\footnotesize Input Layer};
			\draw[] (0,0.7) node[anchor=center, align=center] {\footnotesize $(X,\psi)$};
			\draw[] (2.53,1.6) node[anchor=center, align=center] {\footnotesize Hidden Layer};
			\draw[] (2.53,1.3) node[anchor=center, align=center] {\footnotesize $(\mathbb{R},\psi_1)$};
			\draw[] (5.1,1.0) node[anchor=center, align=center] {\footnotesize Output Layer};
			\draw[] (5.1,0.7) node[anchor=center, align=center] {\footnotesize $(Y,\Vert \cdot \Vert_Y)$};
			\draw[] (-0.9,0) node[anchor=center, align=center] {\footnotesize $x \in X$};
			\draw[] (6.2,0) node[anchor=center, align=center] {\footnotesize $\varphi(x) \in Y$};
			\draw[] (1.2,-0.8) node[anchor=center, align=center] {\footnotesize $\mathcal{H}$};
			\draw[] (3.8,-0.8) node[anchor=center, align=center] {\footnotesize $\mathcal{L}$};
			\draw[] (2.53,-1.26) node[anchor=center, align=center] {\footnotesize $\rho$};
		\end{tikzpicture}
		\caption{A functional input neural network $\varphi: X \rightarrow Y$ with additive family $\mathcal{H}$, activation function $\rho \in C^0(\mathbb{R})$, linear readout $\mathcal{L} \subseteq Y$, and $N = 3$ number of neurons.}
		\label{FigANN}
	\end{figure}
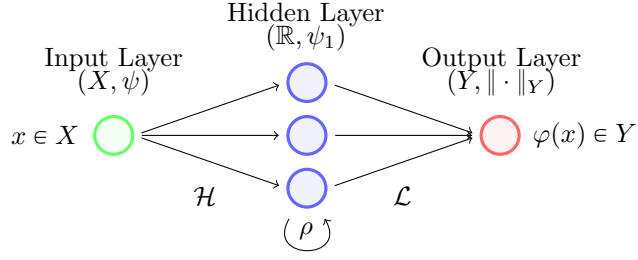
	
	\begin{remark}
	    Definition~\ref{DefANN} extends the notion of classical shallow neural networks between Euclidean spaces. Indeed, let $\varphi: \mathbb{R}^m \rightarrow \mathbb{R}^d$ be a classical neural network of the form
	    \begin{equation*}
		    \varphi(x) = W \rho(Ax + b) = \sum_{n=1}^N y_n \rho\left( a_n^\top x + b_n \right),
		\end{equation*}
	    for some $W \in \mathbb{R}^{d \times N}$, $A \in \mathbb{R}^{N \times m}$, and $b = (b_1,\ldots,b_N)^\top \in \mathbb{R}^N$, where $y_1,\ldots,y_N \in \mathbb{R}^d$ denote the columns of $W \in \mathbb{R}^{d \times N}$, and $a_1,\ldots,a_N \in \mathbb{R}^m$ are the rows of $A \in \mathbb{R}^{N \times m}$. 
	    Note that by a slight abuse of notation $\rho$ is applied componentwise to $Ax +b$ in the first equality. If we choose $\mathcal{H}$ as in Example~\ref{ExAddFamFinDim} and $\mathcal{L} = \mathbb{R}^d$, we can conclude that $\varphi \in \mathcal{NN}^{\mathcal{H},\rho,\mathcal{L}}_{\mathbb{R}^m,\mathbb{R}^d}$. 
	\end{remark}
	
	\begin{remark}
		\label{RemDeepANN}
		For two weighted spaces $(\mathbb{R},\psi_1)$ and $(\mathbb{R},\psi_2)$, let $\mathcal{H} \subseteq \mathcal{B}_\psi(X)$ be an additive family on $X$ and let $\rho_1, \rho_2 \in C^0(\mathbb{R})$ be $\psi_1$-activating and $\psi_2$-activating, respectively. Then, $\mathcal{NN}^{\mathcal{H},\rho_1,\mathbb{R}}_{X,\mathbb{R}}$ satisfies \ref{H1}-\ref{H3} and is thus (up to the requirement $\mathcal{NN}^{\mathcal{H},\rho_1,\mathbb{R}}_{X,\mathbb{R}} \subseteq \mathcal{B}_\psi(X)$) an additive family on $X$. Hence, a functional input neural network with two hidden layers $\varphi: X \rightarrow Y$ is of the form
		\begin{equation*}
			\varphi(x) = \sum_{n_2=1}^{N_2} y_{n_2} \rho_2 \left( \varphi_{n_2}(x) \right) = \sum_{n_2=1}^{N_2} y_{n_2} \rho_2 \left( \sum_{n_1=1}^{N_1} w_{n_2, n_1} \rho_1(h_{n_2, n_1}(x)) \right),
		\end{equation*}
		for $x \in X$, where $y_{n_2} \in Y$ and $\big( x \mapsto \varphi_{n_2}(x) = \sum_{n_1=1}^{N_1} w_{n_2, n_1} \rho_1(h_{n_2, n_1}(x)) \big) \in \mathcal{NN}^{\mathcal{H},\rho_1,\mathbb{R}}_{X,\mathbb{R}}$, with $w_{n_2, n_1} \in \mathbb{R}$ and $h_{n_2, n_1} \in \mathcal{H}$, see Figure~\ref{FigANN2}. By analogous concatenation, it is possible to construct deep functional input neural networks with finitely many hidden layers.
	\end{remark}
	
	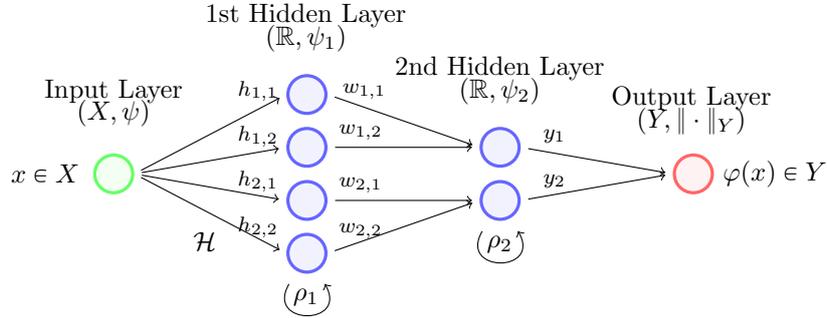
\begin{figure}[ht]
		\centering
		\begin{tikzpicture}[
			inputnode/.style={circle, draw=green!60, fill=green!5, very thick, minimum size=5mm},
			hiddennode/.style={circle, draw=blue!60, fill=blue!5, very thick, minimum size=5mm},
			outputnode/.style={circle, draw=red!60, fill=red!5, very thick, minimum size=5mm},
			invisiblenode/.style={circle, draw=blue!0, fill=blue!0, very thick, minimum size=5mm},
			node distance=7mm
			]
			% Neurons
			\node[inputnode] (x1) {};
			\node[invisiblenode] (y3) [right = 2cm of x1] {};
			\node[hiddennode] (y2) [above = -1.9mm of y3] {};
			\node[hiddennode] (y1) [above of = y2] {};
			\node[hiddennode] (y4) [below = -1.9mm of y3] {};
			\node[hiddennode] (y5) [below of = y4] {};
			\node[invisiblenode] (z2) [right = 2cm of y3] {};
			\node[hiddennode] (z1) [above = -1.9mm of z2] {};
			\node[hiddennode] (z3) [below = -1.9mm of z2] {};
			\node[outputnode] (o1) [right = 2cm of z2] {};
			
			% Lines
			\draw[shorten >=0.1cm,shorten <=0.1cm, ->] (x1.east) -- (y1.west);	
			\draw[shorten >=0.1cm,shorten <=0.1cm, ->] (x1.east) -- (y2.west);
			\draw[shorten >=0.1cm,shorten <=0.1cm, ->] (x1.east) -- (y4.west);
			\draw[shorten >=0.1cm,shorten <=0.1cm, ->] (x1.east) -- (y5.west);
			\draw[shorten >=0.1cm,shorten <=0.1cm, ->] (y1.east) -- (z1.west);
			\draw[shorten >=0.1cm,shorten <=0.1cm, ->] (y2.east) -- (z1.west);
			\draw[shorten >=0.1cm,shorten <=0.1cm, ->] (y4.east) -- (z3.west);
			\draw[shorten >=0.1cm,shorten <=0.1cm, ->] (y5.east) -- (z3.west);
			\draw[shorten >=0.1cm,shorten <=0.1cm, ->] (z1.east) -- (o1.west);
			\draw[shorten >=0.1cm,shorten <=0.1cm, ->] (z3.east) -- (o1.west);
			\draw[shorten >=0.2cm,shorten <=0.2cm, ->] (y5) to[out=-120,in=-60,loop] ();
			\draw[shorten >=0.2cm,shorten <=0.2cm, ->] (z3) to[out=-120,in=-60,loop] ();
			
			% Text
			\draw[] (0,1.1) node[anchor=center, align=center] {\footnotesize Input Layer};
			\draw[] (0,0.8) node[anchor=center, align=center] {\footnotesize $(X,\psi)$};
			\draw[] (2.53,2.1) node[anchor=center, align=center] {\footnotesize 1st Hidden Layer};
			\draw[] (2.53,1.8) node[anchor=center, align=center] {\footnotesize $(\mathbb{R},\psi_1)$};
			\draw[] (5.08,1.4) node[anchor=center, align=center] {\footnotesize 2nd Hidden Layer};
			\draw[] (5.08,1.1) node[anchor=center, align=center] {\footnotesize $(\mathbb{R},\psi_2)$};
			\draw[] (7.6,1.0) node[anchor=center, align=center] {\footnotesize Output Layer};
			\draw[] (7.6,0.7) node[anchor=center, align=center] {\footnotesize $(Y,\Vert \cdot \Vert_Y)$};
			\draw[] (-0.9,0) node[anchor=center, align=center] {\footnotesize $x \in X$};
			\draw[] (8.7,0) node[anchor=center, align=center] {\footnotesize $\varphi(x) \in Y$};
			\draw[] (1.2,-0.9) node[anchor=center, align=center] {\footnotesize $\mathcal{H}$};
			\draw[] (2.53,-1.63) node[anchor=center, align=center] {\footnotesize $\rho_1$};
			\draw[] (5.08,-0.93) node[anchor=center, align=center] {\footnotesize $\rho_2$};
			
			% Annotations
			\draw[] (1.9,1.1) node[anchor=center, align=center] {\scriptsize $h_{1,1}$};
			\draw[] (1.9,0.5) node[anchor=center, align=center] {\scriptsize $h_{1,2}$};
			\draw[] (1.9,-0.1) node[anchor=center, align=center] {\scriptsize $h_{2,1}$};
			\draw[] (1.9,-0.7) node[anchor=center, align=center] {\scriptsize $h_{2,2}$};
			\draw[] (3.3,1.1) node[anchor=center, align=center] {\scriptsize $w_{1,1}$};
			\draw[] (3.25,0.55) node[anchor=center, align=center] {\scriptsize $w_{1,2}$};
			\draw[] (3.25,-0.15) node[anchor=center, align=center] {\scriptsize $w_{2,1}$};
			\draw[] (3.25,-0.75) node[anchor=center, align=center] {\scriptsize $w_{2,2}$};
			\draw[] (5.8,0.5) node[anchor=center, align=center] {\scriptsize $y_1$};
			\draw[] (5.8,-0.1) node[anchor=center, align=center] {\scriptsize $y_2$};
		\end{tikzpicture}
		\caption{A deep two-hidden-layer FNN $\varphi: X \rightarrow Y$ with $(N_1,N_2) = (2,2)$.}
		\label{FigANN2}
	\end{figure}
	
	We now give some examples of additive families. If the weighted space $(X,\psi)$ is a vector space equipped with a norm, we can use the dual space of the normed vector space.
	
	\begin{lemma}
		\label{LemmaAddWeightFam}
      	Let $X_0 := (X,\Vert \cdot \Vert_0)$ be a normed vector space and let $X_0 \ni x \mapsto \psi(x) = \eta\left( \Vert x \Vert_1 \right) \in (0,\infty)$ be an admissible weight function, for some continuous and non-decreasing function $\eta: [0,\infty) \rightarrow (0,\infty)$ with $\lim_{r \rightarrow \infty} \frac{r}{\eta(r)} = 0$ and a norm $\Vert \cdot \Vert_1$ on $X$ such that the identity $\id: (X,\Vert \cdot \Vert_1) \hookrightarrow (X,\Vert \cdot \Vert_0) =: X_0$ is continuous. Then, 
       \begin{equation*}
           \mathcal{H} := \left\lbrace X \ni x \mapsto l(x) + b \in \mathbb{R}: l \in X_0^*, \, b \in \mathbb{R} \right\rbrace \subseteq \mathcal{B}_\psi(X_0)
       \end{equation*}
       is an additive family.
	\end{lemma}
	\begin{proof}
        First, we show that $\mathcal{H} \subseteq \mathcal{B}_\psi(X_0)$. For every fixed $l \in X_0^*$ and $b \in \mathbb{R}$, it holds that $\left( x \mapsto h(x) := l(x) + b \right) \in C^0(X_0)$. Moreover, since the identity $\id: X_1 := (X,\Vert \cdot \Vert_1) \hookrightarrow (X,\Vert \cdot \Vert_0) =: X_0$ is continuous, we have $\Vert \id \Vert_{L(X_1;X_0)} := \sup_{x \in X, \, \Vert x \Vert_1 \leq 1} \Vert x \Vert_0 < \infty$ and thus
        \begin{equation*}
            \begin{aligned}
                \lim_{R \rightarrow \infty} \sup_{x \in X \setminus K_R} \frac{\vert h(x) \vert}{\psi(x)} & \leq \lim_{R \rightarrow \infty} \sup_{x \in X \setminus K_R} \frac{\vert l(x) \vert}{\eta\left( \Vert x \Vert_1 \right)} + \lim_{R \rightarrow \infty} \sup_{x \in X \setminus K_R} \frac{\vert b \vert}{\eta\left( \Vert x \Vert_1 \right)} \\
                & \leq \Vert l \Vert_{X_0^*} \Vert \id \Vert_{L(X_1;X_0)} \lim_{R \rightarrow \infty} \sup_{x \in X \setminus K_R} \frac{\Vert x \Vert_1}{\eta\left( \Vert x \Vert_1 \right)} + \lim_{R \rightarrow \infty} \frac{\vert b \vert}{R} = 0.
            \end{aligned}
        \end{equation*}
     Hence, Lemma~\ref{LemmaBpsiEquivChar}~\ref{LemmaBpsiEquivChar2} shows that $h \in \mathcal{B}_\psi(X_0)$ and thus $\mathcal{H} \subseteq \mathcal{B}_\psi(X_0)$.

      Now, we show that $\mathcal{H}$ is an additive family on $X_0$. By definition, $\mathcal{H}$ is closed under addition and contains the constant functions. Concerning point separation, let $x_1, x_2 \in X_0$ be distinct. Since the dual space $X_0^*$ is by the Hahn-Banach theorem point separating on $X$, there exists some $l \in X_0^*$ such that $h := l \in \mathcal{H}$ satisfies $h(x_1) = l(x_1) \neq l(x_2) = h(x_2)$. This shows that $\mathcal{H}$ is point separating on $X_0$.
	\end{proof}

    \begin{remark}
        If $X_0 := (X,\Vert \cdot \Vert_0)$ is a Banach space such that $\psi: X_0 \rightarrow (0,\infty)$ is admissible, then $X_0$ is by Remark~\ref{RemWeight}~\ref{RemWeight3} finite dimensional. Hence, for an infinite dimensional weighted space $(X_0,\psi)$, the normed vector space $X_0 := (X,\Vert \cdot \Vert_0)$ is necessarily incomplete.
    \end{remark}
	
	With the additive family $\mathcal{H}$ from Lemma~\ref{LemmaAddWeightFam}, a corresponding FNN $\varphi: X \rightarrow Y$ over a weighted normed vector space $(X,\psi)$ is of the form
	\begin{equation*}
		\varphi(x) = \sum_{n=1}^N y_n \rho\left( l_n(x) + b_n \right),
	\end{equation*}
	where $l_1,\ldots,l_N \in X_0^*$, $b_1,\ldots,b_N \in \mathbb{R}$, $y_1,\ldots,y_N \in \mathcal{L} \subseteq Y$ and $\rho \in C^0(\mathbb{R})$. 
 
    On the other hand, if the weighted space $(X,\psi)$ is a dual Banach space equipped with the weak-$*$-topology as in Example~\ref{ExWeightedSpaces}~\ref{ExWeightedSpaceDual}, we obtain the following result.
	
	\begin{lemma}
		\label{LemmaAddWeightFam2}
      Let $X$ be a dual Banach space equipped with the weak-$*$-topology (and with predual $(E,\Vert \cdot \Vert_E)$) and let $X \ni x \mapsto \psi(x) := \eta\left( \Vert x \Vert_X \right) \in (0,\infty)$ be an admissible weight function, for some continuous and non-decreasing function $\eta: [0,\infty) \rightarrow (0,\infty)$ with $\lim_{r \rightarrow \infty} \frac{r}{\eta(r)} = 0$. Then,
      \begin{equation*}
          \mathcal{H} = \left\lbrace X \ni x \mapsto \langle x, e \rangle_{X \times E} + b \in \mathbb{R}: e \in E, \, b \in \mathbb{R} \right\rbrace \subseteq \mathcal{B}_\psi(X)
      \end{equation*}
      is an additive family.
	\end{lemma}
	\begin{proof}
      First, we show that $\mathcal{H} \subseteq \mathcal{B}_\psi(X)$. Indeed, for every fixed $e \in E$ and $b \in \mathbb{R}$, it follows that $x \mapsto h(x) := \langle x, e \rangle_{X \times E} + b \in C^0(X)$. Moreover, we observe that
        \begin{equation*}
            \begin{aligned}
                \lim_{R \rightarrow \infty} \sup_{x \in X \setminus K_R} \frac{\vert h(x) \vert}{\psi(x)} & \leq \lim_{R \rightarrow \infty} \sup_{x \in X \setminus K_R} \frac{\vert \langle x, e \rangle_{X \times E} \vert}{\eta\left( \Vert x \Vert_X \right)} + \lim_{R \rightarrow \infty} \sup_{x \in X \setminus K_R} \frac{\vert b \vert}{\eta\left( \Vert x \Vert_X \right)} \\
                & \leq \Vert e \Vert_E \lim_{R \rightarrow \infty} \sup_{x \in X \setminus K_R} \frac{\Vert x \Vert_X}{\eta\left( \Vert x \Vert_X \right)} + \lim_{R \rightarrow \infty} \frac{\vert b \vert}{R} = 0.
            \end{aligned}
        \end{equation*}
      Hence, Lemma~\ref{LemmaBpsiEquivChar}~\ref{LemmaBpsiEquivChar2} shows that $h \in \mathcal{B}_\psi(X)$ and thus $\mathcal{H} \subseteq \mathcal{B}_\psi(X)$.
        
	   Now, we show that $\mathcal{H}$ is an additive family on $X$. By definition, $\mathcal{H}$ is closed under addition and contains the constant functions. Concerning point separation, we can assume without loss of generality that $X \neq \lbrace 0 \rbrace$, which implies that $E \neq \lbrace 0 \rbrace$. Then, for distinct points $x_1, x_2 \in X$ there exists some $e \in E$ such that $\left( x \mapsto h(x) := \langle x, e \rangle_{X \times E} \right) \in \mathcal{H}$ satisfies $h(x_1) = \langle x_1, e \rangle_{X \times E} \neq \langle x_2, e \rangle_{X \times E}= h(x_2)$. Hence, $\mathcal{H}$ is point separating on $X$.
	\end{proof}

    With the additive family $\mathcal{H}$ from Lemma~\ref{LemmaAddWeightFam2}, a corresponding FNN $\varphi: X \rightarrow Y$ over a weighted dual Banach space $(X,\psi)$ equipped with the weak-$*$-topology is given by
	\begin{equation*}
		\varphi(x) = \sum_{n=1}^N y_n \rho\left( \langle x, e_n \rangle_{X \times E} + b_n \right),
	\end{equation*}
	where $e_1,\ldots,e_N \in E$, $b_1,\ldots,b_N \in \mathbb{R}$, $y_1,\ldots,y_N \in \mathcal{L}$, and $\rho \in C^0(\mathbb{R})$.
	
	\begin{example}
		\label{ExANN}
		Let us give some examples of additive families, where we revisit among others the examples from Section~\ref{SubsecMotEx}. Hereby, $\eta: [0,\infty) \rightarrow (0,\infty)$ is a continuous and non-decreasing function satisfying $\lim_{r \rightarrow \infty} \frac{r}{\eta(r)} = 0$.
		\begin{enumerate}
			\item\label{ExANN1} For the space of $\alpha$-H\"older continuous functions $X = C^\alpha(S)$ as in Example~\ref{ExWeightedSpaces}~\ref{ExWeightedSpaceHoelderWeaker} with weight function $\psi(x) = \eta\left( \Vert x \Vert_{C^\alpha(S)} \right)$, an additive family is given by
			\begin{equation*}
				\mathcal{H} = \left\lbrace C^\alpha(S) \ni x \mapsto \int_S x(s) \nu(ds) + b \in \mathbb{R}: \, 
                \begin{matrix}
					b \in \mathbb{R}, \\
					\nu: \mathcal{F}_S \rightarrow \mathbb{R}
				\end{matrix} 
                \right\rbrace \subseteq \mathcal{B}_\psi(C^\alpha(S)),
			\end{equation*}
			where $\nu: \mathcal{F}_S \rightarrow \mathbb{R}$ runs through the set of all finite signed Radon measures on $S$.
			
			\item\label{ExANN2} For the space of stopped $\alpha$-H\"older continuous paths $X = \Lambda^\alpha_T$ in Example~\ref{ExWeightedSpaces}~\ref{ExWeightedSpaceHoelderStoppedPaths} with weight function $\psi(t,x) = \eta\left( \Vert x \Vert_{C^\alpha([0,T];\mathbb{R}^d)} \right)$, an additive family is given by
			\begin{equation*}
				\mathcal{H} = \left\lbrace \Lambda_T^\alpha \ni (t,x) \mapsto at + \sum_{i=1}^d \int_0^t \phi_i(s) x_i(s) ds + b \in \mathbb{R}: \, 
                \begin{matrix}
					a,b \in \mathbb{R} \\
					\phi_i \in \mathcal{NN}^{\widetilde{\rho}}_{\mathbb{R},\mathbb{R}}
				\end{matrix} 
				\right\rbrace \subseteq \mathcal{B}_\psi(\Lambda^\alpha_T),
			\end{equation*}
			where $\mathcal{NN}^{\widetilde{\rho}}_{\mathbb{R},\mathbb{R}}$ is the set of classical neural networks $\phi: \mathbb{R} \rightarrow \mathbb{R}$ with activation function $\widetilde{\rho}: \mathbb{R} \rightarrow \mathbb{R}$ such that $\mathcal{NN}^{\widetilde{\rho}}_{\mathbb{R},\mathbb{R}} \big\vert_{[0,T]}$ is a dense subset of $L^1([0,T])$.
   
            \item\label{ExANN3} For $X = L^\infty(\Omega,\mathcal{F},\mathbb{P})$ with probability space $(\Omega,\mathcal{F},\mathbb{P})$ as in Example~\ref{ExWeightedSpaces}~\ref{ExWeightedSpaceLpPredual} and weight function $\psi(x) = \eta\left( \Vert x \Vert_{L^\infty(\Omega,\mathcal{F},\mathbb{P})} \right)$, an additive family is given by
            \begin{equation*}
                \mathcal{H} = \left\lbrace L^\infty(\Omega,\mathcal{F},\mathbb{P}) \ni x \mapsto \mathbb{E}[z x] + b \in \mathbb{R}: \, 
                \begin{matrix}
                    b \in \mathbb{R}, \\
					z \in L^1(\Omega,\mathcal{F},\mathbb{P})
				\end{matrix} 
                \right\rbrace \subseteq \mathcal{B}_\psi(L^\infty(\Omega,\mathcal{F},\mathbb{P})).
            \end{equation*}

            \item\label{ExANN4} For the space of signed Radon measures $X = \mathcal{M}_{\psi_\Omega}(\Omega)$ on a weighted space $(\Omega,\psi_\Omega)$ as in Example~\ref{ExWeightedSpaces}~\ref{ExWeightedSpaceMeasures} with weight function $\psi(x) = \eta\big( \Vert x \Vert_{\mathcal{M}_{\psi_\Omega}(\Omega)} \big)$, an additive family is given by
    		\begin{equation*}
    			\mathcal{H} = \left\lbrace \mathcal{M}_{\psi_\Omega}(\Omega) \ni x \mapsto \int_\Omega f(\omega) x(d\omega) + b \in \mathbb{R}: \, 
                \begin{matrix}
					b \in \mathbb{R}, \\
					f \in \mathcal{B}_{\psi_\Omega}(\Omega)
				\end{matrix}
                \right\rbrace \subseteq \mathcal{B}_\psi(\mathcal{M}_{\psi_\Omega}(\Omega)).
    		\end{equation*}
            This can be used to learn a map with a (probability) measure as input.
		\end{enumerate}
	\end{example}
	\begin{proof}
		For \ref{ExANN1}, we first observe that $C^\alpha(S) \subseteq C^0(S)$ and that every continuous linear functional $l \in (C^\alpha(S),\Vert \cdot \Vert_{C_0(S)})^*$ can be extended by the Hahn-Banach theorem to a continuous linear functional $l_0 \in C^0(S)^*$. Since $C^0(S)^*$ is by the Riesz representation theorem isomorphic to the space of finite signed Radon measures $\nu: \mathcal{F}_S \rightarrow \mathbb{R}$, the conclusion follows from Lemma~\ref{LemmaAddWeightFam}.
		
		For \ref{ExANN2}, we cannot apply Lemma~\ref{LemmaAddWeightFam} nor Lemma~\ref{LemmaAddWeightFam2} directly. In order to show that $\mathcal{H} \subseteq \mathcal{B}_\psi(\Lambda^\alpha_T)$, we first observe for every fixed $a,b \in \mathbb{R}$ and $\phi_1,\ldots,\phi_d \in \mathcal{NN}^{\widetilde{\rho}}_{\mathbb{R},\mathbb{R}}$ that $\big( (t,x) \mapsto h(t,x) := at + \sum_{i=1}^d \int_0^t \phi_i(s) x_i(s) ds + b \big) \in C^0(\Lambda^\alpha_T)$. Moreover, by using that $\Vert x \Vert_{C^0([0,T];\mathbb{R}^d)} \leq \left( 1 + T^\alpha \right) \Vert x \Vert_{C^\alpha([0,T];\mathbb{R}^d)}$ for any $(t,x) \in \Lambda^\alpha_T$ (see \eqref{EqThmHoelderEmbeddingProof1}), we conclude that
        \begin{equation*}
            \begin{aligned}
                & \lim_{R \rightarrow \infty} \sup_{(t,x) \in \Lambda^\alpha_T \setminus K_R} \frac{\vert h(t,x) \vert}{\psi(x)} \leq \lim_{R \rightarrow \infty} \sup_{(t,x) \in \Lambda^\alpha_T \setminus K_R} \frac{\vert at \vert + \vert b \vert}{\eta\left( \Vert x \Vert_{C^\alpha([0,T];\mathbb{R}^d)} \right)} \\
                & \quad\quad\quad\quad\quad\quad\quad\quad\quad\quad\quad\quad + \lim_{R \rightarrow \infty} \sup_{(t,x) \in \Lambda^\alpha_T \setminus K_R} \frac{\left\vert \sum_{i=1}^d \int_0^t \phi_i(s) x_i(s) ds \right\vert}{\eta\left( \Vert x \Vert_{C^\alpha([0,T];\mathbb{R}^d)} \right)} \\
                & \quad\quad \leq \lim_{R \rightarrow \infty} \frac{\vert a \vert T + \vert b \vert}{R} + d \max_{i=1,\ldots,d} \Vert \phi_i \Vert_{L^1([0,T])} \lim_{R \rightarrow \infty} \sup_{(t,x) \in \Lambda^\alpha_T \setminus K_R} \frac{\left( 1 + T^\alpha \right) \Vert x \Vert_{C^\alpha([0,T];\mathbb{R}^d)}}{\eta\left( \Vert x \Vert_{C^\alpha([0,T];\mathbb{R}^d)} \right)} = 0.
            \end{aligned}
        \end{equation*}
        Hence, Lemma~\ref{LemmaBpsiEquivChar}~\ref{LemmaBpsiEquivChar2} shows that $h \in \mathcal{B}_\psi(\Lambda^\alpha_T)$ and thus $\mathcal{H} \subseteq \mathcal{B}_\psi(\Lambda^\alpha_T)$. In addition, we observe that $\mathcal{H}$ is closed under addition and contains the constant functions. For point separation, let $(t_1,x_1^{t_1}), (t_2,x_2^{t_2}) \in \Lambda^\alpha_T$ be distinct. If $t_1 \neq t_2$, then $\left( (t,x) \mapsto t \right) \in \mathcal{H}$ separates $(t_1,x_1^{t_1})$ from $(t_2,x_2^{t_2})$. Otherwise, if $t_1 = t_2 =: t$ and the paths $x_1^t, x_2^t \in C^0([0,T];\mathbb{R}^d)$ differ, there exist some $i = 1,\ldots,d$ and $s \in [0,t)$ such that $x_{1,i}(s) \neq x_{2,i}(s)$. Recall that we assume that $\mathcal{NN}^{\widetilde{\rho}}_{\mathbb{R},\mathbb{R}}\big\vert_{[0,T]}$ is a dense subset of $L^1([0,T])$ and note that $C^0([0,T]) \subsetneq L^\infty([0,T]) \cong L^1([0,T])^*$. Hence, there exists some $\phi_i \in \mathcal{NN}^{\widetilde{\rho}}_{\mathbb{R},\mathbb{R}}$ such that $\int_0^t \phi_i(s) x_{1,i}(s) ds \neq \int_0^t \phi_i(s) x_{2,i}(s) ds$. Thus, by defining $\phi_j := 0 \in \mathcal{NN}^{\widetilde{\rho}}_{\mathbb{R},\mathbb{R}}$ for $j \neq i$, the map $\big( (t,x) \mapsto h(t,x) := \sum_{i=1}^d \int_0^t \phi_i(s) x_i(s) ds \big) \in \mathcal{H}$ separates $(t_1,x_1^{t_1})$ from $(t_2,x_2^{t_2})$.

        For \ref{ExANN3} and \ref{ExANN4}, we have $L^1(\Omega,\mathcal{F},\mathbb{P})^* \cong L^\infty(\Omega,\mathcal{F},\mathbb{P})$ and $\mathcal{B}_{\psi_\Omega}(\Omega)^* \cong \mathcal{M}_{\psi_\Omega}(\Omega)$, respectively (see \cite[Theorem~2.4]{doersek10} for the latter). Hence, Lemma~\ref{LemmaAddWeightFam2} shows the results.
	\end{proof}
	
	These results give an overview of the broad variety of functional input neural networks, for which we show the universal approximation property in the following subsection.
	
	\subsection{Global universal approximation}
	
	Universal approximation results were first proved for neural networks between Euclidean spaces in \cite{cybenko89,hornik91} using functional analytical arguments such as the Hahn-Banach theorem and Fourier theory. Later, further universal approximation results were found for different types of activation functions, such as in \cite{pinkus99}, while the results of e.g.~\cite{barron93,candes98,bgkp17} are concerned with proving quantitative approximation rates under more restrictive assumptions on the involved functions.
	
	In the following, we prove a universal approximation result for functional input neural networks by lifting the universal approximation property of one-dimensional neural networks encoded in Definition~\ref{DefAct} to this infinite dimensional setting. For this purpose, we assume that the input space $(X,\psi)$ and the hidden layer space $(\mathbb{R},\psi_1)$ are weighted spaces, and the output space is a Banach space $(Y,\Vert \cdot \Vert_Y)$.
	
	\begin{theorem}[Universal approximation on $\mathcal{B}_\psi(X;Y)$]
		\label{ThmUAT}
		Let $\mathcal{H} \subseteq \mathcal{B}_\psi(X)$ be an additive family with $\sup_{x \in X} \frac{\psi_1(h(x))}{\psi(x)} < \infty$, for all $h \in \mathcal{H}$, let $\rho \in C^0(\mathbb{R})$ be $\psi_1$-activating, and let $\mathcal{L} \subseteq Y$ be a vector subspace that is dense in $Y$. Then, $\mathcal{NN}^{\mathcal{H},\rho,\mathcal{L}}_{X,Y}$ is dense in $\mathcal{B}_\psi(X;Y)$, i.e.~for every $f \in \mathcal{B}_\psi(X;Y)$ and $\varepsilon > 0$ there exists some $\varphi \in \mathcal{NN}^{\mathcal{H},\rho,\mathcal{L}}_{X,Y}$ such that
		\begin{equation*}
			\Vert f - \varphi \Vert_{\mathcal{B}_\psi(X;Y)} = \sup_{x \in X} \frac{\Vert f(x) - \varphi(x) \Vert_Y}{\psi(x)} < \varepsilon.
		\end{equation*}
	\end{theorem}
	\begin{proof}
        First, we show that $\mathcal{NN}^{\mathcal{H},\rho,\mathcal{L}}_{X,Y} \subseteq \mathcal{B}_\psi(X;Y)$. Since $\mathcal{NN}^{\mathcal{H},\rho,\mathcal{L}}_{X,Y}$ is defined as the linear span of maps of the form $X \ni x \mapsto y \rho(h(x)) \in Y$, with $h \in \mathcal{H}$ and $y \in \mathcal{L}$, and the mapping $\mathcal{B}_\psi(X) \ni g \mapsto y \cdot g \in \mathcal{B}_\psi(X;Y)$ is continuous and well-defined, for all $y \in Y$, it suffices to show that $\rho \circ h \in \mathcal{B}_\psi(X)$, for all $h \in \mathcal{H}$. To this end, we fix some $h \in \mathcal{H}$ as well as $\varepsilon > 0$, and define the constant $C_h := \sup_{x \in X} \frac{\psi_1(h(x))}{\psi(x)} + 1 > 0$. Then, by using that $\rho \in \mathcal{B}_{\psi_1}(\mathbb{R})$ (as $\rho \in C^0(\mathbb{R})$ is $\psi_1$-activating and thus $(z \mapsto y \rho(az+b)) \in \mathcal{NN}^\rho_{\mathbb{R},\mathbb{R}} \subseteq \mathcal{B}_{\psi_1}(\mathbb{R})$ with $a = y = 1$ and $b = 0$) there exists by definition of $\mathcal{B}_{\psi_1}(\mathbb{R})$ some $\widetilde{\rho} \in C^0_b(\mathbb{R})$ such that
        \begin{equation}
            \label{EqThmUATProof1}
            \Vert \rho - \widetilde{\rho} \Vert_{\mathcal{B}_{\psi_1}(\mathbb{R})} = \sup_{z \in \mathbb{R}} \frac{\left\vert \rho(z) - \widetilde{\rho}(z) \right\vert}{\psi_1(z)} < \frac{\varepsilon}{2C_h}.
        \end{equation}
        Moreover, since $h \in \mathcal{H} \subseteq \mathcal{B}_\psi(X)$, i.e.~$h\vert_{K_R} \in C^0(K_R)$, for all $R > 0$ (see Lemma~\ref{LemmaBpsiEquivChar}~\ref{LemmaBpsiEquivChar1}), and $\widetilde{\rho} \in C^0_b(\mathbb{R})$, we observe that $\left( \widetilde{\rho} \circ h \right)\vert_{K_R} \in C^0(K_R)$, for all $R > 0$, and that $\widetilde{\rho} \circ h: X \rightarrow \mathbb{R}$ is bounded, whence $\widetilde{\rho} \circ h \in \mathcal{B}_\psi(X)$ by Lemma~\ref{LemmaBpsiEquivChar}~\ref{LemmaBpsiEquivChar2}. Thus, by definition of $\mathcal{B}_\psi(X)$, there exists some $g \in C^0_b(X)$ such that $\Vert \widetilde{\rho} \circ h - g \Vert_{\mathcal{B}_\psi(X)} < \varepsilon/2$. Hence, by combining this with \eqref{EqThmUATProof1}, it follows that
        \begin{equation*}
            \begin{aligned}
                \Vert \rho \circ h - g \Vert_{\mathcal{B}_\psi(X)} & \leq \Vert \rho \circ h - \widetilde{\rho} \circ h \Vert_{\mathcal{B}_\psi(X)} + \Vert \widetilde{\rho} \circ h - g \Vert_{\mathcal{B}_\psi(X)} \\
                & < \sup_{x \in X} \frac{\left\vert \rho(h(x)) - \widetilde{\rho}(h(x)) \right\vert}{\psi(x)} + \frac{\varepsilon}{2} \\
                & \leq \sup_{x \in X} \frac{\psi_1(h(x))}{\psi(x)} \sup_{x \in X} \frac{\left\vert \rho(h(x)) - \widetilde{\rho}(h(x)) \right\vert}{\psi_1(h(x))} + \frac{\varepsilon}{2} \\
                & \leq C_h \sup_{z \in \mathbb{R}} \frac{\left\vert \rho(z) - \widetilde{\rho}(z) \right\vert}{\psi_1(z)} + \frac{\varepsilon}{2} \\
                & < C_h \frac{\varepsilon}{2C_h} + \frac{\varepsilon}{2} = \varepsilon.
            \end{aligned}
        \end{equation*}
        Since $\varepsilon > 0$ was chosen arbitrarily and $g \in C^0_b(X)$, this shows that $\rho \circ h \in \mathcal{B}_\psi(X)$.
        
		Next, we define $\mathcal{A} = \linspan\left( \left\lbrace \cos(h(\cdot)): h \in \mathcal{H} \right\rbrace \cup \left\lbrace \sin(h(\cdot)): h \in \mathcal{H} \right\rbrace \right)$. Then, by using that $\mathcal{H} \subseteq \mathcal{B}_\psi(X)$, i.e.~$h\vert_{K_R} \in C^0(K_R)$, for all $h \in \mathcal{H}$ and $R > 0$ (see Lemma~\ref{LemmaBpsiEquivChar}~\ref{LemmaBpsiEquivChar1}), and that $\cos, \sin \in C^0_b(\mathbb{R})$, we observe that $\cos(h(\cdot))\vert_{K_R}, \sin(h(\cdot))\vert_{K_R} \in C^0(K_R)$, for all $h \in \mathcal{H}$ and $R > 0$, and that $\cos(h(\cdot)), \sin(h(\cdot)): X \rightarrow \mathbb{R}$ are bounded, for all $h \in \mathcal{H}$, whence $\mathcal{A} \subseteq \mathcal{B}_\psi(X)$ by Lemma~\ref{LemmaBpsiEquivChar}~\ref{LemmaBpsiEquivChar2}. Moreover, by using the trigonometric identities
        \begin{equation}
            \label{EqThmUATProof3}
            \begin{aligned}
    			\cos(s) \cos(t) & = \frac{1}{2} \big( \cos(s-t)+\cos(s+t) \big), \\
    			\cos(s) \sin(t) & = \frac{1}{2} \big( \sin(s+t)-\sin(s-t) \big), \quad \text{and} \\
                \sin(s) \sin(t) & = \frac{1}{2} \big( \cos(s-t)-\cos(s+t) \big),
    		\end{aligned}
        \end{equation}
        for any $s,t \in \mathbb{R}$, we observe that $\mathcal{A}$ is a subalgebra of $\mathcal{B}_\psi(X)$. In addition, we define the vector subspace $\mathcal{W} = \linspan\left\lbrace X \ni x \mapsto a(x) y_0 \in Y: \, a \in \mathcal{A}, \, y_0 \in \mathcal{L} \right\rbrace \subseteq \mathcal{B}_\psi(X;Y)$, which is by definition an $\mathcal{A}$-submodule. In the following step, we first show that $\mathcal{W}$ is contained in the closure of $\mathcal{NN}^{\mathcal{H},\rho,\mathcal{L}}_{X,Y}$ with respect to $\Vert \cdot \Vert_{\mathcal{B}_\psi(X;Y)}$, from which the conclusion follows by applying the weighted vector-valued Stone-Weierstrass theorem (Theorem~\ref{ThmStoneWeierstrassBpsi}); see below for the details.
	    
	    Next, we prove that $\mathcal{W}$ is contained in the closure of $\mathcal{NN}^{\mathcal{H},\rho,\mathcal{L}}_{X,Y}$ with respect to $\Vert \cdot \Vert_{\mathcal{B}_\psi(X;Y)}$. To this end, we fix some $\varepsilon > 0$, $h \in \mathcal{H}$, and $y_0 \in \mathcal{L}$, and define the constant $C_h := \sup_{x \in X} \frac{\psi_1(h(x))}{\psi(x)} + 1 > 0$. Then, by using that $\rho \in C^0(\mathbb{R})$ is $\psi_1$-activating, i.e.~$\mathcal{NN}^\rho_{\mathbb{R},\mathbb{R}}$ is a dense subset of $\mathcal{B}_{\psi_1}(\mathbb{R})$, there exists some $\varphi_1 \in \mathcal{NN}^\rho_{\mathbb{R},\mathbb{R}}$ of the form $\mathbb{R} \ni z \mapsto \varphi_1(z) = \sum_{n=1}^N w_n \rho(a_n z + b_n) \in \mathbb{R}$, with $N \in \mathbb{N}$, $a_n \in \mathbb{N}_0$, and $b_n, w_n \in \mathbb{R}$, such that
		\begin{equation}
			\label{EqThmUATProof4}
			\Vert \cos - \varphi_1 \Vert_{\mathcal{B}_{\psi_1}(\mathbb{R})} = \sup_{z \in \mathbb{R}} \frac{\vert \cos(z) - \varphi_1(z) \vert}{\psi_1(z)} < \frac{\varepsilon}{C_h (1+\Vert y_0 \Vert_Y)}.
		\end{equation}
		Now, for the function $\left( x \mapsto w(x) := \cos(h(x)) y_0 \right) \in \mathcal{W}$, we define the corresponding FNN $\left( x \mapsto \varphi(x) := y_0 \varphi_1(h(x)) \right) \in \mathcal{NN}^{\mathcal{H},\rho,\mathcal{L}}_{X,Y}$. Hereby, we have used that for every $n = 1,\ldots,N$ it holds that $a_n h + b_n \in \mathcal{H}$ (as $\mathcal{H}$ is closed under addition and contains the constants) and that $w_n y_0 \in \mathcal{L}$ (as $\mathcal{L} \subseteq Y$ is a vector subspace). Then, we conclude from \eqref{EqThmUATProof4} that
		\begin{equation*}
		    \begin{aligned}
		        \Vert w - \varphi \Vert_{\mathcal{B}_\psi(X;Y)} & = \sup_{x \in X} \frac{\Vert y_0 \cos(h(x)) - y_0 \varphi_1(h(x)) \Vert_Y}{\psi(x)} \\
		        & \leq \Vert y_0 \Vert_Y \sup_{x \in X} \frac{\psi_1(h(x))}{\psi(x)} \sup_{x \in X} \frac{\vert \cos(h(x)) - \varphi_1(h(x)) \vert}{\psi_1(h(x))} \\
		        & \leq C_h \Vert y_0 \Vert_Y \sup_{z \in \mathbb{R}} \frac{\vert \cos(z) - \varphi_1(z) \vert}{\psi_1(z)} \\
		        & < C_h \Vert y_0 \Vert_Y \frac{\varepsilon}{C_h (1+\Vert y_0 \Vert_Y)} \leq \varepsilon.
		    \end{aligned}
		\end{equation*}
		Since $\varepsilon > 0$ was chosen arbitrarily, the map $\left( x \mapsto w(x) = \cos(h(x)) y_0 \right) \in \mathcal{W}$ belongs to the closure of $\mathcal{NN}^{\mathcal{H},\rho,\mathcal{L}}_{X,Y}$ with respect to $\Vert \cdot \Vert_{\mathcal{B}_\psi(X;Y)}$, which holds analogously true for $\left( x \mapsto \sin(h(x)) y_0 \right) \in \mathcal{W}$. Hence, we conclude from the triangle inequality that the entire $\mathcal{A}$-submodule $\mathcal{W}$ is contained in the closure of $\mathcal{NN}^{\mathcal{H},\rho,\mathcal{L}}_{X,Y}$ with respect to $\Vert \cdot \Vert_{\mathcal{B}_\psi(X;Y)}$.
		
		Now, we apply the weighted vector-valued Stone-Weierstrass theorem (Theorem~\ref{ThmStoneWeierstrassBpsi}) to show that $\mathcal{NN}^{\mathcal{H},\rho,\mathcal{L}}_{X,Y}$ is dense in $\mathcal{B}_\psi(X;Y)$. To this end, we observe that $\mathcal{A} \subseteq \mathcal{B}_\psi(X)$ is point separating, vanishes nowhere, and consists only of bounded maps, which shows that $\mathcal{A} \subseteq \mathcal{B}_\psi(X)$ is point separating and nowhere vanishing of $\psi_w$-moderate growth, for all $w \in \mathcal{W}$ (see Remark~\ref{RemStoneWstrassBdedFct}). In addition, by using that $\mathcal{L}$ is dense in $Y$, the set $\mathcal{W}(x) = \left\lbrace w(x): w \in \mathcal{W} \right\rbrace = \mathcal{L}$ is dense in $Y$, for all $x \in X$. Hence, by applying the weighted vector-valued Stone-Weierstrass theorem (Theorem~\ref{ThmStoneWeierstrassBpsi}), we conclude that $\mathcal{W}$ is dense in $\mathcal{B}_\psi(X;Y)$. Finally, since $\mathcal{W}$ is by the previous step contained in the closure of $\mathcal{NN}^{\mathcal{H},\rho,\mathcal{L}}_{X,Y}$ with respect to $\Vert \cdot \Vert_{\mathcal{B}_\psi(X;Y)}$, $\mathcal{NN}^{\mathcal{H},\rho,\mathcal{L}}_{X,Y}$ is also dense in $\mathcal{B}_\psi(X;Y)$.
	\end{proof}
	
	\begin{remark}
	    The global universal approximation result in Theorem~\ref{ThmUAT} can be generalized to locally convex topological vector space as output spaces (see \cite{schaefer99} for a study of these infinite dimensional vector spaces and also Remark~\ref{RemBanachOutput} and Remark~\ref{RemBanachOutput2}).
	\end{remark}
	
	\subsection{Global universal approximation for non-anticipative functionals}
	\label{SecUANAF}
	
	The space of stopped paths was originally introduced in \cite{dupire09,cont10,contfournie13} to generalize F\"ollmer's pathwise Ito calculus \cite{foellmer81} to non-anticipative functionals. By introducing a special type of FNNs, so-called \emph{non-anticipative functional input neural networks}, we now present the universal approximation result for this kind of functionals.
 
    We consider again the space of stopped $\alpha$-H\"older continuous paths given by
	\begin{equation*}
		\Lambda^\alpha_T = \lbrace (t,x^t): (t,x) \in [0,T] \times C^\alpha([0,T];\mathbb{R}^d) \rbrace \cong \left([0,T] \times C^\alpha([0,T];\mathbb{R}^d)\right) / \sim,
	\end{equation*}
	which is equipped with the metric $d_\infty((t,x),(s,y)) = \vert t-s \vert + \sup_{u \in [0,T]} \Vert x^t(u) - y^s(u) \Vert$, for $(t,x), (s,y) \in \Lambda^\alpha_T$. Then, we define $\psi: \Lambda^\alpha_T \rightarrow (0,\infty)$ again by $\psi(t,x) = \eta\left( \Vert x \Vert_{C^\alpha([0,T];\mathbb{R}^d)} \right)$, for $(t,x) \in \Lambda^\alpha_T$, where the function $\eta: [0,\infty) \rightarrow (0,\infty)$ is continuous, non-decreasing, and unbounded. Hence, $(\Lambda^\alpha_T,\psi)$ is a weighted space (see Example~\ref{ExWeightedSpaces}~\ref{ExWeightedSpaceHoelderStoppedPaths}).
	
	Now, we introduce non-anticipative functionals as defined in \cite{dupire09,cont10,contfournie13}, which are measurable maps $f: \Lambda^\alpha_T \rightarrow \mathbb{R}$ and can therefore include the history of a path.
	
	\begin{definition}
        \label{DefNAFNN}
		A map $f: \Lambda^\alpha_T \rightarrow \mathbb{R}$ is called a \emph{non-anticipative functional} if $f: \Lambda^\alpha_T \rightarrow \mathbb{R}$ is measureable. Moreover, $f: \Lambda^\alpha_T \rightarrow \mathbb{R}$ is called \emph{continuous} if $f: \Lambda^\alpha_T \rightarrow \mathbb{R}$ is continuous with respect to the metric $d_\infty$.
	\end{definition}
	
	By Lemma~\ref{LemmaBpsiEquivChar}~\ref{LemmaBpsiEquivChar2}, we observe that a non-anticipative functional $f: \Lambda^\alpha_T \rightarrow \mathbb{R}$ belongs to $\mathcal{B}_\psi(\Lambda^\alpha_T)$ if $f: \Lambda^\alpha_T \rightarrow \mathbb{R}$ is continuous and it holds that
	\begin{equation*}
	    \lim_{R \rightarrow \infty} \sup_{x \in \Lambda^\alpha_T \setminus K_R} \frac{\vert f(x) \vert}{\psi(x)} = 0.
	\end{equation*}
	Hence, we can apply the global universal approximation result below at least for continuous non-anticipative functionals, whose growth is controlled by $\psi: \Lambda^\alpha_T \rightarrow (0,\infty)$.
	
	Motivated by Example~\ref{ExANN}~\ref{ExANN2}, we introduce a specific type of functional input neural network defined on $\Lambda^\alpha_T$, which preserve the non-anticipative behaviour.
	
	\begin{definition}
        \label{DefFuncNN}
		For $\rho, \widetilde{\rho} \in C^0(\mathbb{R})$ and classical neural networks $\phi_{n,i} \in \mathcal{NN}^{\widetilde{\rho}}_{\mathbb{R},\mathbb{R}}$, we define a \emph{non-anticipative functional input neural network} $\varphi: \Lambda^\alpha_T \rightarrow \mathbb{R}$ by
		\begin{equation}
			\label{EqDefFuncNN}
			\varphi(t,x) = \sum_{n=1}^N y_n \rho\left( a_n t + \sum_{i=1}^d \int_0^t \phi_{n,i}(s) x_i(s) ds + b_n \right),
		\end{equation}
		for $(t,x) \in \Lambda^\alpha_T$, where $N \in \mathbb{N}$, $a_1 \ldots a_N, b_1, \ldots b_N, y_1 \ldots y_N \in \mathbb{R}$. We denote by $\mathcal{FN}^{\rho,\widetilde{\rho}}_{\Lambda^\alpha_T}$ the set of functional input neural networks of the form \eqref{EqDefFuncNN}.
	\end{definition}
	
	Note that a non-anticipative FNN $\varphi: \Lambda^\alpha_T \rightarrow \mathbb{R}$ satisfies $\varphi(t,x) = \varphi(t,x^t)$ by construction, for all $(t,x) \in \Lambda^\alpha_T$, which shows that $\varphi: \Lambda^\alpha_T \rightarrow \mathbb{R}$ is indeed a non-anticipative functional. Moreover, $\varphi: \Lambda^\alpha_T \rightarrow \mathbb{R}$ is continuous, as $\rho, \widetilde{\rho} \in C^0(\mathbb{R})$.
	
	We now apply the global universal approximation result in Theorem~\ref{ThmUAT} to show that every $f \in \mathcal{B}_\psi(\Lambda^\alpha_T)$ can be approximated by some non-anticipative FNN.
	
	\begin{corollary}[Universal approximation on $\mathcal{B}_\psi(\Lambda^\alpha_T)$]
		\label{CorUAT}
        For a continuous non-decreasing unbounded function $\eta: [0,\infty) \rightarrow (0,\infty)$, let $\Lambda^\alpha_T \ni (t,x) \mapsto \psi(t,x) := \eta\left( \Vert x \Vert_{C^\alpha([0,T];\mathbb{R}^d)} \right) \in (0,\infty)$ be an admissible weight function. Moreover, let $\psi_1: \mathbb{R} \rightarrow (0,\infty)$ be an admissible weight function such that $\psi_1: \mathbb{R} \rightarrow (0,\infty)$ is even, $\psi_1\vert_{[0,\infty)}: [0,\infty) \rightarrow (0,\infty)$ is increasing, and $\sup_{r > 0} \frac{\psi_1(C+cr)}{\eta(r)} < \infty$ for all $C,c > 0$. In addition, let $\rho \in C^0(\mathbb{R})$ be $\psi_1$-activating and let $\widetilde{\rho}: \mathbb{R} \rightarrow \mathbb{R}$ be such that $\mathcal{NN}^{\widetilde{\rho}}_{\mathbb{R},\mathbb{R}} \big\vert_{[0,T]}$ is a dense subset of $L^1([0,T])$.\footnote{We refer to e.g.~\cite[Theorem~4]{cybenko89}, \cite[Theorem~1]{hornik91}, and \cite[Proposition~1]{leshno93} for sufficient conditions.} Then, $\mathcal{FN}^{\rho,\widetilde{\rho}}_{\Lambda^\alpha_T}$ is dense in $\mathcal{B}_\psi(\Lambda^\alpha_T)$, i.e.~for every $f \in \mathcal{B}_\psi(\Lambda^\alpha_T)$ and $\varepsilon > 0$ there exists $\varphi \in \mathcal{FN}^{\rho,\widetilde{\rho}}_{\Lambda^\alpha_T}$ such that
		\begin{equation*}
			\Vert f - \varphi \Vert_{\mathcal{B}_\psi(\Lambda^\alpha_T)} < \varepsilon.
		\end{equation*}
	\end{corollary}
	\begin{proof}
		The result follows from Theorem~\ref{ThmUAT} with additive family $\mathcal{H} \subseteq \mathcal{B}_\psi(\Lambda^\alpha_T)$ as in Example~\ref{ExANN}~\ref{ExANN2}, activation function $\rho \in C^0(\mathbb{R})$, and $\mathcal{L} = \mathbb{R}$ for the linear readouts. Hence, we only need to show that $\sup_{(t,x) \in \Lambda^\alpha_T} \frac{\psi_1(h(t,x))}{\psi(t,x)} < \infty$, for all $h \in \mathcal{H}$. Indeed, for every $h \in \mathcal{H}$ of the form $h(t,x) = at + \sum_{i=1}^d \int_0^t \phi_i(s) x_i(s) ds + b \in \mathcal{H}$, with $a,b \in \mathbb{R}$ and $\phi_1,\ldots,\phi_d \in \mathcal{NN}^{\widetilde{\rho}}_{\mathbb{R},\mathbb{R}}$, we use that $\psi_1: \mathbb{R} \rightarrow (0,\infty)$ is even and the inequality $\Vert x \Vert_{C^0([0,T];\mathbb{R}^d)} \leq \left( 1 + T^\alpha \right) \Vert x \Vert_{C^\alpha([0,T];\mathbb{R}^d)}$ for any $(t,x) \in \Lambda^\alpha_T$ (see \eqref{EqThmHoelderEmbeddingProof1}) to conclude that
        \begin{equation*}
            \begin{aligned}
                & \sup_{(t,x) \in \Lambda^\alpha_T} \frac{\psi_1(h(t,x))}{\psi(t,x)} = \sup_{(t,x) \in \Lambda^\alpha_T} \frac{\psi_1(\vert h(t,x) \vert)}{\psi(t,x)} \leq \sup_{(t,x) \in \Lambda^\alpha_T} \frac{\psi_1\left( \vert at \vert + \left\vert \sum_{i=1}^d \int_0^t \phi_i(s) x_i(s) ds \right\vert + \vert b \vert \right)}{\eta\left( \Vert x \Vert_{C^\alpha([0,T];\mathbb{R}^d)} \right)} \\
                & \quad\quad \leq \sup_{(t,x) \in \Lambda^\alpha_T} \frac{\psi_1\left( \vert a \vert T + \vert b \vert + d \max_{i=1,\ldots,d} \Vert \phi_i \Vert_{L^1([0,T])} \Vert x \Vert_{C^0([0,T];\mathbb{R}^d)} \right)}{\eta\left( \Vert x \Vert_{C^\alpha([0,T];\mathbb{R}^d)} \right)} \\
                & \quad\quad \leq \sup_{(t,x) \in \Lambda^\alpha_T} \frac{\psi_1\left( \vert a \vert T + \vert b \vert + \left( 1 + T^\alpha \right) d \max_{i=1,\ldots,d} \Vert \phi_i \Vert_{L^1([0,T])} \Vert x \Vert_{C^\alpha([0,T];\mathbb{R}^d)} \right)}{\eta\left( \Vert x \Vert_{C^\alpha([0,T];\mathbb{R}^d)} \right)} \\
                & \quad\quad \leq \sup_{r > 0} \frac{\psi_1\left( \vert a \vert T + \vert b \vert + \left( 1 + T^\alpha \right) d \max_{i=1,\ldots,d} \Vert \phi_i \Vert_{L^1([0,T])} r \right)}{\eta(r)} < \infty.
            \end{aligned}
        \end{equation*}
        Hence, the conclusion follows from Theorem~\ref{ThmUAT}.
	\end{proof}
	
	In Section~\ref{SecNE}, we provide some numerical examples, which illustrate how the theoretical approximation result in Corollary~\ref{CorUAT} can be applied to learn a non-anticipative functional by a non-anticipative FNN. In particular for applications in finance, non-anticipative FNNs can be used within the framework of Stochastic Portfolio Theory (SPT) introduced by R. Fernholz in \cite{fernholz02} to learn optimal path-dependent portfolios, see \cite{schmocker22}.
	
	\section{Universal approximation on weighted spaces for linear functions of the signature}
	\label{SecSM}
	
	In this section, we give another application of the weighted real-valued Stone-Weierstrass theorem (Theorem~\ref{ThmStoneWeierstrassBpsiR}), which is related to the theory of rough paths introduced by Terry Lyons in \cite{lyons07} (we also refer to the textbooks \cite{friz10,friz2020course}). More precisely, we will show that a path space functional can be approximated \emph{globally} by using linear functions of the paths' signature, a notion which goes back to the work of Chen in \cite{chen57}. This is in contrast to the usual approximation results on compact sets of the path space, e.g.~for finite variation paths or for continuous functions depending on the whole signature (instead of the rough path), see \cite[Theorem~3.1]{levin13}, \cite[Theorem~1]{kiraly19}, \cite[Proposition~4.5]{lyons20} and \cite[Section~3]{cuchiero2022signature}. Moreover, we refer to \cite{BHRS:21,moeller22} for the situation involving not just the approximation of a functional at a fixed \emph{final value} $T > 0$, but a uniform approximation on the whole time interval $[0,T]$.
	
	In order to make the current article self-contained, we briefly introduce the main ideas of the signature. In many cases, the data is ordered sequentially and can be seen as a discretization of a path $X: [0,T] \rightarrow Z$ with values in some Banach space $(Z,\Vert \cdot \Vert_Z)$, where we assume in the following that $Z = \mathbb{R}^d$ is finite dimensional. For example, the path $X: [0,T] \rightarrow Z$ may describe the air temperature at a given location (depending on time), the motion of a pen on a paper, or the stock price processes in a financial market. Now, given a continuous path $X: [0,T] \rightarrow \mathbb{R}^d$ of finite variation, we may define its \emph{signature} (at terminal time) as the infinite collection of iterated integrals of the path, i.e.
	\begin{equation*}
		\mathbb{X}_T := \left( 1, \int_{0 < u_1 < T} dX_{u_1}, \ldots, \int_{0 < u_1 < \ldots < u_N < T} dX_{u_1} \otimes \cdots \otimes dX_{u_N}, \ldots \right) \in T((\mathbb{R}^d)),
	\end{equation*}
	where $T((\mathbb{R}^d)) := \prod_{n=0}^{\infty} (\mathbb{R}^d)^{\otimes n}$ denotes the extended tensor algebra (see Section~\ref{SubsecConcepts}). Notice that the iterated integrals are well-defined as the path $X: [0,T] \rightarrow \mathbb{R}^d$ is of finite variation. For more irregular paths (e.g.~$\alpha$-H\"older continuous paths $X: [0,T] \rightarrow \mathbb{R}^d$), one relies on the theory of rough paths to define its signature.
	
	In this case, a linear function of the signature (at terminal time $T$) is a map of the form $\mathbb{X}_T \mapsto l(\mathbb{X}_T) := \sum_{I \in \lbrace 1,\ldots,d \rbrace^N} a_I \langle e_I, \mathbb{X}_T \rangle$, for some $N \in \mathbb{N}_0$, where $a_I \in \mathbb{R}$, and where $e_I \in T((\mathbb{R}^d))$ denotes the $I$-th basis element of $T((\mathbb{R}^d))$ so that $\langle e_I, \mathbb{X}_T \rangle$ returns the $I$-th entry of the signature  $\mathbb{X}_T$. In the following, we show that these linear functions of the signature are able to approximate any sufficiently regular path functional.
	
	In order to derive the universal approximation results for linear functions of the signature, one relies in particular on two key properties. First, that the signature at the terminal time determines uniquely the path (up to so-called tree-like equivalences, see \cite{hambly10,boedihardjo16}), which ensures point separation. Second, that every polynomial of the signature can be realized as a linear function via the so-called shuffle product, which yields the algebra property of linear functions of the signature. While the classical Stone-Weierstrass theorem then only guarantees the approximation on compact subsets of paths (e.g.~with respect to a certain $p$-variation metric), there have already been some attempts to go beyond this rather restrictive setting. As already mentioned in the introduction, \cite{chevyrev22} used the strict topology originating from \cite{giles71} to formulate a global approximation result, however not with the ``true'' signature but with a bounded normalization implying that the tractability properties of the signature (like computing the expected signature analytically as e.g.~done in~\cite{cuchiero2023signature} for generic classes of stochastic processes) are lost.

	In contrast, our weighted space setup allows to work with the ``true'' signature and thus to preserve its tractability properties, while we can still provide a global approximation result. To this end, we consider the path space of $\alpha$-H\"older continuous paths as a weighted space, which we either equip with an $0 \leq \alpha' < \alpha$-Hölder topology or the weak-$*$-topology as considered in Appendix~\ref{AppBanachPredual}.
 
	\subsection{Concepts and notation related to the signature of (rough) paths}
	\label{SubsecConcepts}
	
	We recall in this section important notions needed to deal with the signature of rough paths. For some fixed $d \in \mathbb{N}$, the \emph{extended tensor algebra (over $\mathbb{R}^d$)} is defined as
    \begin{equation*}
        T((\mathbb{R}^d)) := \prod_{n=0}^{\infty} (\mathbb{R}^d)^{\otimes n}
    \end{equation*}
    where $(\mathbb{R}^d)^{\otimes n}$ denotes the $n$-th tensor product of $\mathbb{R}^d$ (with convention $(\mathbb{R}^d)^{\otimes 0} := \mathbb{R}$), which is equipped with the Euclidean norm $\Vert \mathbf{a} \Vert_{(\mathbb{R}^d)^{\otimes n}} := \big( \sum_{i_1,\ldots,i_n=1}^d \vert \mathbf{a}_{i_1,\ldots,i_n} \vert^2 \big)^{1/2}$ for $\mathbf{a} := (\mathbf{a}_{i_1,\ldots,i_n})_{i_1,\ldots,i_n\in\lbrace 1,\ldots,d \rbrace}$. We endow $T((\mathbb{R}^d))$ with the standard addition ``$+$'', tensor multiplication ``$\otimes$'', and scalar multiplication ``$\cdot$'' defined by
	\begin{equation*}
		\mathbf{a} + \mathbf{b} = \left( \mathbf{a}^{(n)} + \mathbf{b}^{(n)} \right)_{n=0}^\infty, \quad \mathbf{a} \otimes \mathbf{b} = \left( \sum_{k=0}^n \mathbf{a}^{(n-k)} \otimes \mathbf{b}^{(k)} \right)_{n=0}^\infty, \quad \lambda \cdot \mathbf{a} = \left( \lambda \mathbf{a}^{(n)} \right)_{n=0}^\infty,
	\end{equation*}
	for $\mathbf{a} = (\mathbf{a}^{(n)})_{n=0}^\infty \in T((\mathbb{R}^d))$, $\mathbf{b} = (\mathbf{b}^{(n)})_{n=0}^\infty \in T((\mathbb{R}^d))$, and $\lambda \in \mathbb{R}$. Moreover, for $N \in \mathbb{N}$, the \emph{truncated tensor algebra} is defined as
    \begin{equation*}
        T^N(\mathbb{R}^d) := \bigoplus_{n=0}^N (\mathbb{R}^d)^{\otimes n},
    \end{equation*}
    where addition ``$+$'', tensor multiplication ``$\otimes$'', and scalar multiplication are defined by
	\begin{equation*}
		\mathbf{a} + \mathbf{b} = \left( \mathbf{a}^{(n)} + \mathbf{b}^{(n)} \right)_{n=0}^N, \quad \mathbf{a} \otimes \mathbf{b} = \left( \sum_{k=0}^n \mathbf{a}^{(n-k)} \otimes \mathbf{b}^{(k)} \right)_{n=0}^N, \quad \lambda \cdot \mathbf{a} = \left( \lambda \mathbf{a}^{(n)} \right)_{n=0}^N,
	\end{equation*}
	for $\mathbf{a} = (\mathbf{a}^{(n)})_{n=0}^N \in T^N(\mathbb{R}^d)$, $\mathbf{b} = (\mathbf{b}^{(n)})_{n=0}^N \in T^N(\mathbb{R}^d)$, and $\lambda \in \mathbb{R}$. In addition, we equip $T^N(\mathbb{R}^d)$ with the norm $\Vert \mathbf{a} \Vert_{T^N(\mathbb{R}^d)} := \max_{n=0,\ldots,N} \Vert \mathbf{a}^{(n)} \Vert_{(\mathbb{R}^d)^{\otimes n}}$, for $\mathbf{a} = (\mathbf{a}^{(n)})_{n=0}^N \in T^N(\mathbb{R}^d)$. Furthermore, we consider the subsets $T^N_0(\mathbb{R}^d)$ and $T^N_1(\mathbb{R}^d)$ of $T^N(\mathbb{R}^d)$ consisting of elements $\mathbf{a} \in T^N(\mathbb{R}^d)$ with $\mathbf{a}^{(0)} = 0$ and $\mathbf{a}^{(0)} = 1$, respectively.

    In order to adapt the Lie group point of view on weakly geometric rough paths, we observe that $T_1^N(\mathbb{R}^d)$ is a Lie group under $\otimes$, truncated at level $N$, with unit element $\mathbf{1} := (1,0,\ldots,0) \in T_1^N(\mathbb{R}^d)$. Moreover, we define the free step-$N$ nilpotent Lie algebra $\mathfrak{g}^N(\mathbb{R}^d)$ via
	\begin{equation*}
	    \mathfrak{g}^N(\mathbb{R}^d) = \mathbb{R}^d \oplus [\mathbb{R}^d,\mathbb{R}^d] \oplus \cdots \oplus [\mathbb{R}^d,[\mathbb{R}^d,\ldots[\mathbb{R}^d,\mathbb{R}^d]]],
	\end{equation*}
	where $(\mathbf{g},\mathbf{h}) \mapsto [\mathbf{g},\mathbf{h}] := \mathbf{g} \otimes \mathbf{h} - \mathbf{h} \otimes \mathbf{g} \in T_0^N(\mathbb{R}^d)$ denotes the Lie bracket, for $\mathbf{g},\mathbf{h} \in T^N(\mathbb{R}^d)$. In addition, we define the exponential map by
	\begin{equation*}
        T_0^N(\mathbb{R}^d) \ni \mathbf{b} \quad \mapsto \quad \exp^{(N)}(\mathbf{b}) := \mathbf{1} + \sum_{n=1}^N \frac{1}{n!} \mathbf{b}^{\otimes n} \in T_1^N(\mathbb{R}^d),
	\end{equation*}
    whose inverse is the logarithm
    \begin{equation*}
        T_1^N(\mathbb{R}^d) \ni \mathbf{1} + \mathbf{a} \quad \mapsto \quad \log^{(N)}(\mathbf{1} + \mathbf{a}) := \sum_{n=1}^N \frac{(-1)^{n+1}}{n} \mathbf{a}^{\otimes n} \in T_0^N(\mathbb{R}^d).
	\end{equation*}
    Then, the image $G^N(\mathbb{R}^d) := \exp^{(N)}\left( \mathfrak{g}^N(\mathbb{R}^d) \right)$ is a subgroup of $T_1^N(\mathbb{R}^d)$ with respect to $\otimes$, which is called the \emph{free step-$N$ nilpotent Lie group} (see also \cite[Definition~7.30]{friz10}).
 
    Moreover, on $G^N(\mathbb{R}^d)$, we consider the Carnot-Carath\'eodory norm $\Vert \cdot \Vert_{cc}$ defined by
	\begin{equation*}
	    \Vert \mathbf{g} \Vert_{cc} := \inf\left\lbrace \int_0^T \Vert dX_t \Vert: X \in C^{1-var}([0,T];\mathbb{R}^d) \text{ such that } \mathbb{X}^N_T = \mathbf{g} \right\rbrace,
	\end{equation*}
	for $\mathbf{g} \in G^N(\mathbb{R}^d)$, see \cite[Theorem~7.32]{friz10}. Hereby, $\mathbb{X}^N$ denotes the \emph{signature truncated at level $N$} of a path $X \in C^0([0,T];\mathbb{R}^d)$ of finite variation, which is given by
	\begin{equation*}
	    \mathbb{X}^N_{s,t} := \left( 1, \int_{s < u < t} dX_u, \ldots, \int_{s < u_1 < \ldots < u_N < t} dX_{u_1} \otimes \cdots \otimes dX_{u_N} \right) \in T^N(\mathbb{R}^d),
	\end{equation*}
	for $0 \leq s \leq t \leq T$. Moreover, for $s=0$, we just write $\mathbb{X}^N_t$. The Carnot-Carath\'eodory norm $\Vert \cdot \Vert_{cc}$ induces a metric via
	\begin{equation*}
	    d_{cc}(\mathbf{g}, \mathbf{h}) := \left\Vert \mathbf{g}^{-1} \otimes \mathbf{h} \right\Vert_{cc},
	\end{equation*}
	for $\mathbf{g}, \mathbf{h} \in G^N(\mathbb{R}^d)$. We equip $G^N(\mathbb{R}^d)$ with this Carnot-Carath\'eodory metric $d_{cc}$, which turns $(G^N(\mathbb{R}^d),d_{cc})$ into a metric space. In addition, we define the set of group-like elements as $G(\mathbb{R}^d) := \lbrace \mathbf{g} \in T((\mathbb{R}^d)): (\mathbf{g}^{(0)},\ldots,\mathbf{g}^{(N)}) \in G^N(\mathbb{R}^d) \text{ for all } N \in \mathbb{N} \rbrace$.

    For a basis $\lbrace e_1,\ldots,e_d \rbrace$ of $\mathbb{R}^d$ and a multi-index\footnote{Note that this differs from the usual notion of multi-indices $\alpha := (\alpha_1,\ldots,\alpha_m) \in \mathbb{N}^m_0$.} $I = (i_1,\ldots,i_n) \in \lbrace 1,\ldots,d \rbrace^n$, we use the notations $\vert I \vert = n$ and $e_I := e_{i_1} \otimes \cdots \otimes e_{i_n} \in (\mathbb{R}^d)^{\otimes n}$. Then, $(e_I)_{\vert I \vert = n}$ forms a basis of $(\mathbb{R}^d)^{\otimes n}$. Moreover, we denote by $e_\varnothing$ the basis element of $(\mathbb{R}^d)^{\otimes 0} = \mathbb{R}$.
    
    For two multi-indices $I = (i_1,\ldots,i_n) \in \lbrace 1,\ldots,d \rbrace^n$ and $J = (j_1,\ldots,j_m) \in \lbrace 1,\ldots,d \rbrace^m$, $m,n \in \mathbb{N}$, and two indices $a,b \in \lbrace 1,\ldots,d \rbrace$, we define the \emph{shuffle product} $\shuffle$ recursively by
	\begin{equation*}
		(I,a) \shuffle (J,b) := ((I \shuffle (J,b)),a) + (((I,a) \shuffle J), b),
	\end{equation*}
	with $I \shuffle \varnothing = \varnothing \shuffle I = I$, where $(I,a)$ denotes the concatenation of multi-indices. In addition, we denote by $\langle e_I, \mathbf{g} \rangle$ the component of any $\mathbf{g} \in G((\mathbb{R}^d))$ corresponding to the $I$-th basis element of $T((\mathbb{R}^d))$. Then, for every $\mathbf{g} \in G((\mathbb{R}^d))$ and multi-indices $I,J$, it holds that
	\begin{equation}
		\label{EqDefShuffleProperty}
		\langle e_I, \mathbf{g} \rangle \, \langle e_J, \mathbf{g} \rangle = \langle e_I \shuffle e_J, \mathbf{g} \rangle,
	\end{equation}
	where $e_I \shuffle e_J := \sum_{k=1}^K e_{I_k}$, with $K$ and $(I_k)_{k=1,\ldots,K}$ being determined via $I \shuffle J = \sum_{k=1}^K I_k$.	

    \subsection{Global universal approximation of $\alpha$-H\"older rough paths}
    \label{sec:hoelder}

    In this section, we consider weakly geometric $\alpha$-H\"older rough paths, which are defined as $\alpha$-H\"older continuous paths with values in $G^{\lfloor 1/\alpha \rfloor}(\mathbb{R}^d)$, see \cite[Definition~9.15~(iii)]{friz10}.
	
	\begin{definition}
	    For $\alpha \in (0,1]$, a continuous path $\mathbf{X}: [0,T] \rightarrow G^{\lfloor 1/\alpha \rfloor}(\mathbb{R}^d)$ of the form
	    \begin{equation*}
	        [0,T] \ni t \quad \mapsto \quad \mathbf{X}_t := (1,X_t,\mathbb{X}^{(2)}_t,\ldots,\mathbb{X}^{(\lfloor 1/\alpha \rfloor)}_t) \in G^{\lfloor 1/\alpha \rfloor}(\mathbb{R}^d)
	    \end{equation*}
	    with $\mathbf{X}_0 = \mathbf{1} := (1, 0 \ldots, 0) \in G^{\lfloor 1/\alpha \rfloor}(\mathbb{R}^d)$ is called a \emph{weakly geometric $\alpha$-H\"older rough path} if the $\alpha$-H\"older norm
	    \begin{equation*}
	        \Vert \mathbf{X} \Vert_{cc,\alpha} := \sup_{s,t \in [0,T] \atop s < t} \frac{d_{cc}(\mathbf{X}_s,\mathbf{X}_t)}{\vert s-t \vert^\alpha}
	    \end{equation*}
	    is finite. We denote by $C^\alpha_0([0,T];G^{\lfloor 1/\alpha \rfloor}(\mathbb{R}^d))$ the space of such \emph{weakly geometric $\alpha$-H\"older rough paths}, which we equip with the metric
        \begin{equation}
            \label{EqDefWGRPHoelder1}
            d_{cc,\alpha'}(\mathbf{X},\mathbf{Y}) := \sup_{s,t \in [0,T] \atop s < t} \frac{d_{cc}(\mathbf{X}_{s,t},\mathbf{Y}_{s,t})}{\vert s-t \vert^{\alpha'}},
        \end{equation}
       for $\mathbf{X},\mathbf{Y} \in C^\alpha_0([0,T];G^{\lfloor 1/\alpha \rfloor}(\mathbb{R}^d))$ and $0 \leq \alpha' \leq \alpha$, where $\mathbf{X}_{s,t} := \mathbf{X}_s^{-1} \otimes \mathbf{X}_t \in G^{\lfloor 1/\alpha \rfloor}(\mathbb{R}^d)$, see \cite[Definition 8.1]{friz10}. Moreover, we define the metric
       \begin{equation*}
            d_{cc,\infty}(\mathbf{X},\mathbf{Y}) = \sup_{t \in [0,T]} d_{cc}(\mathbf{X}_t,\mathbf{Y}_t),
        \end{equation*}
        for $\mathbf{X},\mathbf{Y} \in C^\alpha_0([0,T];G^{\lfloor 1/\alpha \rfloor}(\mathbb{R}^d))$.
    \end{definition}

    In order to relate the rough path metrics $d_{cc,\alpha'}$ and $d_{cc,\infty}$ to the $C^{\alpha'}$-topologies on H\"older spaces (see Section~\ref{SecNotation}), we introduce for every $\alpha' \in [0,\alpha] \cup \lbrace \infty \rbrace$ the embedding
    \begin{equation}
        \label{EqRPHoelderEmb}
        (C^\alpha_0([0,T];G^{\lfloor 1/\alpha \rfloor}(\mathbb{R}^d)),d_{cc,\alpha'}) \ni \mathbf{X} \,\,\, \mapsto \,\,\, \big( \mathbf{X}^{(1)}, \ldots, \mathbf{X}^{(\lfloor 1/\alpha \rfloor)} \big) \in \bigoplus_{n=1}^{\lfloor 1/\alpha \rfloor} (C^{n \alpha}_0([0,T];(\mathbb{R}^d)^{\otimes n}),\tau_{n\alpha'}),
    \end{equation}
    where $\bigoplus_{n=1}^{\lfloor 1/\alpha \rfloor} C^{n \alpha}_0([0,T];(\mathbb{R}^d)^{\otimes n})$ is a Banach space under the norm $(\mathbf{X}^{(1)}, \ldots, \mathbf{X}^{(\lfloor 1/\alpha \rfloor)}) \mapsto \max_{n=1,\ldots,\lfloor 1/\alpha \rfloor} \Vert \mathbf{X}^{(n)} \Vert_{C^{n \alpha}_0([0,T];(\mathbb{R}^d)^{\otimes n})}$. Now, we show that the map \eqref{EqRPHoelderEmb} is continuous.

    \begin{lemma}
        \label{LemmaRPHoelderEmb}
        For every $\alpha' \in [0,\alpha] \cup \lbrace \infty \rbrace$, the embedding \eqref{EqRPHoelderEmb} is continuous and admits a continuous left-inverse.
    \end{lemma}
    \begin{proof}
        Fix some $0 \leq \alpha' \leq \alpha$. Then, for every $n = 1,\ldots,\lfloor 1/\alpha \rfloor$, we first observe that $(\mathbb{R}^d)^{\otimes n}$ is a finite dimensional vector space, whence a dual Banach space, where the norm topology induced by $\Vert \cdot \Vert_{(\mathbb{R}^d)^{\otimes n}}$ coincides with the weak-$*$-topology. Moreover, by using \cite[Theorem~7.44]{friz10}, the homogeneous norm $G^{\lfloor 1/\alpha \rfloor}(\mathbb{R}^d) \ni \mathbf{g} \mapsto \max_{n=1,\ldots,\lfloor 1/\alpha \rfloor} \Vert \mathbf{g}^{(n)} \Vert_{(\mathbb{R}^d)^{\otimes n}}^{1/n} \in \mathbb{R}$ is equivalent to the Carnot-Carath\'eodory norm $\Vert \cdot \Vert_{cc}$, i.e.~there exists a constant $C_\alpha > 0$ such that for every $\mathbf{g},\mathbf{h} \in G^{\lfloor 1/\alpha \rfloor}(\mathbb{R}^d)$ it holds that
        \begin{equation}
            \label{EqLemmaRPHoelderEmbProof1}
            \frac{1}{C_\alpha} \max_{n=1,\ldots,\lfloor 1/\alpha \rfloor} \Vert \mathbf{g}^{(n)} \Vert_{(\mathbb{R}^d)^{\otimes n}}^\frac{1}{n} \leq \Vert \mathbf{g} \Vert_{cc} \leq C_\alpha \max_{n=1,\ldots,\lfloor 1/\alpha \rfloor} \Vert \mathbf{g}^{(n)} \Vert_{(\mathbb{R}^d)^{\otimes n}}^\frac{1}{n}.
        \end{equation}
        Hence, for every $\varepsilon \in (0,1)$, we can choose some $\delta \in (0,\varepsilon/C_\alpha)$ to conclude from \eqref{EqLemmaRPHoelderEmbProof1} that for every $\mathbf{X},\mathbf{Y} \in C^\alpha_0([0,T];G^{\lfloor 1/\alpha \rfloor}(\mathbb{R}^d))$ with $d_{cc,\alpha'}(\mathbf{X},\mathbf{Y}) < \delta$ it holds that
        \begin{equation*}
            \begin{aligned}
                & \max_{n=1,\ldots,\lfloor 1/\alpha \rfloor} \Vert \mathbf{X}^{(n)} - \mathbf{Y}^{(n)} \Vert_{C^{n\alpha'}([0,T];(\mathbb{R}^d)^{\otimes n})} = \sup_{s,t \in [0,T] \atop s \neq t} \frac{\Vert ((\mathbf{X}-\mathbf{Y})_s - (\mathbf{X}-\mathbf{Y})_t)^{(n)} \Vert_{(\mathbb{R}^d)^{\otimes n}}}{\vert t-s \vert^{n\alpha'}} \\
                & \quad\quad \leq \max_{n=1,\ldots,\lfloor 1/\alpha \rfloor} \left( \sup_{s,t \in [0,T] \atop s \neq t} \frac{\Vert ((\mathbf{X}-\mathbf{Y})_s - (\mathbf{X}-\mathbf{Y})_t)^{(n)} \Vert_{(\mathbb{R}^d)^{\otimes n}}^\frac{1}{n}}{\vert t-s \vert^{\alpha'}} \right)^n \\
                & \quad\quad \leq \max_{n=1,\ldots,\lfloor 1/\alpha \rfloor} \left( \sup_{s,t \in [0,T] \atop s \neq t} \frac{\max_{m=1,\ldots,\lfloor 1/\alpha \rfloor} \Vert ((\mathbf{X}-\mathbf{Y})_s - (\mathbf{X}-\mathbf{Y})_t)^{(m)} \Vert_{(\mathbb{R}^d)^{\otimes m}}^\frac{1}{m}}{\vert t-s \vert^{\alpha'}} \right)^n \\
                & \quad\quad \leq \max_{n=1,\ldots,\lfloor 1/\alpha \rfloor} \left( C_\alpha \sup_{s,t \in [0,T] \atop s \neq t} \frac{d_{cc}(\mathbf{X}_{s,t},\mathbf{Y}_{s,t})}{\vert t-s \vert^{\alpha'}} \right)^n = \max_{n=1,\ldots,\lfloor 1/\alpha \rfloor} \left( C_\alpha d_{cc,\alpha'}(\mathbf{X},\mathbf{Y}) \right)^n \\
                & \quad\quad < \max_{n=1,\ldots,\lfloor 1/\alpha \rfloor} \left( C_\alpha \delta \right)^n < \max_{n=1,\ldots,\lfloor 1/\alpha \rfloor} \varepsilon^n \leq \varepsilon.
            \end{aligned}
        \end{equation*}
        Thus, the images of $\mathbf{X}, \mathbf{Y}$ under \eqref{EqRPHoelderEmb} have at most distance $\varepsilon$ in $\bigoplus_{n=1}^{\lfloor 1/\alpha \rfloor} C^{n \alpha'}_0([0,T];(\mathbb{R}^d)^{\otimes n})$, which shows that \eqref{EqRPHoelderEmb} is continuous.

        On the other hand, by using that \eqref{EqRPHoelderEmb} is injective, there exists a left-inverse of \eqref{EqRPHoelderEmb}. Then, for every $\varepsilon > 0$, we can choose some $\delta \in (0,\min(1,(\varepsilon/C_\alpha)^{\lfloor 1/\alpha \rfloor}))$ to conclude from \eqref{EqLemmaRPHoelderEmbProof1} that for every $\mathbf{X},\mathbf{Y} \in C^\alpha_0([0,T];G^{\lfloor 1/\alpha \rfloor}(\mathbb{R}^d))$ with $\max_{n=1,\ldots,\lfloor 1/\alpha \rfloor} \vert \mathbf{X}-\mathbf{Y} \vert_{C^{n\alpha'}([0,T];(\mathbb{R}^d){\otimes n})} < \delta$ it holds that
        \begin{equation*}
            \begin{aligned}
                d_{cc,\alpha'}(\mathbf{X},\mathbf{Y}) & = \sup_{s,t \in [0,T] \atop s \neq t} \frac{d_{cc}(\mathbf{X}_{s,t},\mathbf{Y}_{s,t})}{\vert t-s \vert^{\alpha'}} \\
                & \leq C_\alpha \sup_{s,t \in [0,T] \atop s \neq t} \frac{\max_{n=1,\ldots,\lfloor 1/\alpha \rfloor} \Vert ((\mathbf{X}-\mathbf{Y})_s - (\mathbf{X}-\mathbf{Y})_t)^{(n)} \Vert_{(\mathbb{R}^d)^{\otimes n}}^\frac{1}{n}}{\vert t-s \vert^{\alpha'}} \\
                & \leq C_\alpha \max_{n=1,\ldots,\lfloor 1/\alpha \rfloor} \left( \sup_{s,t \in [0,T] \atop s \neq t} \frac{\Vert ((\mathbf{X}-\mathbf{Y})_s - (\mathbf{X}-\mathbf{Y})_t)^{(n)} \Vert_{(\mathbb{R}^d)^{\otimes n}}}{\vert t-s \vert^{n\alpha'}} \right)^\frac{1}{n} \\
                & \leq C_\alpha \max_{n=1,\ldots,\lfloor 1/\alpha \rfloor} \left( \max_{m=1,\ldots,\lfloor 1/\alpha \rfloor} \Vert \mathbf{X}-\mathbf{Y} \Vert_{C^{m\alpha'}([0,T];(\mathbb{R}^d)^{\otimes m})} \right)^\frac{1}{n} \\
                & < C_\alpha \max_{n=1,\ldots,\lfloor 1/\alpha \rfloor} \delta^\frac{1}{n} < C_\alpha \delta^\frac{1}{\lfloor 1/\alpha \rfloor} \leq \varepsilon.
            \end{aligned}
        \end{equation*}
        This shows that the left-inverse of \eqref{EqRPHoelderEmb} is continuous, for all $0 \leq \alpha' \leq \alpha$. Moreover, a similar argument also holds true for the case $\alpha' = \infty$.
    \end{proof}

    Moreover, by using that each H\"older space $C^{n \alpha}_0([0,T];(\mathbb{R}^d)^{\otimes n})$, $n = 1,\ldots,\lfloor 1/\alpha \rfloor$, carries a weak-$*$-topology, we can equip $C^\alpha_0([0,T];G^{\lfloor 1/\alpha \rfloor}(\mathbb{R}^d))$ also with a weak-$*$-topology.

    \begin{remark}
        \label{RemRPHoelderWeakStar}
        Since each space $(C^{n \alpha}_0([0,T];(\mathbb{R}^d)^{\otimes n}),\Vert \cdot \Vert_{C^{n \alpha}_0([0,T];(\mathbb{R}^d)^{\otimes n})})$, $n = 1,\ldots,\lfloor 1/\alpha \rfloor$, is by Theorem~\ref{ThmHoelderPredual} a dual Banach space carrying a weak-$*$-topology $\tau_{w^*}^{(n)}$, we can equip $\bigoplus_{n=1}^{\lfloor 1/\alpha \rfloor} C^{n \alpha}([0,T];(\mathbb{R}^d)^{\otimes n})$ with a weak-$*$-topology $\tau_{w^*}$ defined as the product topology of $\tau_{w^*}^{(n)}$, $n = 1,\ldots,\lfloor 1/\alpha \rfloor$. Hence, the space of weakly geometric $\alpha$-H\"older rough paths $C^\alpha_0([0,T];G^{\lfloor 1/\alpha \rfloor}(\mathbb{R}^d))$ carries a weak-$*$-topology $\tau_{w^*}$ defined as the initial topology with respect to the inclusion $C^\alpha_0([0,T];G^{\lfloor 1/\alpha \rfloor}(\mathbb{R}^d)) \hookrightarrow (\bigoplus_{n=1}^{\lfloor 1/\alpha \rfloor} C^{n \alpha}([0,T];(\mathbb{R}^d)^{\otimes n}),\tau_{w^*})$.  
    \end{remark}
	
	Next, we define the truncated signature of order $N > \lfloor 1/\alpha \rfloor$ of a weakly geometric $\alpha$-H\"older rough path $\mathbf{X} \in C^\alpha_0([0,T];G^{\lfloor 1/\alpha \rfloor}(\mathbb{R}^d))$ as the unique Lyons extension yielding a path $\mathbb{X}^N: [0,T] \rightarrow G^N(\mathbb{R}^d)$, see \cite[Corollary~9.11~(ii)]{friz10}. Then, $\mathbb{X}^N$ starts from the unit element $\mathbf{1} \in G^N(\mathbb{R}^d)$ and has finite $\alpha$-H\"older norm $\Vert \cdot \Vert_{cc,\alpha}$. The components of $\mathbb{X}^N$ in $(\mathbb{R}^d)^{\otimes n}$ are denoted by $\mathbb{X}^{(n)}$ and the signature of $\mathbb{X}$ is given by
	\begin{equation*}
	    [0,T] \ni t \quad \mapsto \quad \mathbb{X}_t := \big( 1, X_t, \mathbb{X}^{(2)}_t, \ldots, \mathbb{X}^{(\lfloor 1/\alpha \rfloor)}_t,\ldots,\mathbb{X}^{(N)}_t,\ldots \big).
	\end{equation*}
	A linear function of the signature (at time $T$) is given as $l(\mathbb{X}_T) = \sum_{0 \leq \vert I \vert \leq N} a_I \langle e_I, \mathbb{X}_T \rangle$, for some $N \in \mathbb{N}_0$ and $a_I \in \mathbb{R}$. 
 
    To ensure point separation (without tree-like equivalences), we shall always add an additional component to $X$ representing time. More precisely, we define the subset
    \begin{equation*}
		\widehat{C}^\alpha_{d,T} := \left\lbrace \widehat{\mathbf{X}} \in C^\alpha_0([0,T];G^{\lfloor 1/\alpha \rfloor}(\mathbb{R}^{d+1})): \widehat{X}_t = (t,X_t), \text{ for all } t \in [0,T] \right\rbrace
	\end{equation*}
	of $C^\alpha_0([0,T];G^{\lfloor 1/\alpha \rfloor}(\mathbb{R}^{d+1}))$, where the $0$-th coordinate denotes running time. 
    
    Now, we equip $\widehat{C}^\alpha_{d,T}$ either with the $\alpha'$-H\"older metric $d_{cc,\alpha'}$ defined in \eqref{EqDefWGRPHoelder1}, for some $0 \leq \alpha' < \alpha$, or with the weak-$*$-topology described in Remark~\ref{RemRPHoelderWeakStar}. Then, in order to get a weighted space, we choose the weight function $\psi: \widehat{C}^\alpha_{d,T} \rightarrow (0,\infty)$ defined by
	\begin{equation}
        \label{eq:weight_holder}
		\psi(\widehat{\mathbf{X}}) = \exp\left( \beta \Vert \widehat{\mathbf{X}} \Vert_{cc,\alpha}^\gamma \right),
	\end{equation}
	for $\widehat{\mathbf{X}} \in \widehat{C}^\alpha_{d,T}$ and some $\beta > 0$ and $\gamma \geq \lfloor 1/\alpha \rfloor$. Then, the pre-image $K_R := \psi^{-1}((0,R])$ is a bounded subset of $(\widehat{C}^\alpha_{d,T},\Vert \cdot \Vert_{cc,\alpha})$, whence by using the embedding \eqref{EqRPHoelderEmb} also bounded in $\bigoplus_{n=1}^{\lfloor 1/\alpha \rfloor} C^{n \alpha}([0,T];(\mathbb{R}^{d+1})^{\otimes n})$. Thus, by using Remark~\ref{RemHoelder0}~\ref{RemHoelder0Embedding}, the pre-image $K_R := \psi^{-1}((0,R])$ is compact in $\bigoplus_{n=1}^{\lfloor 1/\alpha \rfloor} (C^{n \alpha}_0([0,T];(\mathbb{R}^{d+1})^{\otimes n}),\tau_{n\alpha'})$, for all $0 \leq \alpha' < \alpha$, whence by Lemma~\ref{LemmaRPHoelderEmb} also compact in $(\widehat{C}^\alpha_{d,T},d_{cc,\alpha'})$, for all $0 \leq \alpha' < \alpha$. On the other hand, for the weak-$*$-topology $\tau_{w^*}$ on $\widehat{C}^\alpha_{d,T}$, the Banach-Alaoglu theorem ensures that $K_R := \psi^{-1}((0,R])$ is compact in $(\bigoplus_{n=0}^{\lfloor 1/\alpha \rfloor} (C^{n \alpha}([0,T];(\mathbb{R}^{d+1})^{\otimes n}),\tau_{w^*})$, whence by Lemma~\ref{LemmaRPHoelderEmb} also compact in $(\widehat{C}^\alpha_{d,T},\tau_{w^*})$. This is in line with the observation that the $C^{\alpha'}$-topology and the weak-$*$-topology coincide on the norm-bounded subset $K_R \subseteq \widehat{C}^\alpha_{d,T}$ (see also Remark~\ref{RemHoelder0}~\ref{RemHoelder0Predual}). Hence, the weight function $\psi: \widehat{C}^\alpha_{d,T} \rightarrow (0,\infty)$ is on both spaces admissible.
	
	Now, we present the global universal approximation theorem for linear functions of the signature, which can approximate any path space functional in $\mathcal{B}_\psi(\widehat{C}^\alpha_{d,T})$. Even though the proof looks involved, the most important ingredients can be summarized as follows: certain linear functions of the signature are given by the integrals
	\begin{equation*}
		\int_0^T \langle e_I, \widehat{\mathbb{X}}_{t} \rangle \frac{t^k}{k!} dt
	\end{equation*}
 	for $k \in \mathbb{N}_0$, $0 \leq \vert I \vert \leq \lfloor 1/\alpha \rfloor$, which determine the rough path uniquely (see definition of $\widetilde{\mathcal{A}}$ in \eqref{EqThmUATHoelderProof1} below), and those linear functions are of $\psi$-moderate growth, i.e.~also their exponential lies in $\mathcal{B}_\psi(\widehat{C}^\alpha_{d,T})$.
	
	\begin{theorem}[Universal approximation on $\mathcal{B}_\psi(\widehat{C}^\alpha_{d,T})$]
		\label{ThmUATHoelder}
		The linear span of the set
        \begin{equation*}
            \left\lbrace \widehat{\mathbf{X}} \mapsto \langle e_I, \widehat{\mathbb{X}}_T \rangle: I \in \lbrace 0,\ldots,d \rbrace^N, \, N \in \mathbb{N}_0 \right\rbrace
        \end{equation*}
        is dense in $\mathcal{B}_\psi(\widehat{C}^\alpha_{d,T})$, i.e.~for every map $f \in \mathcal{B}_\psi(\widehat{C}^\alpha_{d,T})$ and every $\varepsilon > 0$ there exists a linear function of the form $\widehat{\mathbb{X}}_T \mapsto l(\widehat{\mathbb{X}}_T) := \sum_{0 \leq \vert I \vert \leq N} a_I \langle e_I, \widehat{\mathbb{X}}_T \rangle$, with some $N \in \mathbb{N}_0$ and $a_I \in \mathbb{R}$, such that
		\begin{equation*}
			\sup_{\widehat{\mathbf{X}} \in \widehat{C}^\alpha_{d,T}} \frac{\left\vert f(\widehat{\mathbf{X}}) - l(\widehat{\mathbb{X}}_T) \right\vert}{\psi(\widehat{\mathbf{X}})} < \varepsilon.
		\end{equation*}
	\end{theorem}
	\begin{proof}
        First, we observe that the function space $\mathcal{B}_\psi(\bigoplus_{n=1}^{\lfloor 1/\alpha \rfloor} (C^{n \alpha}_0([0,T];(\mathbb{R}^d)^{\otimes n}))$ is by Remark~\ref{RemHoelder0}~\ref{RemHoelder0BPsi} the same for each underlying space $(\bigoplus_{n=1}^{\lfloor 1/\alpha \rfloor} C^{n \alpha}_0([0,T];(\mathbb{R}^d)^{\otimes n}),\tau_{w^*})$ and $\bigoplus_{n=1}^{\lfloor 1/\alpha \rfloor} (C^{n \alpha}_0([0,T];(\mathbb{R}^d)^{\otimes n}),\tau_{n\alpha'})$, with $\alpha' \in [0,\alpha) \cup \lbrace \infty \rbrace$. Hence, by using Lemma~\ref{LemmaRPHoelderEmb}, the space $\mathcal{B}_\psi(\widehat{C}^\alpha_{d,T})$ is the same for each underlying space $(\widehat{C}^\alpha_{d,T},\tau_{w^*})$ and $(\widehat{C}^\alpha_{d,T},d_{cc,\alpha'})$, with $\alpha' \in [0,\alpha) \cup \lbrace \infty \rbrace$. Thus, it suffices to show the conclusion only for the space $(\widehat{C}^\alpha_{d,T},d_{cc,\infty})$. To this end, we apply the weighted real-valued Stone-Weierstrass theorem (Theorem~\ref{ThmStoneWeierstrassBpsiR}) to
		\begin{equation*}
			\mathcal{A} = \linspan\left\lbrace \widehat{\mathbf{X}} \mapsto \langle e_I, \widehat{\mathbb{X}}_T \rangle : I \in \lbrace 0,\ldots,d \rbrace^N, \, N \in \mathbb{N}_0 \right\rbrace.
		\end{equation*}
		Therefore, we need to prove that $\mathcal{A} \subseteq \mathcal{B}_\psi(\widehat{C}^\alpha_{d,T})$ is a vector subspace and a subalgebra that is point separating and nowhere vanishing of $\psi$-moderate growth, where
		\begin{equation}
		    \label{EqThmUATHoelderProof1}
			\widetilde{\mathcal{A}} = \linspan\left\lbrace \widehat{\mathbf{X}} \mapsto \langle \left( e_I \shuffle e_0^{\otimes k} \right) \otimes e_0, \widehat{\mathbb{X}}_T \rangle: \, 
            \begin{matrix}
                k \in \mathbb{N}_0, \, N \in \lbrace 0,\ldots,\lfloor 1/\alpha \rfloor \rbrace, \\
                I \in \lbrace 0,\ldots,d \rbrace^N
            \end{matrix}
            \right\rbrace \subseteq \mathcal{A}
		\end{equation}
		is a possible candidate for the point separating and nowhere vanishing vector subspace of $\psi$-moderate growth.
		
		In order to prove that $\mathcal{A} \subseteq \mathcal{B}_\psi(\widehat{C}^\alpha_{d,T})$ is a vector subspace, we fix some $a \in \mathcal{A}$ of the form $\widehat{C}^\alpha_{d,T} \ni \widehat{\mathbf{X}} \mapsto a(\widehat{\mathbf{X}}) := \langle e_I, \widehat{\mathbb{X}}_T \rangle \in \mathbb{R}$, for some $I \in \lbrace 0,\ldots,d \rbrace^N$ and $N \in \mathbb{N}_0$. Moreover, we fix some $R > 0$ and observe that the pre-image $K_R := \psi^{-1}((0,R])$ is bounded with respect to $\Vert \cdot \Vert_{cc,\alpha}$. Then, it follows from \cite[Corollary~10.40]{friz10} that the map
		\begin{equation*}
			(K_R,d_{cc,\infty}) \ni \widehat{\mathbf{X}} \quad \mapsto \quad \widehat{\mathbb{X}}^N \in (C^\alpha_0([0,T];G^N(\mathbb{R}^{d+1})),d_{cc,\infty})
		\end{equation*}
		is continuous on $K_R$ with respect to $d_{cc,\infty}$. This together with the continuity of the evaluation map
		\begin{equation*}
			(C^\alpha_0([0,T];G^N(\mathbb{R}^{d+1})),d_{cc,\infty}) \ni \widehat{\mathbb{X}}^N \quad \mapsto \quad \widehat{\mathbb{X}}^N_T \in (G^N(\mathbb{R}^{d+1}),d_{cc})
		\end{equation*}
		shows that the map
		\begin{equation*}
			(K_R,d_{cc,\infty}) \ni \widehat{\mathbf{X}} \quad \mapsto \quad \widehat{\mathbb{X}}^N_T \in (G^N(\mathbb{R}^{d+1}),d_{cc})
		\end{equation*}
		is continuous on $K_R$ with respect to $d_{cc,\infty}$. Furthermore, since linear functions on the finite dimensional space $G^N(\mathbb{R}^{d+1})$ are continuous, it follows that the map
		\begin{equation}
		    \label{EqThmUATHoelderProof2}
			(K_R,d_{cc,\infty}) \ni \widehat{\mathbf{X}} \quad \mapsto \quad a(\widehat{\mathbf{X}}) = \langle e_I, \widehat{\mathbb{X}}_T \rangle \in \mathbb{R}
		\end{equation}
		is continuous on $K_R$ with respect to $d_{cc,\infty}$. Since $R > 0$ was chosen arbitrarily, this shows that $a\vert_{K_R} \in C^0(K_R)$, for all $R > 0$. Moreover, by using the inequality
        \begin{equation*}
            \Vert g-h \Vert_{T^N(\mathbb{R}^{d+1})} \leq C_1 \max\left( d_{cc}(g,h) \max\left(1, \Vert g \Vert_{cc}^{N-1} \right), d_{cc}(g,h)^N\right)
        \end{equation*}
        (see \cite[Proposition 7.49]{friz10}) for each $g,h \in G^N(\mathbb{R}^{d+1})$ and some constant $C_1 \geq 1$ and by choosing $g = \widehat{\mathbb{X}}^N_0$ and $h = \widehat{\mathbb{X}}^N_T$ we obtain for every $\widehat{\mathbf{X}} \in \widehat{C}^\alpha_{d,T}$ that
        \begin{equation*}
            \vert a(\widehat{\mathbf{X}}) \vert = \vert \langle e_I, \widehat{\mathbb{X}}_T \rangle \vert \leq \Vert \widehat{\mathbb{X}}^N_T \Vert_{T^N(\mathbb{R}^{d+1})} \leq \Vert \widehat{\mathbb{X}}^N_T - \widehat{\mathbb{X}}^N_0 \Vert_{T^N(\mathbb{R}^{d+1})} + 1 \leq C_1\left( d_{cc}(\widehat{\mathbb{X}}^N_T, \widehat{\mathbb{X}}^N_0)^N + 2 \right).
        \end{equation*}
    	This can be further estimated by applying the inequality $d_{cc}(\widehat{\mathbb{X}}^N_s,\widehat{\mathbb{X}}^N_t) \leq C_{N,\alpha} d_{cc}(\widehat{\mathbf{X}}_s, \widehat{\mathbf{X}}_t)$ for all $\mathbf{X} \in \widehat{C}^\alpha_{d,T}$ and some constant $C_{N,\alpha} > 0$ (see \cite[Theorem~9.5]{friz10} for the $p$-variation case, which also holds for the $\alpha$-H\"older case by \cite[p.~182]{friz10}), so that
		\begin{equation}
            \label{EqThmUATHoelderProof3}
            \begin{aligned}
                \vert a(\widehat{\mathbf{X}}) \vert & \leq C_1 \left( d_{cc}(\widehat{\mathbb{X}}^N_T, \widehat{\mathbb{X}}^N_0)^N+2 \right) \leq C_1 \left( T^{\alpha N} \left(\sup_{s,t \in [0,T] \atop s < t} \frac{d_{cc}(\widehat{\mathbb{X}}^N_s, \widehat{\mathbb{X}}^N_t)}{\vert s-t \vert^\alpha}\right)^N + 2 \right) \\
                & \leq C_1 \left( C^N_{N,\alpha} T^{\alpha N}\left( \sup_{s,t \in [0,T] \atop s < t} \frac{d_{cc}(\widehat{\mathbf{X}}_s, \widehat{\mathbf{X}}_t)}{\vert s-t \vert^\alpha}\right)^N + 2 \right) = C_1 \left( C^N_{N,\alpha} T^{\alpha N} \Vert \widehat{\mathbf{X}} \Vert_{cc,\alpha}^N + 2 \right).
            \end{aligned}
		\end{equation}
		Since the exponential function dominates any polynomial, we conclude that
		\begin{equation}
            \label{EqThmUATHoelderProof4}
    		\lim_{R \rightarrow \infty} \sup_{\widehat{\mathbf{X}} \in \widehat{C}^\alpha_{d,T} \setminus K_R} \frac{\vert a(\widehat{\mathbf{X}}) \vert}{\psi(\widehat{\mathbf{X}})} \leq C_1 \lim_{R \rightarrow \infty} \sup_{\widehat{\mathbf{X}} \in \widehat{C}^\alpha_{d,T} \setminus K_R} \frac{C^N_{N,\alpha} T^{\alpha N} \Vert \widehat{\mathbf{X}} \Vert^N_{cc,\alpha} + 2}{\exp\left( \beta \Vert \widehat{\mathbf{X}} \Vert_{cc,\alpha}^{\gamma} \right)} = 0.
		\end{equation}
		Hence, it follows from Lemma~\ref{LemmaBpsiEquivChar}~\ref{LemmaBpsiEquivChar2} that $a \in \mathcal{B}_\psi(\widehat{C}^\alpha_{d,T})$, which shows that $\mathcal{A} \subseteq \mathcal{B}_\psi(\widehat{C}^\alpha_{d,T})$.
		
		Next, we observe that $\mathcal{A}$ is by the shuffle property in \eqref{EqDefShuffleProperty} a subalgebra of $\mathcal{B}_\psi(\widehat{C}^\alpha_{d,T})$. In order to show that $\mathcal{A}$ is point separating and nowhere vanishing of $\psi$-moderate growth, we claim that the vector subspace $\widetilde{\mathcal{A}} \subseteq \mathcal{A}$ defined in \eqref{EqThmUATHoelderProof1} is point separating, nowhere vanishing, and for every $\widetilde{a} \in \widetilde{\mathcal{A}}$ there exists some $\lambda > 0$ such that $\exp( \lambda \vert \widetilde{a}(\cdot) \vert) \in \mathcal{B}_\psi(\widehat{C}^\alpha_{d,T})$. For the latter, we fix some $\big( \widehat{\mathbf{X}} \mapsto \widetilde{a}(\widehat{\mathbf{X}}) = l(\widehat{\mathbb{X}}_T) \big) \in \widetilde{\mathcal{A}}$ with linear function $l(\widehat{\mathbb{X}}_T) = \sum_{0 \leq \vert I \vert \leq N} \sum_{k=0}^K a_{I,k} \langle \left( e_I \shuffle e_0^{\otimes k} \right) \otimes e_0, \widehat{\mathbb{X}}_T \rangle$, for some $K \in \mathbb{N}_0$ and $N \in \lbrace 0,\ldots,\lfloor 1/\alpha \rfloor \rbrace$ and $a_{I,k} \in \mathbb{R}$. Then, by similar arguments as for \eqref{EqThmUATHoelderProof2}, we have $\exp\left( \left\vert \lambda \widetilde{a}(\cdot) \right\vert \right)\vert_{K_R} \in C^0(K_R)$, for all $\lambda, R > 0$. Again by a similar reasoning as for \eqref{EqThmUATHoelderProof3} and using the form of the elements in $\widetilde{\mathcal{A}}$ as given in \eqref{eq:formAtilde} below, it follows for every $\widehat{\mathbf{X}} \in \widehat{C}^\alpha_{d,T}$ that
		\begin{equation}
            \label{EqThmUATHoelderProof5}
		    \begin{aligned}
		        \vert \widetilde{a}(\widehat{\mathbf{X}}) \vert & = \vert l(\widehat{\mathbb{X}}_T) \vert \leq C_1 \Vert l \Vert_{T^{N+K+1}(\mathbb{R}^{d+1})^*} \left( T^{\alpha (K+1) N} \sup_{s,t \in [0,T] \atop s < t} \left(\frac{d_{cc}(\widehat{\mathbb{X}}^{N}_s,\widehat{\mathbb{X}}^{N}_t)}{\vert s-t \vert^\alpha}\right)^N +1\right) \\
                & \leq C_1 \Vert l \Vert_{T^{N+K+1}(\mathbb{R}^{d+1})^*}\left( C^N_{N,\alpha}  T^{\alpha (K+1) N} \left(\sup_{s,t \in [0,T] \atop s < t} \frac{d_{cc}(\widehat{\mathbf{X}}_s,\widehat{\mathbf{X}}_t)}{\vert s-t \vert^\alpha}\right)^N +1 \right) \\
		        & = C_1 \Vert l \Vert_{T^{N+K+1}(\mathbb{R}^{d+1})^*} \left( C^N_{N,\alpha}  T^{\alpha (K+1)N} \Vert \widehat{\mathbf{X}} \Vert_{cc,\alpha}^N +1 \right).
		    \end{aligned}
		\end{equation}
		Then, for $C_2 := \max(C_1 \Vert l \Vert_{T^{N+K+1}(\mathbb{R}^{d+1})^*} C^N_{N,\alpha} T^{\alpha(K+1)N}, C_1 \Vert l \Vert_{T^{N+K+1}(\mathbb{R}^{d+1})^*}) > 0$, we have
		\begin{equation}
            \label{EqThmUATHoelderProof6}
            \lim_{R \rightarrow \infty} \sup_{\widehat{\mathbf{X}} \in \widehat{C}^\alpha_{d,T} \setminus K_R} \frac{\exp( \lambda \vert \widetilde{a}(\widehat{\mathbf{X}}) \vert)}{\psi(\widehat{\mathbf{X}})} \leq \lim_{R \rightarrow \infty} \sup_{\widehat{\mathbf{X}} \in \widehat{C}^\alpha_{d,T} \setminus K_R} \frac{\exp\left( \lambda C_2 (\Vert \widehat{\mathbf{X}} \Vert_{cc,\alpha}^N +1) \right)}{\exp\left( \beta \Vert \widehat{\mathbf{X}} \Vert_{cc,\alpha}^\gamma \right)} = 0,
		\end{equation}
		where the last equality follows by choosing $\lambda < \beta/C_2$ small enough ensuring that the denominator tends faster to infinity than the nominator (as $\gamma \geq \lfloor 1/\alpha \rfloor \geq N$). Hence, Lemma~\ref{LemmaBpsiEquivChar}~\ref{LemmaBpsiEquivChar2} shows that $\exp( \lambda \vert \widetilde{a}(\cdot) \vert) \in \mathcal{B}_\psi(\widehat{C}^\alpha_{d,T})$ which holds true for any $\widetilde{a} \in \widetilde{\mathcal{A}}$.
  
        Now, to show that $\widetilde{\mathcal{A}}$ is point separating, let $\widehat{\mathbf{Y}}, \widehat{\mathbf{Z}} \in \widehat{C}^\alpha_{d,T}$ be distinct. By contradiction, let us assume that for every $k \in \mathbb{N}_0$, $N \in \lbrace 0,\ldots,\lfloor 1/\alpha \rfloor \rbrace$, and $I \in \lbrace 0,\ldots,d \rbrace^N$ it holds that
		\begin{equation*}
		    \langle \left( e_I \shuffle e_0^{\otimes k} \right) \otimes e_0, \widehat{\mathbb{Y}}_T \rangle = \langle \left( e_I \shuffle e_0^{\otimes k} \right) \otimes e_0, \widehat{\mathbb{Z}}_T \rangle,
		\end{equation*}
        where we observe by \eqref{EqDefShuffleProperty} that
        \begin{equation}
        	\label{eq:formAtilde}
            \langle \left( e_I \shuffle e_0^{\otimes k} \right) \otimes e_0, \widehat{\mathbb{X}}_T \rangle = \int_0^T \langle e_I \shuffle e_0^{\otimes k}, \widehat{\mathbb{X}}_t \rangle dt = \int_0^T \langle e_I, \widehat{\mathbb{X}}_t \rangle \langle e_0^{\otimes k}, \widehat{\mathbb{X}}_t \rangle dt = \int_0^T \langle e_I, \widehat{\mathbb{X}}_t \rangle \frac{t^k}{k!} dt,
        \end{equation}
		for all $\widehat{\mathbf{X}} \in \widehat{C}^\alpha_{d,T}$. Thus, we conclude for every $k \in \mathbb{N}_0$, $N \in \lbrace 0,\ldots,\lfloor 1/\alpha \rfloor \rbrace$, and $I \in \lbrace 0,\ldots,d \rbrace^N$ that
		\begin{equation*}
		    \int_0^T \langle e_I, \widehat{\mathbb{Y}}_t - \widehat{\mathbb{Z}}_t \rangle \frac{t^k}{k!} dt = 0.
		\end{equation*}
		Hence, by using \cite[Corollary 4.24]{brezis11} and that $\Pol([0,T])$ are dense in $C^\infty_c([0,T])$ with respect to the supremum norm (see Theorem~\ref{ThmWstrass}), we have $\langle e_I, \widehat{\mathbb{Y}}_t - \widehat{\mathbb{Z}}_t \rangle = 0$. Since the Hahn-Banach theorem implies that continuous linear functions separate points, we conclude that $\widehat{\mathbf{Y}}_t=\widehat{\mathbb{Y}}^{\lfloor 1/\alpha \rfloor}_t = \widehat{\mathbb{Z}}^{\lfloor 1/\alpha \rfloor}_t=\widehat{\mathbf{Z}}_t$, for all $t \in [0,T]$. This however contradicts the assumption that $\widehat{\mathbf{Y}}, \widehat{\mathbf{Z}} \in \widehat{C}^\alpha_{d,T}$ are distinct, which shows that $\widetilde{\mathcal{A}}$ is point separating.

        Finally, we observe that $\widetilde{\mathcal{A}}$ vanishes nowhere. Indeed, by using the map $\big( \widehat{\mathbf{X}} \mapsto \widetilde{a}(\widehat{\mathbf{X}}) := \langle \left( e_\varnothing \shuffle e_0^{\otimes 0} \right) \otimes e_0, \widehat{\mathbb{X}}_T \rangle \big) \in \widetilde{\mathcal{A}}$, we observe that $\widetilde{a}(\widehat{\mathbf{X}}) = \int_0^T t dt = \frac{T^2}{2} \neq 0$, for all $\widehat{\mathbf{X}} \in \widehat{C}^\alpha_{d,T}$. Hence, we can apply the weighted real-valued Stone-Weierstrass theorem (Theorem~\ref{ThmStoneWeierstrassBpsiR}) to conclude that $\mathcal{A}$ is dense in $\mathcal{B}_\psi(\widehat{C}^\alpha_{d,T})$.
	\end{proof}

    \begin{remark}
        \label{RemSigChar}
        Let us point out the following remarks concerning Theorem~\ref{ThmUATHoelder}:
        \begin{enumerate}
	        \item\label{RemSigChar1} Recall that a feature map $\Phi$ with input space $X$ is called \emph{universal} to $\mathcal{B}_\psi(X)$ if the set of linear functionals of $\Phi$ is dense in $\mathcal{B}_\psi(X)$ (see \cite[Definition 6]{chevyrev22}). Thus, Theorem~\ref{ThmUATHoelder} shows that the feature map $\Phi(\widehat{\mathbf{X}}) = \widehat{\mathbb{X}}_T$ is universal to $\mathcal{B}_\psi(\widehat{C}^\alpha_{d,T})$.
	        
	  		\item By \cite[Theorem~7]{chevyrev22}, Theorem~\ref{ThmUATHoelder} thus also implies that the signature is \emph{characteristic} for the measure space $\mathcal{M}_\psi(\widehat{C}^\alpha_{d,T}) \cong \mathcal{B}_\psi(\widehat{C}^\alpha_{d,T})^*$ introduced in Example~\ref{ExWeightedSpaces}~\ref{ExWeightedSpaceMeasures}. Indeed, the map
	    	\begin{equation*}
	    	    \mathcal{M}_\psi(\widehat{C}^\alpha_{d,T}) \ni \mu \quad \mapsto \quad \left( a \mapsto \int_{\widehat{C}^\alpha_{d,T}} a(\widehat{\mathbf{X}}) \mu(d\widehat{\mathbf{X}}) \right) \in \mathbb{R},
	    	\end{equation*}
	    	is injective, with $a \in \mathcal{A} := \linspan\big\lbrace \widehat{\mathbf{X}} \mapsto \langle e_I, \widehat{\mathbb{X}}_T \rangle : I \in \lbrace 0,\ldots,d \rbrace^N, \, N \in \mathbb{N}_0 \big\rbrace$. Note that the latter is exactly the definition of being characteristic (see \cite[Definition~6]{chevyrev22}). 
	    	
	    	Laws of stochastic processes $\widehat{\mathbf{X}}$ on path space are thus characterized by their expected signature if the following (super-) exponential moment condition is satisfied:
	    	\begin{equation*}
	    		\mathbb{E}\left[ \psi(\widehat{\mathbf{X}}) \right] =\mathbb{E}\left[ \exp\left( \beta \Vert \widehat{\mathbf{X}} \Vert_{cc,\alpha}^\gamma \right) \right] < \infty,
	    	\end{equation*}
	    	for $\gamma \geq \lfloor 1/\alpha \rfloor$. In particular, this holds true for $G^2(\mathbb{R}^{d+1})$-enhanced (time-extended) Brownian motions $\widehat{\mathbf{X}}$ due to Fernique's theorem in \cite[Theorem~13.14]{friz10}. Let us also remark that an alternative sufficient condition stating when the law of a group-valued random variable is characterized by its expected signature is given in \cite[Corollary~6.6]{Chevyrev16}.

            \item An alternative proof for Theorem~\ref{ThmUATHoelder} is given by the following dual consideration: assume that the assertion is wrong, then by the Hahn-Banach theorem there is $0 \neq \mu \in \mathcal{M}_\psi(\widehat{C}^\alpha_{d,T})$ such that all $l(\widehat{\mathbb{X}}_T)$ are annihilated by $\mu$, i.e.
            $$
            \int_{\widehat{C}^\alpha_{d,T}} l(\widehat{\mathbb{X}}_T) \mu(d\widehat{\mathbf{X}}) = 0
            $$
            for all $l$. Now consider the $\sigma$-algebra generated by the signature components (seen as a map from $\widehat{C}^\alpha_{d,T}$ to $\mathbb{R}$, both equipped with the Borel-$\sigma$-algebra) and notice that it is identical to the $\sigma$-algebra generated by the signature increments up to order $ \lfloor 1/\alpha \rfloor $ by construction of higher order signatures.
            
            Let $\mathcal{F}_n$ be the $\sigma$-algebra generated by $\langle e_I, \widehat{\mathbb{X}}_T \rangle$ for $|I| \leq n$ and let $Q$ be the probability measure constructed by normalizing $|\mu|$. Clearly $\mu$ has a Radon-Nikodym derivative with respect to $Q$ which takes two non-zero values almost surely (see, e.g., \cite[Theorem~6.12]{rudin87}). Fix $ f \in \mathcal{B}_\psi(\widehat{C}^\alpha_{d,T}) $ with $\int f d \mu \neq 0 $ and consider the martingale convergence $ E[f \, | \mathcal{F}_n ] \to f$ in $L^1(Q)$. Then, on the one hand $E[f \, | \mathcal{F}_n ]$ can be approximated by polynomials in $\langle e_I, \widehat{\mathbb{X}}_T \rangle $ for $|I| \leq n$ due to the assumption of exponential integrability, and, on the other hand, those polynomials are actually some $l(\widehat{\mathbb{X}}_T)$ converging to $f$ in $L^1(Q)$, whence also when multiplied by $\frac{d \mu}{d Q}$. But this yields a contradiction since $\mu$ annihilates all $ l(\widehat{\mathbb{X}}_T)$ but not $f$.
            
	 		\item Alternatively to linear functions of the signature we could also consider the linear span of characteristic functions with complex coefficients given by
	 		\begin{equation*}
	 			\mathcal{W} := \linspan_{\mathbb{C}}\left\lbrace \widehat{\mathbf{X}} \mapsto \exp\left( \mathbf{i} u \langle e_I, \widehat{\mathbb{X}}_T \rangle \right) : I \in \lbrace 0,\ldots,d \rbrace^N, \, N \in \mathbb{N}_0, u \in \mathbb{R} \right\rbrace.
	 		\end{equation*}
            For an appropriate weight function $\psi: \widehat{C}^\alpha_{d,T} \rightarrow (0,\infty)$, we observe that
            \begin{equation*}
                \begin{aligned}
                    \mathcal{A} & := \linspan\Big( \left\lbrace \widehat{\mathbf{X}} \mapsto \cos\left( u \langle e_I, \widehat{\mathbb{X}}_T \rangle \right): I \in \lbrace 0,\ldots,d \rbrace^N, \, N \in \mathbb{N}_0, u \in \mathbb{R} \right\rbrace \\
                    & \quad\quad\quad\quad\quad \cup \left\lbrace \widehat{\mathbf{X}} \mapsto \sin\left( u \langle e_I, \widehat{\mathbb{X}}_T \rangle \right): I \in \lbrace 0,\ldots,d \rbrace^N, \, N \in \mathbb{N}_0, u \in \mathbb{R} \right\rbrace \Big)
                \end{aligned}
            \end{equation*}
            is a subalgebra of $\mathcal{B}_\psi(\widehat{C}^\alpha_{d,T})$ (using the  trigonometric identities \eqref{EqThmUATProof3}) that is point separating and nowhere vanishing of $\psi$-moderate growth (as each $a \in \mathcal{A}$ is bounded). Moreover, since every $w \in \mathcal{W}$ is also bounded, $\mathcal{A}$ is also point separating and nowhere vanishing of $\psi_w$-moderate growth, for all $w \in \mathcal{W}$. In addition, by using
            \begin{equation*}
                \begin{aligned}
                    \cos(s) e^{\mathbf{i} s} & = \cos(s) \cos(s) + \mathbf{i} \cos(s) \sin(s) = \frac{1}{2} + \frac{1}{2} \cos(2s) + \frac{\mathbf{i}}{2} \sin(2s) = \frac{1}{2} + \frac{1}{2} e^{2 \mathbf{i} s}, \\
                    \sin(s) e^{\mathbf{i} s} & = \sin(s) \cos(s) + \mathbf{i} \sin(s) \sin(s) = \frac{1}{2} \sin(2s) + \frac{\mathbf{i}}{2} - \frac{\mathbf{i}}{2} \cos(2s) = \frac{\mathbf{i}}{2} - \frac{\mathbf{i}}{2} e^{2 \mathbf{i} s},
                \end{aligned}
            \end{equation*}
            for any $s \in \mathbb{R}$, we observe that $\mathcal{W}$ is an $\mathcal{A}$-submodule. Hence, $\mathcal{W}$ is by the weighted $\mathbb{C}$-valued Stone-Weierstrass theorem (Theorem \ref{ThmStoneWeierstrassBpsi}) dense in $\mathcal{B}_\psi(\widehat{C}^\alpha_{d,T};\mathbb{C})$. We therefore have universality of the corresponding feature map and that
	 		\begin{equation*}
	 			\left\lbrace \mathbb{E}\left[ \exp\left( \mathbf{i} u \langle e_I, \widehat{\mathbb{X}}_T \rangle \right) \right] : I \in \lbrace 0,\ldots,d \rbrace^N, \, N \in \mathbb{N}_0, u \in \mathbb{R} \right\rbrace
	 		\end{equation*}
			characterizes the law of a stochastic process $\widehat{\mathbf{X}}$ on path space without exponential moment conditions. Observe that for the generic class of signature SDEs, methods from affine  processes can be used to compute $\mathbb{E}\big[ \exp\big( \mathbf{i} u \langle e_I, \widehat{\mathbb{X}}_T \rangle \big) \big]$ analytically (see \cite{cuchiero2023signature}). Note that this characteristic function differs from the notion used in \cite{Chevyrev16}, where other characteristic functions for probability measures on the signatures of geometric rough paths are considered with the main example of expected signature.
    	\end{enumerate}
    \end{remark}

    \subsection{Global universal approximation of $p$-variation rough paths}
	\label{sec:prough}

    In this section, we provide the universal approximation result for weakly geometric $p$-variation rough paths (see \cite[Definition~9.15~(i)]{friz10}). However, in order to get a weighted space, we need to consider the subset of weakly geometric $p$-variation rough paths which are also H\"older continuous (see Example~\ref{ExWeightedSpaces}~\ref{ExWeightedSpacePVarPredual} and \ref{ExWeightedSpacePVarWeaker} as well as Remark~\ref{RemWeakerTop}).
	
	\begin{definition}
	    For $(p,\alpha) \in (1,\infty) \times (0,1)$ with $p \alpha < 1$, a continuous path $\mathbf{X}: [0,T] \rightarrow G^{\lfloor 1/\alpha \rfloor}(\mathbb{R}^d)$ of the form
	    \begin{equation*}
	        [0,T] \ni t \quad \mapsto \quad \mathbf{X}_t := (1,X_t,\mathbb{X}^{(2)}_t,\ldots,\mathbb{X}^{(\lfloor 1/\alpha \rfloor)}_t) \in G^{\lfloor 1/\alpha \rfloor}(\mathbb{R}^d)
	    \end{equation*}
	    with $\mathbf{X}_0 = \mathbf{1} := (1, 0 \ldots, 0) \in G^{\lfloor 1/\alpha \rfloor}(\mathbb{R}^d)$ is called a \emph{weakly geometric $(p,\alpha)$-rough path} if the $(p,\alpha)$-norm
	    \begin{equation*}
	        \Vert \mathbf{X} \Vert_{cc,p-var,\alpha} := \sup_{s,t \in [0,T] \atop s < t} \frac{d_{cc}(\mathbf{X}_s,\mathbf{X}_t)}{\vert s-t \vert^\alpha} + \left( \sup_{(t_i) \in \mathcal{D}([0,T])} \sum_i d_{cc}(\mathbf{X}_{t_i},\mathbf{X}_{t_{i+1}})^p \right)^\frac{1}{p}
	    \end{equation*}
	    is finite. We denote by $C^{p-var,\alpha}_0([0,T];G^{\lfloor 1/\alpha \rfloor}(\mathbb{R}^d))$ the space of such \emph{weakly geometric $(p,\alpha)$-rough paths}, which we equip with the metric
        \begin{equation}
            \label{EqDefWGRPPVar1}
            d_{cc,p'-var,\alpha'}(\mathbf{X},\mathbf{Y}) := \sup_{s,t \in [0,T] \atop s < t} \frac{d_{cc}(\mathbf{X}_{s,t},\mathbf{Y}_{s,t})}{\vert s-t \vert^{\alpha'}} + \left( \sup_{(t_i) \in \mathcal{D}([0,T])} \sum_i d_{cc}(\mathbf{X}_{t_i,t_{i+1}},\mathbf{Y}_{t_i,t_{i+1}})^{p'} \right)^\frac{1}{p'},
        \end{equation}
       for $\mathbf{X}, \mathbf{Y} \in C^{p-var,\alpha}_0([0,T];G^{\lfloor 1/\alpha \rfloor}(\mathbb{R}^d))$ and $(p',\alpha') \in [p,\infty] \times [0,\alpha]$ with $p' \alpha' < 1$, where $\mathbf{X}_{s,t} := \mathbf{X}_s^{-1} \otimes \mathbf{X}_t \in G^{\lfloor 1/\alpha \rfloor}(\mathbb{R}^d)$. For $p' = \infty$, we omit the $p'$-variation part in \eqref{EqDefWGRPPVar1}, and for $(p',\alpha') = (\infty,0)$, the metric $d_{cc,\infty-var,0}$ is equivalent to $d_{cc,\infty}$.
	\end{definition}

    Note that if $\mathbf{X} \in C^{p-var,\alpha}_0([0,T];G^{\lfloor 1/\alpha \rfloor}(\mathbb{R}^d))$ is a weakly geometric $(p,\alpha)$-rough path, then $\mathbb{X}^{\lfloor p \rfloor}$ is a weakly geometric $p$-variation rough path (see \cite[Definition~9.15~(i)]{friz10}). Moreover, the truncated signature of order $N > \lfloor 1/\alpha \rfloor$ is defined as the unique Lyons extension to a path $\mathbb{X}^N: [0,T] \rightarrow G^N(\mathbb{R}^d)$, see \cite[Corollary~9.11]{friz10}. 

    In order to relate the rough path metrics $d_{cc,p'-var,\alpha'}$ to the $C^{p'-var,\alpha'}$-topologies on $C^{p-var,\alpha}$-spaces (see Section~\ref{SecNotation}), we introduce for every $(p',\alpha') \in (p,\infty] \times [0,\alpha)$ with $p' \alpha' < 1$ the embedding
    \begin{equation}
        \label{EqRPPVarEmb}
        \begin{aligned}
            & (C^{p-var,\alpha}_0([0,T];G^{\lfloor 1/\alpha \rfloor}(\mathbb{R}^d)),d_{cc,p'-var,\alpha'}) \ni \mathbf{X} \quad \mapsto \\
            & \quad\quad \big( \mathbf{X}^{(1)}, \ldots, \mathbf{X}^{(\lfloor 1/\alpha \rfloor)} \big) \in \bigoplus_{n=1}^{\lfloor 1/\alpha \rfloor} (C^{p-var,n \alpha}_0([0,T];(\mathbb{R}^d)^{\otimes n}),\tau_{p'-var,n\alpha'}).
        \end{aligned}
    \end{equation}
    By following the arguments of Lemma~\ref{LemmaRPHoelderEmb}, the map \eqref{EqRPPVarEmb} is continuous with continuous left-inverse. Moreover, following Remark~\ref{RemRPHoelderWeakStar}, we equip $C^{p-var,\alpha}_0([0,T];G^{\lfloor 1/\alpha \rfloor}(\mathbb{R}^d))$ with a weak-$*$-topology defined as the product topology on $\bigoplus_{n=1}^{\lfloor 1/\alpha \rfloor} (C^{p-var,n \alpha}_0([0,T];(\mathbb{R}^d)^{\otimes n}),\tau_{w^*}^{(n)})$, where each space $C^{p-var,n \alpha}_0([0,T];(\mathbb{R}^d)^{\otimes n})$ is a dual Banach space carrying a weak-$*$-topology (see Appendix~\ref{AppPVarPredual}).
    
    In addition, we add again an additional time component by defining the subset
	\begin{equation*}
		\widehat{C}^{p-var,\alpha}_{d,T} := \left\lbrace \widehat{\mathbf{X}} \in C^{p-var,\alpha}_0([0,T];G^{\lfloor 1/\alpha \rfloor}(\mathbb{R}^{d+1})): \widehat{X}_t = (t,X_t), \text{ for all } t \in [0,T] \right\rbrace,
	\end{equation*}
	of $C^{p-var,\alpha}_0([0,T];G^{\lfloor 1/\alpha \rfloor}(\mathbb{R}^{d+1}))$, where the $0$-th coordinate denotes running time. Now, in order to get a weighted space, we choose similarly as above the weight function $\psi: \widehat{C}^{p-var,\alpha}_{d,T} \rightarrow (0,\infty)$ defined by
	\begin{equation}
        \label{eq:weight_pvar}
		\psi(\widehat{\mathbf{X}}) = \exp\left( \beta \Vert \widehat{\mathbf{X}} \Vert_{cc,p-var,\alpha}^\gamma \right),
	\end{equation}
	for $\widehat{\mathbf{X}} \in \widehat{C}^{p-var,\alpha}_{d,T}$ and some $\beta > 0$ and $\gamma \geq \lfloor 1/\alpha \rfloor$. Then, by the arguments as for $\widehat{C}^\alpha_{d,T}$ but now using Remark~\ref{RemPVar0}, we conclude that $\psi: \widehat{C}^{p-var,\alpha}_{d,T} \rightarrow (0,\infty)$ is admissible on both spaces $(\widehat{C}^{p-var,\alpha}_{d,T},\tau_{w^*})$ and $(\widehat{C}^{p-var,\alpha}_{d,T},d_{cc,p'-var,\alpha'})$, with $(p',\alpha') \in (p,\infty] \times [0,\alpha)$ and $p' \alpha' < 1$. Moreover, the space $\mathcal{B}_\psi(\widehat{C}^{p-var,\alpha}_{d,T})$ is independent of the choice of topology on the underlying space $\widehat{C}^{p-var,\alpha}_{d,T}$ (i.e., the weak-$*$-topology or any $C^{p'-var,\alpha'}$-topology induced by the metric $d_{cc,p'-var,\alpha'}$, with $(p',\alpha') \in (p,\infty] \times [0,\alpha)$ and $p' \alpha' < 1$) in the sense that it always contains the same functions.

	Now, we present our global universal approximation theorem for linear functions of the signature in $\mathcal{B}_\psi(\widehat{C}^{p-var,\alpha}_{d,T})$.

    \begin{corollary}[Universal approximation on $\mathcal{B}_\psi(\widehat{C}^{p-var,\alpha}_{d,T})$]
		\label{CorUATPVar}
		The linear span of the set
        \begin{equation*}
            \left\lbrace \widehat{\mathbf{X}} \mapsto \langle e_I, \widehat{\mathbb{X}}_T \rangle: I \in \lbrace 0,\ldots,d \rbrace^N, \, N \in \mathbb{N}_0 \right\rbrace
        \end{equation*}
        is dense in $\mathcal{B}_\psi(\widehat{C}^{p-var,\alpha}_{d,T})$, i.e.~for every map $f \in \mathcal{B}_\psi(\widehat{C}^{p-var,\alpha}_{d,T})$ and for every $\varepsilon > 0$ there exists a linear function of the form $\widehat{\mathbb{X}}_T \mapsto l(\widehat{\mathbb{X}}_T) := \sum_{0 \leq \vert I \vert \leq N} a_I \langle e_I, \widehat{\mathbb{X}}_T \rangle$, with some $N \in \mathbb{N}_0$ and $a_I \in \mathbb{R}$, such that
		\begin{equation*}
			\sup_{\widehat{\mathbf{X}} \in \widehat{C}^{p-var,\alpha}_{d,T}} \frac{\left\vert f(\widehat{\mathbf{X}}) - l(\widehat{\mathbb{X}}_T) \right\vert}{\psi(\widehat{\mathbf{X}})} < \varepsilon.
		\end{equation*}
	\end{corollary}
    \begin{proof}
      The proof follows along the lines of the proof of Theorem~\ref{ThmUATHoelder}. In particular,  we use that $\widehat{C}^{p-var,\alpha}_{d,T} \subsetneq \widehat{C}^\alpha_{d,T}$ to conclude that $(K_R,d_{cc,\infty-var,0}) \ni \widehat{\mathbf{X}} \mapsto a(\widehat{\mathbf{X}}) = \langle e_I, \widehat{\mathbb{X}}_T \rangle $ is continuous. Moreover, the additional $p$-variation term in \eqref{EqDefWGRPPVar1} can be just added to the $\alpha$-H\"older norm for the estimates in \eqref{EqThmUATHoelderProof3} and \eqref{EqThmUATHoelderProof5}, which imply the analogous limits as in \eqref{EqThmUATHoelderProof4} and \eqref{EqThmUATHoelderProof6}. The remainder of the proof is identical.
    \end{proof}

    Corollary~\ref{CorUATPVar} shows that every path space functional in $\mathcal{B}_\psi(\widehat{C}^{p-var,\alpha}_{d,T})$ can be approximated with linear functions of the signature. Moreover, by following Remark~\ref{RemSigChar}, we observe that the signature is characteristic for the dual of $\mathcal{B}_\psi(\widehat{C}^{p-var,\alpha}_{d,T})$.

    \section{Gaussian process regression with applications to signature kernels}
    \label{SecKernels}
    
    As an important counterpart to the density results established so far, which make linear regressions feasible, we introduce a Gaussian process perspective, or, equivalently, the perspective of reproducing kernel Hilbert spaces on regression. Even though this can be done for general input spaces and general Banach spaces of functions thereon, we specialize here to weighted spaces as input spaces and $\mathcal{B}_\psi$-spaces as spaces of functions thereon. An important application is given through signature kernel regression on path spaces. The corresponding (general) signature kernels have already been considered in \cite{kiraly19, salvi2021signature, cass2021general} and are of the form
 	\begin{equation}
 		\label{eq:sigkernel}
    	k(\widehat{\mathbf{X}},\widehat{\mathbf{Y}})=\sum_{I} a^2_{I} \langle e_I, \widehat{\mathbb{X}}_T \rangle \langle e_I, \widehat{\mathbb{Y}}_T \rangle,
    \end{equation}
	for appropriate choices of $a_I$. We provide here a novel Gaussian process perspective in terms of $\mathcal{B}_\psi$-spaces, which together with the above results allows to treat approximation with the true signature kernels (without a normalization procedure as considered in \cite{chevyrev22}) rigorously. Consequently, we expect these results to be useful for uncertainty quantification and a theory of regularization for regressions on signature components.

    Let $(X,\psi)$ be a general weighted space. A reproducing kernel Hilbert space $H$ on $(X,\psi)$ is a Hilbert space $H$ consisting of maps $f: X \rightarrow \mathbb{R}$, which is continuously embedded into $\mathcal{B}_\psi(X)$, i.e.~$H \hookrightarrow \mathcal{B}_\psi(X)$. Its kernel $k: X \times X \rightarrow \mathbb{R}$ is uniquely defined through $k(x,\cdot) \in H$ with $\int_X f(y) \delta_x(dy) = f(x) = \langle k(x,\cdot),f \rangle$ for all $f \in H$, i.e.~the functional representation in $H$ of the Dirac measure $\delta_x$ located at $x$ via Riesz representation. Whence the kernel is symmetric, positive semi-definite and $k(x,\cdot) \in \mathcal{B}_\psi(X) $ for all $x \in X$. If, on the other hand, we are given a symmetric, positive semi-definite function $k$ such that $k(x,\cdot) \in \mathcal{B}_\psi(X)$ for all $x \in X$, and such that the linear span of $k(x,\cdot)$ equipped with the pre-inner product $\langle k(x_1,\cdot), k(x_2,\cdot) \rangle := k(x_1,x_2)$ is continuously embedded in $\mathcal{B}_\psi(X)$, then the span's closure is a reproducing kernel Hilbert space on $(X,\psi)$ with kernel $k$. 

    \begin{definition}
        A kernel $k: X \times X \rightarrow \mathbb{R}$ is called \emph{universal} with respect to $\mathcal{B}_{\psi}(X)$ if the corresponding reproducing kernel Hilbert space $H$ on $(X,\psi)$ is dense in $\mathcal{B}_\psi(X)$ in its topology.
    \end{definition}
    
    \begin{example}
        Consider the setting $X := \widehat{C}^\alpha_{d,T}$ or $X := \widehat{C}^{p-var,\alpha}_{d,T}$ of Section~\ref{sec:hoelder} or \ref{sec:prough} with weight function $\psi: X \rightarrow (0,\infty)$ defined in \eqref{eq:weight_holder} or \eqref{eq:weight_pvar}, respectively. Then, the signature kernel $k: X \times X \rightarrow \mathbb{R}$ of the form \eqref{eq:sigkernel} is symmetric and positive semi-definite. Moreover, since $\big(\widehat{\mathbf{X}} \mapsto \langle e_I, \widehat{\mathbb{X}}_T \rangle\big) \in \mathcal{B}_\psi(X)$ for all $I$ (see the proof of Theorem~\ref{ThmUATHoelder}), we have $k(\widehat{\mathbf{X}},\cdot) \in \mathcal{B}_\psi(X)$ for all $\widehat{\mathbf{X}} \in X$. In addition, for every $\sum_{n=1}^N c_n k(\widehat{\mathbf{X}}^{(n)},\cdot) \in H_0 := \linspan\big(\big\lbrace k(\widehat{\mathbf{X}},\cdot): \widehat{\mathbf{X}} \in X \big\rbrace\big)$ with $c_n \in \mathbb{R}$ and $\widehat{\mathbf{X}}^{(n)} \in X$, we apply the Cauchy-Schwarz inequality and assume that $(a_I)_I$ are such that $C := \sup_{\widehat{\mathbf{Y}} \in X} \frac{\left\vert \sum_I a_I^2 \langle e_I, \widehat{\mathbb{Y}}_T \rangle^2 \right\vert}{\psi(\widehat{\mathbf{Y}})^2} < \infty$ to obtain
        \begin{equation*}
            \begin{aligned}
                \left\Vert \sum_{n=1}^N c_n k(\widehat{\mathbf{X}}^{(n)},\cdot) \right\Vert_{\mathcal{B}_\psi(X)}^2 & = \sup_{\widehat{\mathbf{Y}} \in X} \frac{\left\vert \sum_I a_I^2 \sum_{n=1}^N c_n \langle e_I, \widehat{\mathbb{X}}^{(n)}_T \rangle \langle e_I, \widehat{\mathbb{Y}}_T \rangle \right\vert^2}{\psi(\widehat{\mathbf{Y}})^2} \\
                & \leq \sup_{\widehat{\mathbf{Y}} \in X} \frac{\left\vert \sum_I a_I^2 \langle e_I, \widehat{\mathbb{Y}}_T \rangle^2 \right\vert}{\psi(\widehat{\mathbf{Y}})^2} \sum_{n,m=1}^N c_n c_m \sum_I a_I^2 \langle e_I, \widehat{\mathbb{X}}^{(n)}_T \rangle \langle e_I, \widehat{\mathbb{X}}^{(m)}_T \rangle \\
                & = C \Big\langle \sum_{n=1}^N c_n k(\widehat{\mathbf{X}}^{(n)},\cdot), \sum_{m=1}^N c_m k(\widehat{\mathbf{X}}^{(m)},\cdot) \Big\rangle = C \left\Vert \sum_{n=1}^N c_n k(\widehat{\mathbf{X}}^{(n)},\cdot) \right\Vert^2,
            \end{aligned}
        \end{equation*}
        where the norm $\Vert \cdot \Vert$ on $H_0$ is induced by the pre-inner product $\langle k(\widehat{\mathbf{X}}_1,\cdot), k(\widehat{\mathbf{X}}_2,\cdot) \rangle := k(\widehat{\mathbf{X}}_1,\widehat{\mathbf{X}}_2)$. This shows that $H_0$ equipped with this pre-inner product is continuously embedded in $\mathcal{B}_\psi(X)$, whence $H := \overline{H_0}$ is the corresponding reproducing kernel Hilbert space on $(X,\psi)$. Thus, by using that the feature map $\Phi(\widehat{\mathbf{X}}) := \widehat{\mathbb{X}}_T$ is universal with respect to $\mathcal{B}_\psi(X)$ (see Remark \ref{RemSigChar}~\ref{RemSigChar1}), we can apply \cite[Proposition 29+43]{chevyrev22} to conclude that the signature kernel $k: X \times X \rightarrow \mathbb{R}$ of the form \eqref{eq:sigkernel} is universal with respect to $\mathcal{B}_{\psi}(X)$.
	\end{example}

    Now, we assume that $(\Omega,\mathcal{F},\mathbb{P})$ is a probability space and that $\mathcal{B}_\psi(X)$ is separable. Notice that separability is never a restriction if we are given a countable set of functions (as signature components) whose span (built with rational coefficients) is dense in $\mathcal{B}_\psi(X)$, which is the typical situation for linear regression. Moreover, we assume that $Z$ is a $\mathcal{B}_\psi(X)$-valued centered Gaussian process, i.e.~a measurable map $Z: \Omega \rightarrow \mathcal{B}_\psi(X)$ such that for every continuous linear functional $l \in \mathcal{B}_\psi(X)^*$ the real-valued random variable $l \circ Z: \Omega \rightarrow \mathbb{R}$ follows a centered normal distribution (see, e.g., \cite[Chapter~3]{ledoux91}).
    
    Next, we construct the \emph{Cameron-Martin space} (see \cite{cameron44,cameron45}). Indeed, we consider the continuous covariance operator $\mathcal{M}_{\psi}(X) \ni \mu \mapsto C_Z(\mu) := \mathbb{E}\left[ Z \, \int_X Z(x) \mu(dx) \right] \in \mathcal{B}_{\psi}(X)$, where the latter is understood in the sense of Bochner integration. Then, we define the \emph{Cameron-Martin space of $Z$}, denoted by $H_Z$, as the closure of $\left\lbrace C_Z(\mu): \mu \in \mathcal{M}_{\psi}(X) \right\rbrace \subseteq \mathcal{B}_{\psi}(X)$ under the inner product $\langle C_Z(\mu_1),C_Z(\mu_2) \rangle := \mathbb{E}\left[ \int_X Z(x)\mu_1(dx)\int_X Z(x)\mu_2(dx) \right]$. 
    
    Now, by Fernique's theorem (see \cite{fernique70} or \cite[Theorem~2.8.5]{bogachev98}), there exists some $\beta > 0$ with $\mathbb{E}\big[ \exp\big( \beta \Vert Z \Vert_{\mathcal{B}_\psi(X)}^2 \big) \big] < \infty$, which implies that $\mathbb{E}\big[ \Vert Z \Vert_{\mathcal{B}_{\psi}(X)}^2 \big] < \infty$. Hence, we have
    \begin{equation*}
    	\Vert C_Z(\mu) \Vert_{\mathcal{B}_\psi(X)} = \Bigg\Vert \mathbb{E}\left[ Z \, \int_X Z(x) \mu(dx) \right] \Bigg\Vert_{\mathcal{B}_\psi(X)} \leq \mathbb{E}\left[ \Vert Z \Vert_{\mathcal{B}_{\psi}(X)}^2 \right]^\frac{1}{2} \Vert C_Z(\mu) \Vert_{H_Z},
    \end{equation*}
    which shows that $H_Z$ is continuously embedded in $\mathcal{B}_\psi(X)$ and forms therefore a reproducing kernel Hilbert space on $(X,\psi)$. Its kernel $k$ corresponds precisely to the point covariances of $Z$, i.e.~$k(x_1,x_2) = \langle C_Z(\delta_{x_1}),C_Z(\delta_{x_2}) \rangle = \mathbb{E}[Z(x_1)Z(x_2)]$ for $x_1,x_2 \in X$, which in turn determines $C_Z$ via the formula $C_Z(\mu) = \int_X k(x,\cdot)\mu(dx)$. 
    
    Notice also that we can define a (real-valued) \emph{stochastic integral} $H_Z \ni f \mapsto (f \bullet Z) \in L^2(\Omega)$ by isometrically extending $\mathbb{E}\big[ Z \int_X Z(x) \mu(dx) \big] \mapsto \int_X Z(x) \mu(dx)$ to its completion $H_Z$ (Ito's isometry). Indeed, for elements $g$ of the form $\mathbb{E}[Z \int_X Z(x) \mu(dx)]$, we have $\mathbb{E}\left[ (g \bullet Z)^2 \right]= \Vert C_Z(\mu) \Vert_{H_Z}^2$, which can be extended to elements $f$ in the completion $H_Z$.
    
    Then, the Cameron-Martin theorem (see, e.g., \cite[Theorem~4.44]{hairer09}) states for $f \in \mathcal{B}_\psi(X)$ that $f \in H_Z$ if and only if the law of $f+Z$ is absolutely continuous to the law of $Z$. The well-known Radon-Nikodym derviative is given by $\exp\left( (f \bullet Z) - \frac{1}{2} \Vert f \Vert_{H_Z}^2 \right)$ (see \cite[Equation~4.14]{hairer09}).

    \begin{remark}
    	In this context the meaning of the Gaussian process $Z$ also corresponds to  the generation of a prior distribution of functions, which are added to the searched function in question.
    \end{remark}

    \begin{example}
    	For the purpose of illustration we consider $X = [0,1]$ with $\psi = 1$. Then $\mathcal{B}_{\psi}(X) = C^0([0,1])$. If $k(s,t) = s \wedge t := \min(s,t)$ for $s,t \in [0,1]$, then a corresponding Gaussian process $Z$ with values in $C^0([0,1])$ can be constructed, namely the Wiener process. Its reproducing kernel Hilbert space is $H^1_0([0,1])$ ($H^1$-functions being $0$ at $0$ with inner product $\langle f,g \rangle := \int_0^1 f'g'$) as it is the closure of elements of the form $s \mapsto \mathbb{E}[Z(s) \int_{0}^1 Z(t)\mu(dt)] = \int_0^1 (s \wedge t) \mu(dt)$, for $\mu$ being a Radon measure, with respect to the inner product $\langle C_Z(\mu_1), C_Z(\mu_2) \rangle := \int_0^1 \int_0^1 (s \wedge t) \mu_1(ds) \mu_2(dt)$. Apparently this coincides with the closure of the span of $s \mapsto s \wedge t$ for different $t$, which connects it to the reproducing kernel Hilbert space generated by $k$.
    \end{example}

    With these preparations we can now formulate the main theorem of this section.

    \begin{theorem}
    	\label{ThmGaussianProcess}
    	Given a kernel $k$ on $(X,\psi)$ and a countable system of functions $(h_i)_i$ in $\mathcal{B}_\psi(X)$ such that $k = \sum_i h_i \otimes h_i $ as function on $X \times X$ and $\sum_i \Vert h_i \Vert_{\mathcal{B}_\psi(X)} < \infty$ with $\mathcal{B}_\psi(X)$ separable. Then the random series $\sum_i Z_i h_i$, for an i.i.d.~sequence of standard Gaussian random variables $(Z_i)_i$, converges in the $\mathcal{B}_{\psi}(X)$-norm to a Gaussian process $Z$ with values in $\mathcal{B}_\psi(X)$ and with kernel $k$. In particular its Cameron-Martin space coincides with the reproducing kernel Hilbert space generated by $k$.
    \end{theorem}

    \begin{proof}
        Let $(Z_i)_i$ be an i.i.d.~sequence of standard Gaussian random variables. Then, 
        \begin{equation*}
            \mathbb{E}\left[ \sum_i \vert Z_i \vert \Vert h_i \Vert_{\mathcal{B}_\psi(X)} \right] \simeq \sum_i \Vert h_i \Vert_{\mathcal{B}_\psi(X)} < \infty,
        \end{equation*}
        implies that $\sum_i \vert Z_i \vert \Vert h_i \Vert_{\mathcal{B}_\psi(X)} < \infty$, $\mathbb{P}$-a.s., where ``$\simeq$'' means that inequality holds true in both ways up to a (different) multiplicative constant. Hence, we can apply the Ito-Nisio theorem (see, e.g., \cite[Theorem~2.4]{ledoux91} or \cite[Theorem~3.5.1]{bogachev98}) to conclude that $\sum_i Z_i h_i$ is again a Gaussian random variable. Its kernel can be easily calculated by pointwise covariances
    	\begin{equation*}
    		\mathbb{E}[Z(x_1)Z(x_2)]=\sum_i h_i(x_1) h_i(x_2) 
    	\end{equation*}
    	and coincides with $k$ by assumption. The pointwise covariances, however, determine $C_Z$ and in turn the Cameron-Martin space as stated above.
    \end{proof}
    
    With Theorem~\ref{ThmGaussianProcess} uncertainty quantification for regression with signature kernels is feasible, since for reasonable choices of $\psi$ the main condition of Theorem~\ref{ThmGaussianProcess} is satisfied.
    
    \begin{remark}
    	Let us illustrate Theorem~\ref{ThmGaussianProcess} in the setting of Section~\ref{sec:prough}:
    	\begin{enumerate}
    		\item Take $X = \widehat{C}^{p-var,\alpha}_{d,T}$ together with 
	    	\begin{equation*}
				\psi(\widehat{\mathbf{X}}) = \exp\left( \beta \Vert \widehat{\mathbf{X}} \Vert_{cc,p-var,\alpha}^\gamma \right),
			\end{equation*}
    		for $\beta > 0 $ and $\gamma > 1$.\footnote{Note that in contrast to Section \ref{sec:prough}, where we had to take $\gamma \geq \lfloor 1/\alpha \rfloor$, it here suffices to take $\gamma >1$ for all $\alpha \in (0,1]$ since we do not apply the weighted Stone-Weierstrass theorem here.} As countable system of functions $(h_I)_I$ for multi-indices $I= \lbrace 0,1, \ldots, d \rbrace^N$ and $N \in \mathbb{N}_0$, we choose
    		\begin{equation*}
    			h_I(\widehat{\mathbf{X}}) = a_{I} \langle e_I, \widehat{\mathbb{X}}_T \rangle,
    		\end{equation*}
    		where $a_I$ are real numbers. More precisely, for $0 < \eta < 1$ satisfying $\gamma(1-\eta) > 1$, we require for $a_I$ that
    		\begin{equation}
    			\label{eq:conda}
    			a_k:=\max_{|I|=k} |a_I| \leq M (k!)^{\delta}
    		\end{equation}    
     		with $\delta < \eta$ and some constant $M>0$. To prove that $\sum_I \vert a_{I} \vert \sup_{\widehat{\mathbf{X}}} \frac{\left\vert \langle e_I, \widehat{\mathbb{X}}_T \rangle \right\vert}{\psi(\widehat{\mathbf{X}})} < \infty$, note that by \cite[Theorem~3.7]{lyons07} (with control $\omega(0,T) := \Vert \widehat{\mathbf{X}} \Vert_{cc,p-var}^p$) we have
     		\begin{equation*}
     			\Vert \widehat{\mathbb{X}}^{(k)}_T \Vert_{(\mathbb{R}^{d+1})^{\otimes k}} \leq \frac{C_p \Vert \widehat{\mathbf{X}} \Vert_{cc,p-var}^k}{k!},
     		\end{equation*}
     		where $C_p > 0$ is some constant depending on $p$ and 
     		\begin{equation*}
     			\Vert \widehat{\mathbf{X}} \Vert_{cc,p-var} = \left( \sup_{(t_i) \in \mathcal{D}([0,T])} \sum_i d_{cc}(\widehat{\mathbf{X}}_{t_i},\widehat{\mathbf{X}}_{t_{i+1}})^p \right)^\frac{1}{p},
     		\end{equation*}
     		whence also
     		\begin{equation*}
     			\Vert \widehat{\mathbb{X}}^{(k)}_T \Vert_{(\mathbb{R}^{d+1})^{\otimes k}} \leq \frac{C_p \Vert \widehat{\mathbf{X}} \Vert_{cc,p-var,\alpha}^k}{k!}
     		\end{equation*}
			holds. Using $a_k := \max_{|I|=k} \vert a_{I} \vert$, H\"older's inequality with $1/r+1/\gamma=1$, and $0 < \eta < 1$ with $\gamma(1-\eta) > 1$, we can estimate
		    \begin{equation*}
		    	\begin{aligned}
		    		\sum_I \vert a_{I} \vert \sup_{\widehat{\mathbf{X}}} \frac{\left\vert \langle e_I, \widehat{\mathbb{X}}_T \rangle \right\vert}{\psi(\widehat{\mathbf{X}})} & \leq 
		    		\sum_{k=0}^{\infty} a_k \sup_{\widehat{\mathbf{X}}} \frac{(d+1)^k C_p \Vert \widehat{\mathbf{X}} \Vert_{cc,p-var,\alpha}^k}{k! \psi(\widehat{\mathbf{X}})}\\
		    		&\leq \left( \sum_{k=0}^{\infty} \frac{a_k^r (d+1)^{rk} C_p^r}{k!^{r \eta}}\right)^{\frac{1}{r}} \left(\sum_{k=0}^{\infty} \frac{\Vert \widehat{\mathbf{X}} \Vert_{cc,p-var,\alpha}^{\gamma k}}{k!^{\gamma (1- \eta)} \exp\left( \beta \gamma \Vert \widehat{\mathbf{X}} \Vert_{cc,p-var,\alpha}^\gamma \right)}\right)^{\frac{1}{\gamma}}\\
		    		&\leq \left( \sum_{k=0}^{\infty} \frac{a_k^r (d+1)^{rk} C_p^r}{k!^{r \eta}}\right)^{\frac{1}{r}} \left(\sum_{k=0}^{\infty} \frac{1}{k!^{\gamma(1- \eta)-1} \beta^k \gamma^k} \right)^{\frac{1}{\gamma}} < \infty,
		    	\end{aligned}
		    \end{equation*}
   			where the second last inequality follows from the estimate $\exp(x)\geq \frac{x^k}{k!}$ for $x \geq 0$. The root test then implies that the expression is finite. Indeed, noticing that for $\epsilon > 0$, we have $(k!)^{\epsilon/k} \to \infty$ as $k \to \infty$, it follows that the first term is finite due to the assumption that $a_k=\max_{|I|=k} \vert a_I \vert \leq M (k!)^{\delta}$ with $\delta < \eta$ and some constant $M>0$. Similarly the second term is finite as well. 
     
     		Notice also that one could replace $\psi$ by a function of an absolutely converging sum of the form
     		\begin{equation*}
     			\psi(\widehat{\mathbf{X}}) = \left(\sum_{k = 0}^{\infty} \frac{b_k \Vert \widehat{\mathbf{X}} \Vert_{cc,p-var,\alpha}^{\gamma k}}{k!}\right)^{\frac{1}{\gamma}}
     		\end{equation*}
    		with positive $b_k$ such that $b_k \geq \frac{M \epsilon ^k}{k!^\kappa}$ for $k \geq 0$, $M>0$, $\epsilon > 0$ and $\kappa < \gamma (1-\eta)$.
    		
    		\item Under the above conditions on $a_I$ and $\psi$, the kernel
    		\begin{equation*}
    			k(\widehat{\mathbf{X}},\widehat{\mathbf{Y}})=\sum_{I} a^2_{I} \langle e_I, \widehat{\mathbb{X}}_T \rangle \langle e_I, \widehat{\mathbb{Y}}_T \rangle 
    		\end{equation*}
    		is thus the covariance of a Gaussian process with values in $\mathcal{B}_\psi(X)$.
    
			Choosing $a_I=a_{|I|}$ identical for all $I$ of the same length and defining
			\begin{equation*}
				K_{a^2}(s,t):=\sum_I a^2_{|I|} \langle e_I, \widehat{\mathbb{X}}_s \rangle \langle e_I, \widehat{\mathbb{Y}}_t \rangle 
			\end{equation*} 
			for bounded variation paths $ \widehat{\mathbf{X}},\widehat{\mathbf{Y}}$, we then get by the same arguments as in \cite[Proposition 2.8]{cass2021general} the following integral equation
			\begin{equation*}
				K_{a^2}(s,t)= a_{\emptyset}^2+ \sum_{i=1}^{d+1} \int_0^s \int_0^t K_{(a^2)_{+1}}(u,v)  d\widehat{\mathbf{X}}^i_u d\widehat{\mathbf{Y}}^i_v, 
			\end{equation*}
			where $(a^2)_{+1}$ denotes the sequence of $a^2$ shifted by $1$, i.e. 
			\begin{equation*}
				K_{(a^2)_{+1}}(s,t) :=\sum_{I} a^2_{\vert I \vert+1} \langle e_I, \widehat{\mathbb{X}}_s \rangle \langle e_I, \widehat{\mathbb{Y}}_t \rangle.
			\end{equation*}
			Note that for both $a^2$ and $a^2_{+1}$ Condition 1 of \cite{cass2021general} which is required in  \cite[Proposition 2.8]{cass2021general} is satisfied due to \eqref{eq:conda}. In particular,  $k(\widehat{\mathbf{X}},\widehat{\mathbf{Y}})$ satisfies
			\begin{equation*}
				k(\widehat{\mathbf{X}},\widehat{\mathbf{Y}})= K_{a^2}(T,T) = a_{\emptyset}^2 + \sum_{i=1}^{d+1} \int_0^T \int_0^T K_{(a^2)_{+1}}(u,v)  d\widehat{\mathbf{X}}^i_u d\widehat{\mathbf{Y}}^i_v. 
			\end{equation*}
  			If $a_{|I|} \equiv 1$ and $\widehat{\mathbf{X}},\widehat{\mathbf{Y}}$ are differentiable, then the original signature kernel satisfies the so-called Goursat PDE
  			\begin{equation*}
  				\frac{\partial^2 K(s,t)}{\partial s \partial t}=\sum_{i=1}^{d+1}  K(s,t) d\widehat{\mathbf{X}}^i_s d \widehat{\mathbf{X}}^i_t,
  			\end{equation*}
			as shown in \cite{salvi2021signature}.
		\end{enumerate}    
   	\end{remark}
   
	\section{Numerical examples}
	\label{SecNE}
	
	In this section, we illustrate with two examples\footnote{The experiments are implemented in \texttt{Python} using \texttt{tensorflow} (for FNN) and \texttt{iisignature} (for signatures) on a Lenovo ThinkPad X13 Gen2a with AMD Ryzen 7 PRO 5850U processor and Radeon Graphics (1901 Mhz, 8 Cores, 16 Logical Processors), see \url{https://github.com/psc25/GlobalUAT}.} how path space functionals can be learned using non-anticipative functional input neural networks (see Section~\ref{SecUANAF}) or a linear function of the signature (see Section~\ref{SecSM}).

    As input data we generate $M = 50 000$ sample paths of a one-dimensional Brownian motion $x^{(m)} = (x^{(m)}(t))_{t \in [0,T]}$, for $m = 1,\ldots,M$, with $T = 1$, which are discretized over $K = 100$ equidistant time points $(t_k)_{k=1,\ldots,K}$. Since the sample paths of Brownian motion are a.s.~$\alpha$-H\"older continuous, for all $\alpha \in (0,1/2)$, we consider the weighted space of stopped $\alpha$-H\"older continuous paths $\Lambda^\alpha_T$ in Example~\ref{ExWeightedSpaces}~\ref{ExWeightedSpaceHoelderStoppedPaths}, with $\alpha < 1/2$. 
    
    On the other hand, every sample path of Brownian motion $x^{(m)}$, $m = 1,\ldots,M$, has finite $p$-variation, $p > 2$, and can therefore be lifted (see \cite{lyons07}) to a weakly geometric $(p,\alpha)$-rough path in $\widehat{C}^{p-var,\alpha}_{1,T}$, for $(p,\alpha) \in (2,\infty) \times (0,1/2)$ with $p \alpha < 1$. Then, we can compute the time-extended signature $\widehat{\mathbb{X}}^{(m)}$ of each $x^{(m)}$.
	
	In the first example, we consider the running maximum
	\begin{equation}
		\label{EqEx1}
		f_1(t,x) = \sup_{s \in [0,t]} x(s),
	\end{equation}
	for $(t,x) \in \Lambda^\alpha_T$, which is continuous with respect to $d_\infty$ (see \cite[Example~9]{contfournie13}). In the second example, we consider the non-anticipative functional $f_2: \Lambda^\alpha_T \rightarrow \mathbb{R}$ defined by
	\begin{equation}
	    \label{EqEx2}
	    f_2(t,x) = \max\left( \frac{1}{t} \int_0^t x(s) ds, -0.3 \right)
	\end{equation}
    for $(t,x) \in \Lambda^\alpha_T$. Since $f_1, f_2: \Lambda^\alpha_T \rightarrow \mathbb{R}$ are continuous and dominated by $\psi_{\text{FNN}}(t,x) := \exp\big( \beta \Vert x \Vert_{C^\alpha([0,T])}^\gamma \big)$, Lemma~\ref{LemmaBpsiEquivChar}~\ref{LemmaBpsiEquivChar2} implies that $f_1, f_2 \in \mathcal{B}_{\psi_{\text{FNN}}}(\Lambda^\alpha_T)$. Similarly, we have $f_1(t,\cdot),f_2(t,\cdot) \in \mathcal{B}_{\psi_{\text{Sig}}}(\widehat{C}^{p-var,\alpha}_{1,T})$, with $\psi_{\text{Sig}}(\widehat{\mathbf{X}}) = \exp\big( \beta \Vert \widehat{\mathbf{X}} \Vert_{cc,p-var,\alpha}^\gamma \big)$, for $\gamma \geq \lfloor 1/\alpha \rfloor$.
    
    We split up the data into $80\%$ for training and $20\%$ for testing, and apply stochastic gradient descent with the Adam algorithm (see \cite{kb15}) over $4000$ epochs with learning rate $10^{-5}$ and batchsize $500$ to minimize the weighted mean squared error
    \begin{align}
        \label{EqDefLossFNN}
        & \frac{1}{MK} \sum_{m=1}^M \sum_{k=1}^K \left( \frac{\left\vert f_i(t_k,x^{(m)}) - \varphi(t_k,x^{(m)}) \right\vert}{\exp\left( \beta \Vert x^{(m)} \Vert_{C^\alpha([0,T])}^\gamma \right)} \right)^2, \quad & & \text{for FNN}, \\
        \label{EqDefLossSig}
        \text{or} \quad & \frac{1}{MK} \sum_{m=1}^M \sum_{k=1}^K \left( \frac{\left\vert f_i(t_k,x^{(m)}) - \sum_{0 \leq \vert I \vert \leq N_{Sig}} a_I \langle e_I, \widehat{\mathbb{X}}^{(m)}_{t_k} \rangle \right\vert}{\exp\left( \beta \Vert \widehat{\mathbf{X}}^{(m)} \Vert_{cc,p-var,\alpha}^\gamma \right)} \right)^2, \quad & & \text{for Sig}.
    \end{align}
	for $i \in \lbrace 1,2 \rbrace$. In both cases, we compute an approximation of $\vert x^{(m)} \vert_\alpha$ and $\Vert \widehat{\mathbf{X}}^{(m)} \Vert_{cc,p-var,\alpha}$ (see the code). Moreover, we choose $\alpha = 0.4$, $p = 2.1$, $\beta = 0.01$, $\gamma = 2.0$, and $N_{Sig} = 7$.
	
	For the FNNs, we consider $\varphi \in \mathcal{FN}^{\rho,\widetilde{\rho}}_{\Lambda^\alpha_T}$ (see Definition~\ref{DefNAFNN}) with $N_{FNN} = 40$ neurons (where $\rho(z) = \widetilde{\rho}(z) = \max(z,0)$ are both the ReLU function and the classical neural networks $(\phi_{n,1})_{n=1,\ldots,N_{FNN}}$ have one hidden layer of $N_1 = 30$ neurons). Moreover, to improve the learning performance, we include a spatial component in $\varphi \in \mathcal{FN}^{\rho,\widetilde{\rho}}_{\Lambda^\alpha_T}$, i.e.
    \begin{equation*}
        \varphi(t,x) = \sum_{n=1}^{N_{FNN}} y_n \rho\left( a_n t + \int_0^t \phi_n(s) x(s) ds + c_n x(t) + b_n \right),
    \end{equation*}
    where $a_1,\ldots,a_{N_{FNN}},b_1,\ldots,b_{N_{FNN}},c_1,\ldots,c_{N_{FNN}} \in \mathbb{R}$, $y_1,\ldots,y_{N_{FNN}} \in \mathbb{R}$, $\phi_1,\ldots,\phi_{N_{FNN}} \in \mathcal{NN}^{\widetilde{\rho}}_{\mathbb{R},\mathbb{R}}$. On the other hand, for $f_2: \Lambda^\alpha_T \rightarrow \mathbb{R}$, we use the FNN of Definition~\ref{DefFuncNN}. In both cases, the time integral $\int_0^t \phi_n(s) x(s) ds$ is numerically computed with a left Riemann sum. The numerical experiments are reported in Figure~\ref{FigEx1}+\ref{FigEx2}.

    Note that weighted mean squared errors in \eqref{EqDefLossFNN}+\eqref{EqDefLossSig} reflect the global nature of the universal approximation results. This is in line with the classical universal approximation theorem on compacta, where one minimizes the unweighted mean squared error to train a neural network. In both cases, the (weighted) mean squared error is an empirical version of a (weighted) $L^2$-norm, being weaker than the (weighted) supremum norm.

    In contrast to UATs on compact subsets, our \emph{global} universal approximation result provides a framework that guarantees the existence of an approximation beyond a given compact set of training data. Indeed, for a pre-specified compact subset (e.g.~paths of Brownian motion), it can never be ensured that the data points of the test set are also contained in this compactum. Our results therefore address this limitation.
    
    In view of specifying a concrete loss function for the training procedure our results do not give a precise strategy how to learn FNNs. This is similar to classical settings where one uses the mean squared error even though the UAT approximation holds for continuous functions. In a similar spirit, we here use a weighted mean squared error which makes errors for largely deviating paths automatically small, thus respecting the weight function appearing in the $\mathcal{B}_\psi(X)$-norm.
	
	\begin{figure}[p]
		\begin{minipage}[t]{0.49\textwidth}
		    \centering
		    \includegraphics[height=5.35cm]{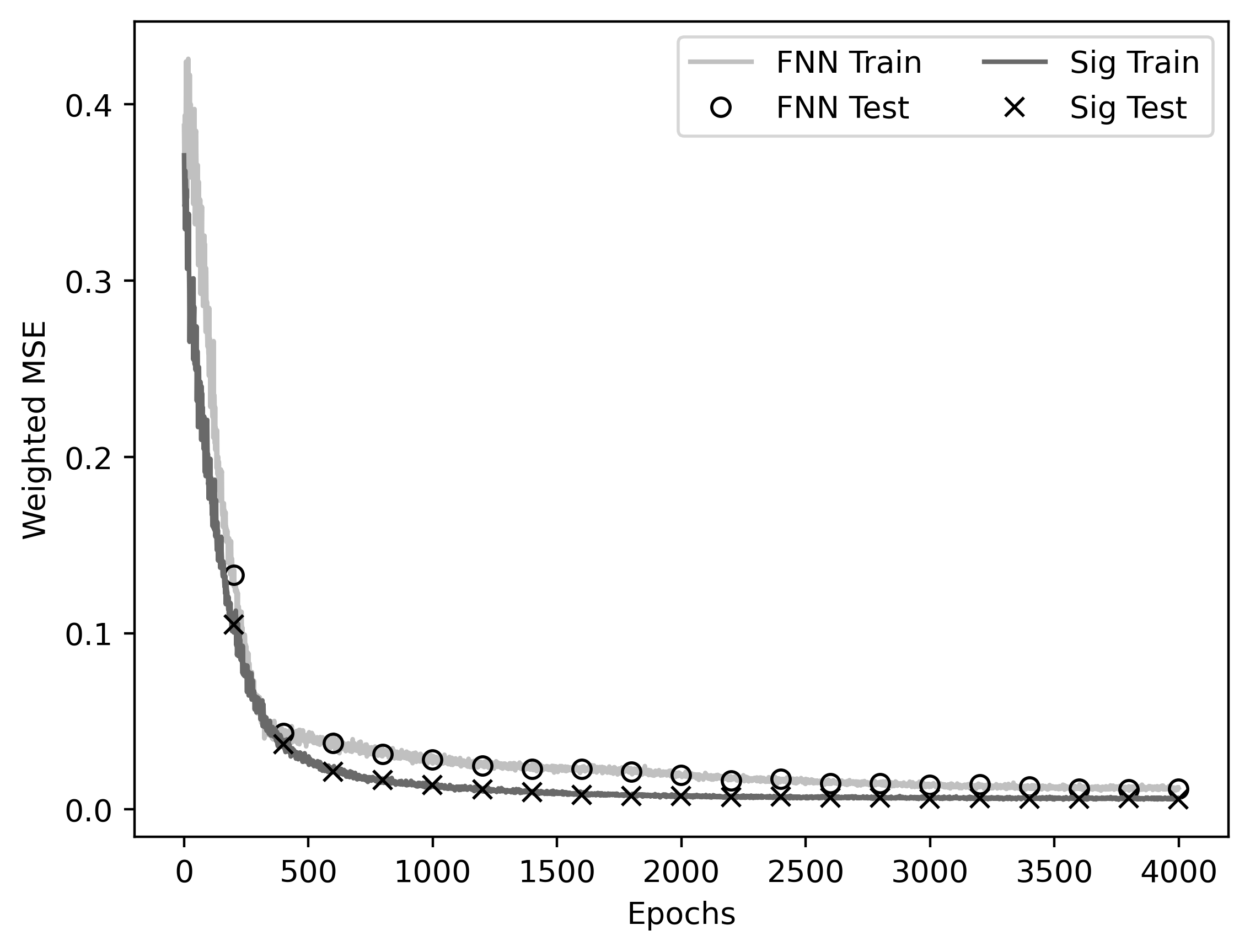}
		    \subcaption{Learning performance}
		\end{minipage}
		\begin{minipage}[t]{0.49\textwidth}
		    \centering
		    \includegraphics[height=5.35cm]{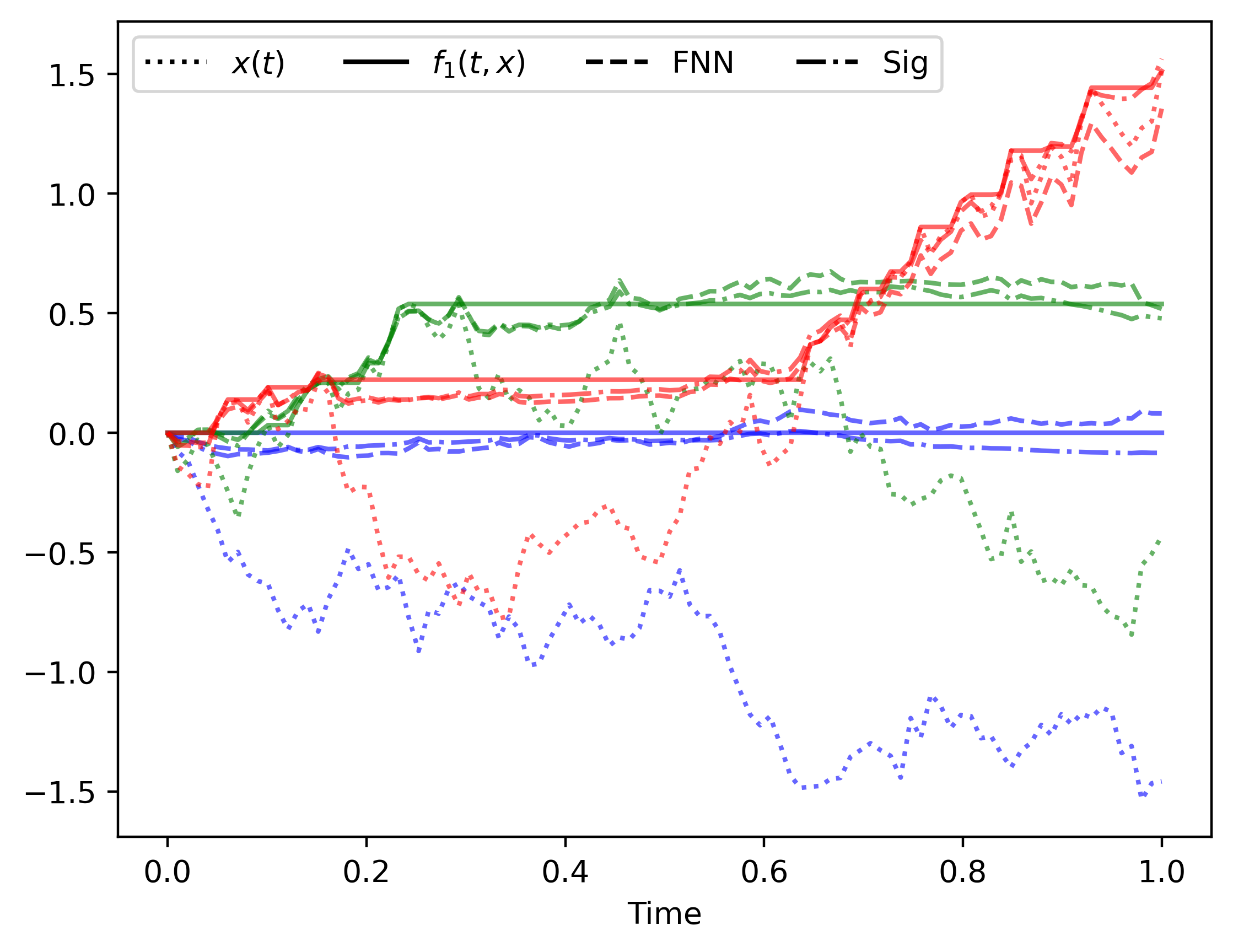}
		    \subcaption{Three samples of the test set}
		\end{minipage}
		\caption{Learning $f_1$ in \eqref{EqEx1} by a non-anticipative FNN $\varphi \in \mathcal{FN}^{\rho,\widetilde{\rho}}_{\Lambda^\alpha_T}$ (label FNN) and a linear function of the signature $\sum_{0 \leq \vert I \vert \leq N_{Sig}} a_I \langle e_I, \widehat{\mathbb{X}}_t \rangle$ (label Sig). In (a), the weighted mean squared errors \eqref{EqDefLossFNN}+\eqref{EqDefLossSig} are evaluated on the training set in each epoch (continuous line) as well as on the test after every 200-th epoch (dots). In (b), three samples of the test set are shown together with $f_1$ and its approximation.}
		\label{FigEx1}
		
		\bigskip
		\vspace{0.5cm}
		
		\begin{minipage}[t]{0.49\textwidth}
		    \centering
		    \includegraphics[height=5.35cm]{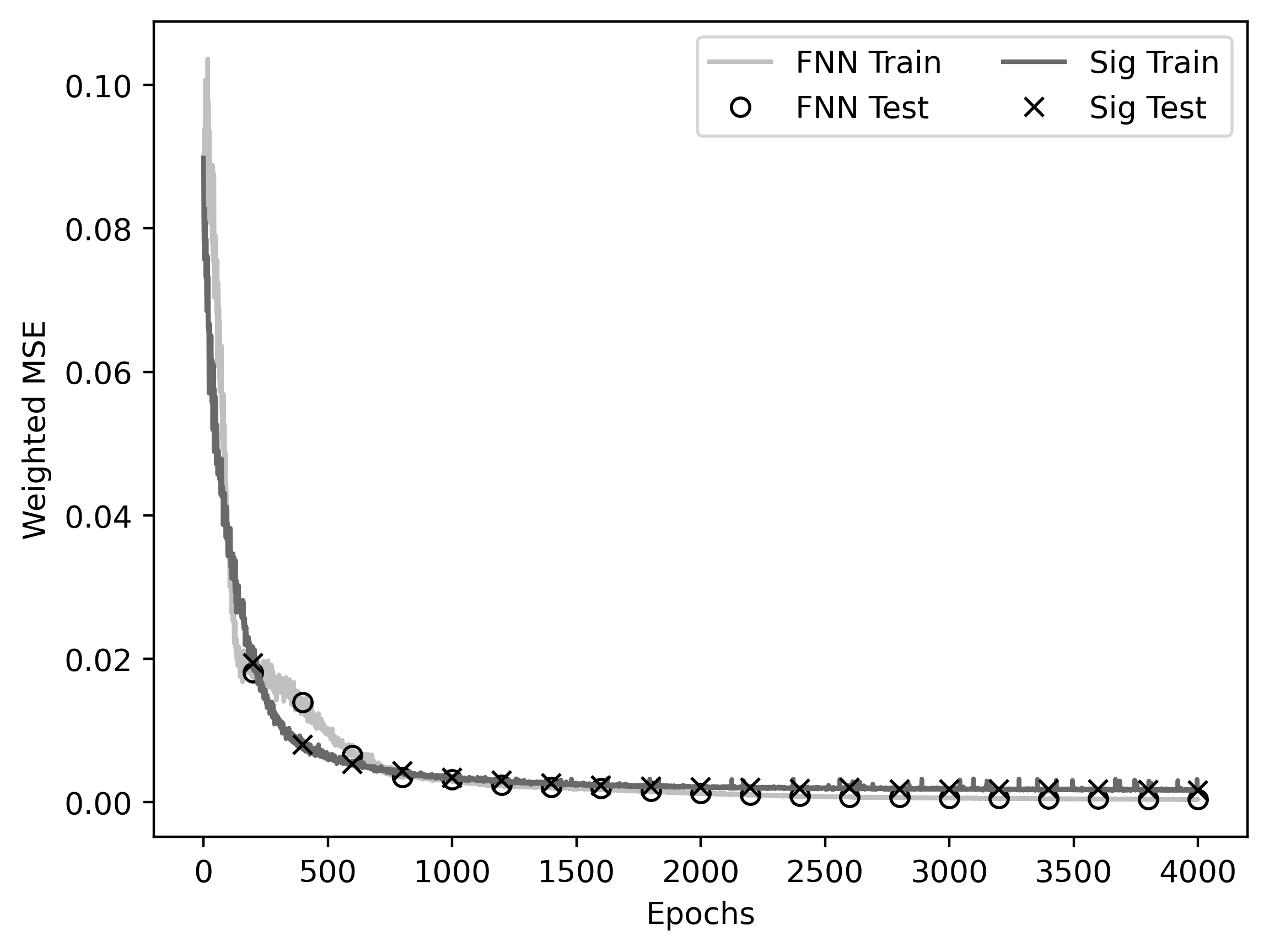}
		    \subcaption{Learning performance}
		\end{minipage}
		\begin{minipage}[t]{0.49\textwidth}
		    \centering
		    \includegraphics[height=5.35cm]{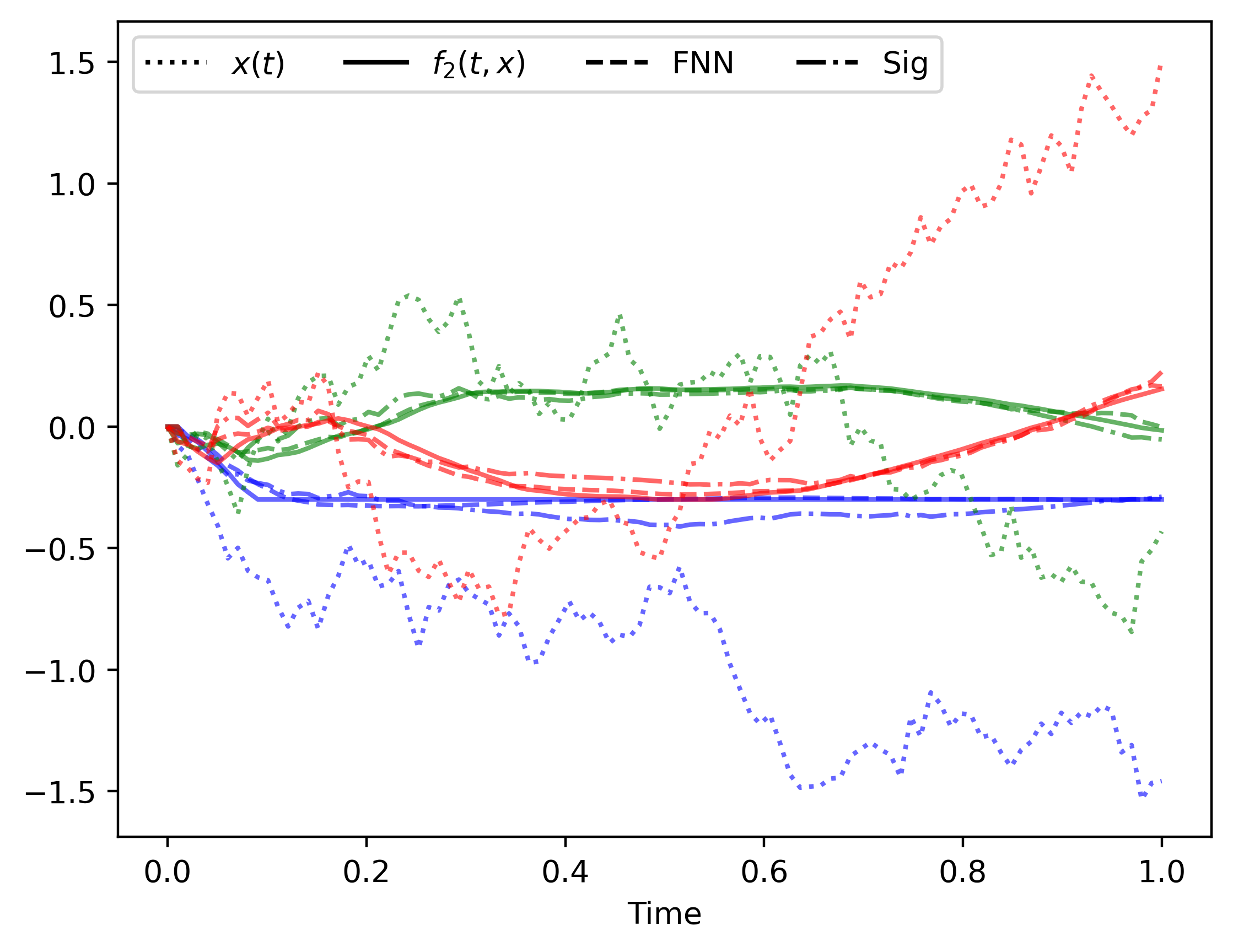}
		    \subcaption{Three samples of the test set}
		\end{minipage}
		\caption{Learning $f_2$ in \eqref{EqEx2} by a non-anticipative FNN $\varphi \in \mathcal{FN}^{\rho,\widetilde{\rho}}_{\Lambda^\alpha_T}$ (label FNN) and a linear function of the signature $\sum_{0 \leq \vert I \vert \leq N_{Sig}} a_I \langle e_I, \widehat{\mathbb{X}}_t \rangle$ (label Sig). In (a), the weighted mean squared errors \eqref{EqDefLossFNN}+\eqref{EqDefLossSig} are evaluated on the training set in each epoch (continuous line) as well as on the test after every 200-th epoch (dots). In (b), three samples of the test set are shown together with $f_2$ and its approximation.}
		\label{FigEx2}
	\end{figure}
	
	\appendix

    \section{Preduals of certain Banach spaces}
    \label{AppBanachPredual}

    In the following, we apply the result of \cite{kaijser77} to characterize a Banach space $(X,\Vert \cdot \Vert_X)$ as dual Banach space, which formally generalizes the Dixmier-Ng theorem (see \cite{dixmier48,ng71}).

    \begin{theorem}[{\cite[Theorem~1]{kaijser77}}]
        \label{ThmBanachPredual}
        Let $(X,\Vert \cdot \Vert_X)$ be a Banach space, and let $E_0 \subseteq X^*$ be a set of continuous linear functionals such that
        \begin{enumerate}
            \item $E_0$ is point separating on $X$, and
            \item the closed unit ball $\overline{B_1(0)} := \left\lbrace x \in X: \Vert x \Vert_X \leq 1 \right\rbrace$ is compact with respect to the weak topology on $X$ induced by $E_0 \subseteq X^*$, i.e.~the weakest topology such that every $e \in E_0$ is continuous from $X$ to $\mathbb{R}$.
        \end{enumerate}
        Then, $X$ is a dual Banach space with predual given by the closure of $\linspan(E_0)$ with respect to $\Vert \cdot \Vert_{X^*}$.
    \end{theorem}

    \begin{remark}
        \label{RemWeakStar}
        Let $E_0$ satisfy the assumptions of Theorem~\ref{ThmBanachPredual} and denote by $E$ the closure of $\linspan(E_0)$ with respect to $\Vert \cdot \Vert_{X^*}$, which is a closed vector subspace of $X^*$. Then, by the Hahn-Banach theorem, the embedding $\iota: E \hookrightarrow X^*$ has an adjoint $\iota^*: X^{**} \rightarrow E^*$, which concatenated with the canonical embedding $\mathcal{I}: X \hookrightarrow X^{**}$ yields an isometric isomorphism $\Phi := \iota^* \circ \mathcal{I}: X \rightarrow E^*$ satisfying $\Phi(x)(e) = e(x)$ for all $x \in X$ and $e \in E_0$. Hence, the dual pairing $X \times E \ni (x,e) \mapsto \langle x, e \rangle_{X \times E} := \Phi(x)(e) \in \mathbb{R}$ is continuous. Thus, a weak-$*$-topology on $X$ is generated by sets of the form $\left\lbrace x \in X: \langle x, e \rangle_{X \times E} \in U \right\rbrace$, for $e \in E$ and $U \subseteq \mathbb{R}$ open. This weak-$*$-topology is weaker than the weak topology induced by $X^*$ (except if $X$ is reflexive), which in turn is weaker than the norm topology induced by $\Vert \cdot \Vert_X$ (except if $X$ is finite dimensional).
    \end{remark}

    Now, we apply Theorem~\ref{ThmBanachPredual} to show the existence of a predual for the Banach spaces presented in Example~\ref{ExWeightedSpaces}, i.e.~H\"older spaces and spaces of finite $p$-variation.

    \subsection{Predual for $C^\alpha(S;Z)$}
    \label{AppHoelderPredual}

    In this section, we show that H\"older spaces introduced in Section~\ref{SecNotation} are dual Banach spaces. To this end, we fix some $\alpha > 0$, a compact metric space $(S,d_S)$ and a dual Banach space $(Z,\Vert \cdot \Vert_Z)$. To this end, we first show the following relations between the $w^*$-uniform topology $\tau_\infty$ and the $C^{\alpha'}$-topologies $\tau_{\alpha'}$, $0 \leq \alpha' \leq \alpha$, which were defined in Section~\ref{SecNotation}.

    \begin{lemma}
        \label{LemmaHoelderTop}
        Let $0 \leq \alpha' < \alpha$. Then, the following holds true:
        \begin{enumerate}
            \item\label{LemmaHoelderTop1} $\tau_\infty = \tau_0 \subseteq \tau_{\alpha'} \subseteq \tau_{\alpha}$ on $C^\alpha(S;Z)$.
            \item\label{LemmaHoelderTop2} $\tau_\infty = \tau_0 = \tau_{\alpha'}$ on $\Vert \cdot \Vert_\alpha$-bounded subsets of $C^\alpha(S;Z)$.
        \end{enumerate}
    \end{lemma}
    \begin{proof}
        Let $(V,\Vert \cdot \Vert_V)$ be a predual for $(Z,\Vert \cdot \Vert_Z)$. Now, we first show that $\tau_\infty = \tau_0$ on $C^\alpha(S;Z)$. To this end, we observe for every $x \in C^\alpha(S;Z)$ and $v \in V$ that
        \begin{equation*}
            \begin{aligned}
                \Vert x \Vert_{\infty,v} & = \sup_{s \in S} \vert \langle x(s), v \rangle_{Z \times V} \vert \leq \vert \langle x(0), v \rangle_{Z \times V} \vert + \sup_{s \in S} \vert \langle x(s) - x(0), v \rangle_{Z \times V} \vert \\
                & \leq \vert \langle x(0), v \rangle_{Z \times V} \vert + \sup_{s,t \in S} \vert \langle x(s) - x(t), v \rangle_{Z \times V} \vert = \Vert x \Vert_{0,v}
            \end{aligned}
        \end{equation*}
        and that
        \begin{equation*}
            \Vert x \Vert_{0,v} = \vert \langle x(0), v \rangle_{Z \times V} \vert + \sup_{s,t \in S} \vert \langle x(s) - x(t), v \rangle_{Z \times V} \vert \leq 3 \sup_{s \in S} \vert \langle x(s), v \rangle_{Z \times V} \vert = 3 \Vert x \Vert_{\infty,v}.
        \end{equation*}
        Since the seminorms $\Vert \cdot \Vert_{\infty,v}$ and $\Vert \cdot \Vert_{0,v}$, $v \in V$, generate the topologies $\tau_\infty$ and $\tau_0$, respectively, we obtain $\tau_\infty = \tau_0$ on $C^\alpha(S;Z)$.

        Next, we show that $\tau_{\alpha'} \subseteq \tau_{\alpha}$ on $C^\alpha(S;Z)$. To this end, we use the constant $C_S := \sup_{s,t \in S} d_S(s,t) < \infty$ to conclude for every $x \in C^\alpha(S;Z)$ and $v \in V$ that
        \begin{equation*}
            \begin{aligned}
                \Vert x \Vert_{\alpha',v} & = \vert \langle x(0), v \rangle_{Z \times V} \vert + \sup_{s,t \in S} \frac{\vert \langle x(s) - x(t), v \rangle_{Z \times V} \vert}{d_S(s,t)^{\alpha'}} \\
                & \leq \vert \langle x(0), v \rangle_{Z \times V} \vert + \left( \sup_{s,t \in S} d_S(s,t)^{\alpha-\alpha'} \right) \sup_{s,t \in S} \frac{\vert \langle x(s) - x(t), v \rangle_{Z \times V} \vert}{d_S(s,t)^\alpha} \\
                & \leq \left( 1 + C_S^{\alpha-\alpha'} \right) \Vert x \Vert_{\alpha,v}.
            \end{aligned}
        \end{equation*}
        Since the seminorms $\Vert \cdot \Vert_{\alpha',v}$ and $\Vert \cdot \Vert_{\alpha,v}$, $v \in V$, generate the topologies $\tau_{\alpha'}$ and $\tau_\alpha$, respectively, we obtain $\tau_{\alpha'} \subseteq \tau_\alpha$ on $C^\alpha(S;Z)$.

        Finally, we show that $\tau_0 = \tau_{\alpha'}$ on every fixed $\Vert \cdot \Vert_\alpha$-bounded subset $B \subseteq C^\alpha(S;Z)$. To this end, we use the constant $M := \sup_{x \in B} \Vert x \Vert_\alpha < \infty$ and that $\sup_{s,t \in S} \frac{\vert \langle x(s) - x(t), v \rangle_{Z \times V} \vert}{d_S(s,t)^\alpha} \leq \vert x \vert_\alpha \Vert v \Vert_V \leq \Vert x \Vert_\alpha \Vert v \Vert_V$ to conclude for every $x \in B$ and $v \in V$ that
        \begin{equation*}
            \begin{aligned}
                \Vert x \Vert_{\alpha',v} & = \vert \langle x(0), v \rangle_{Z \times V} \vert + \sup_{s,t \in S} \frac{\vert \langle x(s) - x(t), v \rangle_{Z \times V} \vert}{d_S(s,t)^{\alpha'}} \\
                & \leq \vert \langle x(0), v \rangle_{Z \times V} \vert + \left( \sup_{s,t \in S} \frac{\vert \langle x(s) - x(t), v \rangle_{Z \times V} \vert}{d_S(s,t)^\alpha} \right)^\frac{\alpha'}{\alpha} \left( \sup_{s,t \in S} \vert \langle x(s) - x(t), v \rangle_{Z \times V} \vert \right)^{1-\frac{\alpha'}{\alpha}} \\
                & \leq \vert \langle x(0), v \rangle_{Z \times V} \vert + (M \Vert v \Vert_V)^\frac{\alpha'}{\alpha} \vert x \vert_{0,v}^{1-\frac{\alpha'}{\alpha}} \\
                & \leq \left( 1 + (M \Vert v \Vert_V)^\frac{\alpha'}{\alpha} \right) \max\left( \Vert x \Vert_{0,v}, \Vert x \Vert_{0,v}^{1-\frac{\alpha'}{\alpha}} \right)
            \end{aligned}
        \end{equation*}
        Since the seminorms $\Vert \cdot \Vert_{\alpha',v}$ and $\Vert \cdot \Vert_{0,v}$, $v \in V$, generate the topologies $\tau_{\alpha'}$ and $\tau_0$, respectively, we obtain $\tau_{\alpha'} \subseteq \tau_0$ on $B$. This together with \ref{LemmaHoelderTop1} shows that $\tau_0 = \tau_{\alpha'}$ on the $\Vert \cdot \Vert_\alpha$-bounded subset $B \subseteq C^\alpha(S;Z)$.
    \end{proof}

    In the following, we denote by $B_R(0) := \lbrace x \in C^\alpha(S;Z): \Vert x \Vert_\alpha < 1 \rbrace$ and $\overline{B_R(0)} := \lbrace x \in C^\alpha(S;Z): \Vert x \Vert_\alpha \leq 1 \rbrace$ the open and closed ball of radius $R > 0$, respectively.
    
    Then, for every $0 \leq \alpha' < \alpha$, we show that the identity $T: (C^\alpha(S;Z),\Vert \cdot \Vert_\alpha) \hookrightarrow (C^{\alpha'}(S;Z),\tau_{\alpha'})$ is a \emph{compact embedding}, i.e.~there exists a $0$-neighborhood of $(C^\alpha(S;Z),\Vert \cdot \Vert_\alpha)$ such that $T(U)$ is relatively compact with respect to $\tau_{\alpha'}$ (see, e.g., \cite[p.~98]{schaefer99}). In our setting with a normed vector space as domain, this is equivalent to the condition that for every $\Vert \cdot \Vert_\alpha$-bounded subset $B \subseteq C^\alpha(S;Z)$ the image $T(U)$ is relatively compact with respect to $\tau_{\alpha'}$. Indeed, the latter implies that $T(B_1(0))$ is relatively compact with respect to $\tau_{\alpha'}$. Conversely, if there exists a $0$-neighborhood $U$ of $(C^\alpha(S;Z),\Vert \cdot \Vert_\alpha)$ with $B_r(0) \subseteq U$, for some $r > 0$, then for every $\Vert \cdot \Vert_\alpha$-bounded subset $B \subseteq B_R(0) \subset C^\alpha(S;Z)$, for some $R > 0$, it holds that $T(B) \subseteq T(B_R(0)) \subseteq \frac{R}{r} T(B_r(0)) \subseteq \frac{R}{r} T(U)$, where $\frac{R}{r} T(U)$ is relatively compact with respect to $\tau_{\alpha'}$.
    
    For $Z = \mathbb{R}$, the following result can be found in \cite[Satz 2.42]{dobrowolski2010angewandte}, while for general dual Banach space $(Z,\Vert \cdot \Vert_Z)$ we provide a proof for completeness.

    \begin{theorem}
        \label{ThmHoelderEmbedding}
        For $0 \leq \alpha' < \alpha$, the embedding $(C^\alpha(S;Z),\Vert \cdot \Vert_\alpha) \hookrightarrow (C^{\alpha'}(S;Z),\tau_{\alpha'})$ is compact.
    \end{theorem}
    \begin{proof}
        Fix some $0 \leq \alpha' < \alpha$ and let $(V,\Vert \cdot \Vert_V)$ be a predual for $(Z,\Vert \cdot \Vert_Z)$. Then, by using the constant $C_S := \sup_{s,t \in S} d_S(s,t) < \infty$, it holds for every $x \in C^\alpha(S;Z)$ and $v \in V$ that
        \begin{equation*}
            \begin{aligned}
                \Vert x \Vert_{\alpha',v} & = \vert \langle x(0), v \rangle_{Z \times V} \vert + \sup_{s,t \in S \atop s \neq t} \frac{\vert \langle x(s) - x(t), v \rangle_{Z \times V} \vert}{d_S(s,t)^{\alpha'}} \\
                & \leq \vert \langle x(0), v \rangle_{Z \times V} \vert + (2C_0)^{\alpha-\alpha'} \sup_{s,t \in S \atop s \neq t} \frac{\vert \langle x(s) - x(t), v \rangle_{Z \times V} \vert}{d_S(s,t)^\alpha} \\
                & \leq \left( 1 + C_S^{\alpha-\alpha'} \right) \Vert x \Vert_{\alpha,v} \leq \left( 1 + C_s^{\alpha-\alpha'} \right) \Vert v \Vert_V \Vert x \Vert_\alpha.
            \end{aligned}
        \end{equation*}
        Since the seminorms $\Vert \cdot \Vert_{\alpha',v}$, $v \in V$, generate the topology $\tau_{\alpha'}$, we can apply \cite[Theorem~III.1.1]{schaefer99} to conclude that $(C^\alpha(S;Z),\Vert \cdot \Vert_\alpha) \hookrightarrow (C^{\alpha'}(S;Z),\tau_{\alpha'})$ is continuous.

        Now, we fix an $\Vert \cdot \Vert_\alpha$-bounded subset $B \subseteq \overline{B_R(0)} \subseteq C^\alpha(S;Z)$, for some $R > 0$. First, we show that $\overline{B_R(0)}$ is pointwise closed, i.e.~closed in the topology of pointwise convergence, where $Z$ is equipped with the weak-$*$-topology. Let $(x_\gamma)_\gamma \subseteq \overline{B_R(0)}$ be a net converging pointwise to some $x \in C^\alpha(S;Z)$. Then, for arbitrary $s,t \in S$, there exists some $v_0,v_1 \in V$ with $\Vert v_i \Vert_V \leq 1$ such that $\Vert x(0) \Vert = \langle x(0), v_0 \rangle_{Z \times V}$ and $\Vert x(s) - x(t) \Vert_Z = \langle x(s) - x(t), v_1 \rangle_{Z \times V}$, which implies that
        \begin{equation}
            \label{EqThmHoelderEmbeddingProof1}
            \begin{aligned}
                \Vert x(0) \Vert_Z + \frac{\Vert x(s) - x(t) \Vert_Z}{d_S(s,t)^\alpha} & = \langle x(0), v_0 \rangle_{Z \times V} + \frac{\langle x(s) - x(t), v_1 \rangle_{Z \times V}}{d_S(s,t)^\alpha} \\
                & = \lim_\gamma \left( \langle x_\gamma(0), v_0 \rangle_{Z \times V} + \frac{\langle x_\gamma(s) - x_\gamma(t), v_1 \rangle_{Z \times V}}{d_S(s,t)^\alpha} \right) \\
                & \leq \lim_\gamma \left( \Vert x_\gamma(0) \Vert_Z + \frac{\Vert x_\gamma(s) - x_\gamma(t) \Vert_Z}{d_S(s,t)^\alpha} \right) \\
                & \leq \lim_\gamma \left( \Vert x_\gamma(0) \Vert_Z + \vert x_\gamma \vert_\alpha \right) = \lim_\gamma \Vert x_\gamma \Vert_\alpha \leq R.
            \end{aligned}
        \end{equation}
        By taking the supremum over $s,t \in S$ with $s \neq t$, we have $\Vert x \Vert_\alpha \leq R$, showing that $\overline{B_R(0)}$ is pointwise closed. Moreover, $\overline{B_R(0)}$ is equicontinuous and pointwise bounded, thus by the Banach-Alaoglu theorem pointwise relatively compact with respect to the weak-$*$-topology of $Z$. Hence, we can apply the Arzel\`a-Ascoli theorem in \cite[Theorem~43.15]{willard04} to conclude that $\overline{B_R(0)}$ is relatively compact with respect to $\tau_\infty$. Thus, by using that the identity $(C^\alpha(S;Z),\tau_\infty) \hookrightarrow (C^\alpha(S;Z),\tau_{\alpha'})$ is continuous (see Lemma~\ref{LemmaHoelderTop}~\ref{LemmaHoelderTop1}), it follows that $B \subseteq \overline{B_R(0)}$ is relatively compact with respect to $\tau_{\alpha'}$.
    \end{proof}
    
    \begin{theorem}
        \label{ThmHoelderPredual}
        For $\alpha > 0$, the Banach space $(C^\alpha(S;Z),\Vert \cdot \Vert_\alpha)$ is a dual Banach space. Moreover, $\tau_{w^*} = \tau_\infty = \tau_{\alpha'}$ on every $\Vert \cdot \Vert_\alpha$-bounded subset of $C^\alpha(S;Z)$, for all $0 \leq \alpha' < \alpha$.
    \end{theorem}
    \begin{proof}
        Let $(V,\Vert \cdot \Vert_V)$ be a predual for $(Z,\Vert \cdot \Vert_Z)$. We want to apply Theorem~\ref{ThmBanachPredual} with
        \begin{equation}
            \label{EqThmHoelderPredualProof1}
            E_0 := \left\lbrace C^\alpha(S;Z) \ni x \mapsto \langle x(s), v \rangle_{Z \times V} \in \mathbb{R}: s \in S, \, v \in V \right\rbrace \subseteq C^\alpha(S;Z)^*.
        \end{equation} 
        In order to show that $E_0$ is point separating on $C^\alpha(S;Z)$, let $x_1,x_2 \in C^\alpha(S;Z)$ be distinct. Then, there exists some $s \in S$ such that $Z \ni x_1(s) \neq x_2(s) \in Z$. Moreover, by using that $V^* \cong Z$, there exists some $v \in V$ such that $\mathbb{R} \ni \langle x_1(s), v \rangle_{Z \times V} \neq \langle x_2(s), v \rangle_{Z \times V} \in \mathbb{R}$. Thus, for $\left( x \mapsto e(x) := \langle x(s), v \rangle_{Z \times V} \right) \in E_0$, we have $e(x_1) = \langle x_1(s), v \rangle_{Z \times V} \neq \langle x_2(s), v \rangle_{Z \times V} = e(x_2)$, which shows that $E_0$ is point separating on $C^\alpha(S;Z)$.

        Moreover, $\overline{B_1(0)}$ is by Theorem~\ref{ThmHoelderEmbedding} relatively compact with respect to $\tau_0$, which is by Lemma~\ref{LemmaHoelderTop}~\ref{LemmaHoelderTop1} identical to $\tau_\infty$. Since the weak topology on $C^\alpha(S;Z)$ induced by $E_0 \subseteq C^\alpha(S;Z)^*$ is weaker than $\tau_\infty$, we observe that $\overline{B_1(0)}$ is relatively compact with respect to this weak topology. In addition, by using a similar argument as in \eqref{EqThmHoelderEmbeddingProof1}, it follows that $\overline{B_1(0)}$ is closed with respect to the weak topology on $C^\alpha(S;Z)$ induced by $E_0 \subseteq C^\alpha(S;Z)^*$ and thus compact with respect to this weak topology. Hence, we can apply Theorem~\ref{ThmBanachPredual} to conclude that the closure of $\linspan(E_0)$ with respect to $\Vert \cdot \Vert_{C^\alpha(S;Z)^*}$, denoted by $E$, is a predual for $C^\alpha(S;Z)$.

        In order to prove that the topologies are identical, we fix a $\Vert \cdot \Vert_\alpha$-bounded subset $B \subseteq \overline{B_R(0)} \subseteq C^\alpha(S;Z)$, for some $R > 0$. Moreover, we observe that $\tau_{w^*}$ is equivalent to the topology of pointwise convergence (when $Z$ is equipped with the weak-$*$-topology), which is weaker than the $w^*$-uniform topology $\tau_\infty$. Hence, $\tau_{w^*} \subseteq \tau_\infty \subseteq \tau_{\alpha'}$ implies that the identity $\id: (\overline{B_R(0)},\tau_{\alpha'}) \rightarrow (\overline{B_R(0)},\tau_{w^*})$ is continuous. Furthermore, since $(\overline{B_R(0)}, \tau_{w^*})$ is Hausdorff and $\overline{B_R(0)}$ is relatively compact with respect to $\tau_{\alpha'}$ (see Theorem~\ref{ThmHoelderEmbedding}) as well as closed with respect to $\tau_{\alpha'}$ (by the same argument as in \eqref{EqThmHoelderEmbeddingProof1}), implying that $(\overline{B_R(0)},\tau_{\alpha'})$ is compact, the inverse $\id^{-1}: (\overline{B_R(0)},\tau_{w^*}) \rightarrow (\overline{B_R(0)},\tau_{\alpha'})$ is by \cite[Theorem~26.6]{munkres14} also continuous, which shows that $\tau_{\alpha'} \subseteq \tau_{w^*}$ on $B \subseteq \overline{B_R(0)}$. Together with Lemma~\ref{LemmaHoelderTop}~\ref{LemmaHoelderTop2}, we therefore conclude that $\tau_{w^*} = \tau_\infty = \tau_{\alpha'}$ on $B$, for all $0 \leq \alpha' < \alpha$.
    \end{proof}

    The point evaluations in \eqref{EqThmHoelderPredualProof1} are also used in \cite{weaver99,godefroy03} to construct the Arens-Eells space $\AE(S)$ and the Lipschitz-free space $\mathcal{F}(S)$, respectively, which are both used as preduals of the space of globally Lipschitz continuous functions. 
    
    Choosing the weight function $C^\alpha(S;Z) \ni x \mapsto \psi(x) := \eta\left( \Vert x \Vert_\alpha \right) \in (0,\infty)$, with a continuous non-decreasing unbounded $\eta: [0,\infty) \rightarrow (0,\infty)$, we consider the weighted space $(C^\alpha(S;Z),\psi)$ either equipped with the weak-$*$-topology $\tau_{w^*}$ (see Example~\ref{ExWeightedSpaces}~\ref{ExWeightedSpaceHoelderPredual}), with the $w^*$-uniform topology $\tau_\infty$, or with any $C^{\alpha'}$-topology $\tau_{\alpha'}$ (see Example~\ref{ExWeightedSpaces}~\ref{ExWeightedSpaceHoelderWeaker}). Then, by using that these topologies coincide on the closed $\alpha$-H\"older balls $K_R = \psi^{-1}((0,R])$, we conclude that $\mathcal{B}_\psi(C^\alpha(S;Z))$ is the same for any choice of these topologies on $C^\alpha(S;Z)$.

    \begin{proposition}
        \label{PropHoelderBPsi}
        Let $C^\alpha_\tau(S;Z)$ denote $C^\alpha(S;Z)$ equipped with the topology $\tau \in \mathcal{T} := \lbrace \tau_{w^*}, \tau_\infty \rbrace \cup \left\lbrace \tau_{\alpha'}: \alpha' \in [0,\alpha) \right\rbrace$. Moreover, for some continuous non-decreasing unbounded $\eta: [0,\infty) \rightarrow (0,\infty)$, let $C^\alpha_\tau(S;Z) \ni x \mapsto \psi(x) := \eta\left( \Vert x \Vert_\alpha \right) \in (0,\infty)$ be an admissible weight function, for all $\tau \in \mathcal{T}$. Then, $\mathcal{B}_\psi(C^\alpha_\tau(S;Z)) = \mathcal{B}_\psi(C^\alpha_{\tau'}(S;Z))$, for all $\tau,\tau' \in \mathcal{T}$.
    \end{proposition}
    \begin{proof}
        We fix some $f: C^\alpha(S;Z) \rightarrow \mathbb{R}$ and denote for every $R > 0$ the pre-image of $\psi: C^\alpha(S;Z) \rightarrow (0,\infty)$ by $K_R := \psi^{-1}((0,R]) = \left\lbrace x \in C^\alpha(S;Z): \Vert x \Vert_\alpha \leq \eta^{-1}(R) \right\rbrace$. Moreover, note that it suffices to assume that the weight function $\psi: C^\alpha_\tau(S;Z) \rightarrow (0,\infty)$ is admissible on $C^\alpha_\tau(S;Z)$, for one topology $\tau \in \mathcal{T}$, since the topologies $\tau \in \mathcal{T}$ are by Theorem~\ref{ThmHoelderPredual} identical on the $\Vert \cdot \Vert_\alpha$-bounded pre-images $K_R$. In addition, by combining Lemma~\ref{LemmaBpsiEquivChar}~\ref{LemmaBpsiEquivChar1}+\ref{LemmaBpsiEquivChar2}, it follows for every $\tau \in \mathcal{T}$ that
        \begin{equation*}
            f \in \mathcal{B}_\psi(C^\alpha_\tau(S;Z)) \quad \Longleftrightarrow \quad
            \begin{cases}
                f\vert_{K_R} \text{ is continuous with respect to $\tau$ on $K_R$,} \\
                \text{for all $R > 0$, and } \lim_{R \rightarrow \infty} \sup_{x \in C^\alpha(S;Z) \setminus K_R} \frac{\vert f(x) \vert}{\psi(x)} = 0.
            \end{cases}
        \end{equation*}
        Hence, by combining this with the fact that all topologies $\tau \in \mathcal{T}$ are identical on $K_R$, we have $f \in \mathcal{B}_\psi(C^\alpha_\tau(S;Z))$ if and only if $f \in \mathcal{B}_\psi(C^\alpha_{\tau'}(S;Z))$, for all $\tau,\tau' \in \mathcal{T}$.
    \end{proof}

    \begin{remark}
        \label{RemHoelder0}
        Let us summarize the consequences of this section for the closed vector subspace $C^\alpha_0(S;Z) \subseteq C^\alpha(S;Z)$ of $\alpha$-H\"older continuous functions preserving the origin.
        \begin{enumerate}
            \item\label{RemHoelder0Embedding} Theorem \ref{ThmHoelderEmbedding} implies that the embedding $(C^\alpha_0(S;Z),\Vert \cdot \Vert_\alpha) \hookrightarrow (C^{\alpha'}_0(S;Z),\tau_{\alpha'})$ is compact, for all $0 \leq \alpha' < \alpha$.
            \item\label{RemHoelder0Predual} By using the point evaluations $C^\alpha_0(S;Z) \ni x \mapsto \langle x(s), v \rangle_{Z \times V} \in \mathbb{R}$, for $s \in S$ and $v \in V$ (where $(V,\Vert \cdot \Vert_V)$ is a predual for $(Z,\Vert \cdot \Vert_Z)$), we conclude similarly as in Theorem~\ref{ThmHoelderPredual} that $C^\alpha_0(S;Z)$ is a dual Banach space. Moreover, on every $\Vert \cdot \Vert_\alpha$-bounded subset, the weak-$*$-topology $\tau_{w^*}$, the $w^*$-uniform topology $\tau_\infty$, and the $C^{\alpha'}$-topology $\tau_{\alpha'}$, $0 \leq \alpha' < \alpha$, are identical.
            \item\label{RemHoelder0BPsi} By Proposition~\ref{PropHoelderBPsi}, the function space $\mathcal{B}_\psi(C^\alpha_0(S;Z))$ with admissible weight function $C^\alpha_0(S;Z) \ni x \mapsto \psi(x) := \eta\left( \Vert x \Vert_\alpha \right) \in (0,\infty)$ does not depend on the choice of these topology on the underlying space $C^\alpha_0(S;Z)$.
        \end{enumerate}
    \end{remark}

    \subsection{Predual for $C^{p-var,\alpha}([0,T];Z)$}
    \label{AppPVarPredual}

    In this section, we show for some fixed $T > 0$, $(p,\alpha) \in [1,\infty) \times (0,1)$ with $p \alpha < 1$, and a dual Banach space $Z$ that $C^{p-var,\alpha}([0,T];Z)$ introduced in Section~\ref{SecNotation} is also a dual Banach space. To this end, we first show the following relations between the $w^*$-uniform topology $\tau_\infty$, the $C^{p'-var}$-topologies $\tau_{p'-var}$, $p' \in [1,p]$, and the $C^{p'-var,\alpha'}$-topologies $\tau_{p'-var,\alpha'}$, $(p',\alpha') \in [1,p] \times [0,\alpha]$ with $p' \alpha' < 1$, which were defined in Section~\ref{SecNotation}.

    \begin{lemma}
        \label{LemmaPVarTop}
        Let $(p,\alpha) \in [1,\infty) \times (0,1)$ with $p \alpha < 1$ and $(p',\alpha') \in (p,\infty] \times [0,\alpha)$ with $p' \alpha' < 1$. Then, the following holds true:
        \begin{enumerate}
            \item\label{LemmaPVarTop1} $\tau_\infty \subseteq \tau_{p'-var} = \tau_{p'-var,0} \subseteq \tau_{p'-var,\alpha'} \subseteq \tau_{p-var,\alpha}$ on $C^{p-var,\alpha}([0,T];Z)$.
            \item\label{LemmaPvarTop2} $\tau_\infty = \tau_{p'-var} = \tau_{p'-var,0} = \tau_{p'-var,\alpha'}$ on $\Vert \cdot \Vert_{p-var,\alpha}$-bounded subsets of $C^{p-var,\alpha}([0,T];Z)$.
        \end{enumerate}
    \end{lemma}
    \begin{proof}
        Let $(V,\Vert \cdot \Vert_V)$ be a predual for $(Z,\Vert \cdot \Vert_Z)$. First, we show that $\tau_\infty \subseteq \tau_{p'-var}$ on $C^{p-var,\alpha}([0,T];Z)$. To this end, we observe for every $x \in C^{p-var,\alpha}([0,T];Z)$ and $v \in V$ that
        \begin{equation*}
            \begin{aligned}
                \Vert x \Vert_{\infty,v} & = \sup_{s \in [0,T]} \vert \langle x(s), v \rangle_{Z \times V} \vert \\
                & \leq \vert \langle x(0), v \rangle_{Z \times V} \vert + \sup_{s,t \in S} \vert \langle x(s) - x(0), v \rangle_{Z \times V} \vert \\
                & \leq \vert \langle x(0), v \rangle_{Z \times V} \vert + \left( \sup_{(t_i)_i \in \mathcal{D}([0,T])} \sum_i \vert \langle x(t_{i+1}) - x(t_i), v \rangle_{Z \times V} \vert^{p'} \right)^\frac{1}{p'} \\
                & = \Vert x \Vert_{p'-var,v}
            \end{aligned}
        \end{equation*}
        Since the seminorms $\Vert \cdot \Vert_{\infty,v}$ and $\Vert \cdot \Vert_{p'-var,v}$, $v \in V$, generate the topologies $\tau_\infty$ and $\tau_{p'-var}$, respectively, we obtain $\tau_\infty \subseteq \tau_{p'-var}$ on $C^{p-var,\alpha}([0,T];Z)$.
        
        Next, we show that $\tau_{p'-var} = \tau_{p'-var,0}$ on $C^{p-var,\alpha}([0,T];Z)$. To this end, we observe for every $x \in C^{p-var,\alpha}([0,T];Z)$ and $v \in V$ that
        \begin{equation*}
            \Vert x \Vert_{p'-var,v} = \vert \langle x(0), v \rangle_{Z \times V} \vert + \Vert x \Vert_{p',v} \leq \vert \langle x(0), v \rangle_{Z \times V} \vert + \Vert x \Vert_{p',v} + \Vert x \Vert_{0,v} = \Vert x \Vert_{p'-var,0,v}
        \end{equation*}
        and that
        \begin{equation*}
            \begin{aligned}
                \Vert x \Vert_{p'-var,0,v} & = \vert \langle x(0), v \rangle_{Z \times V} \vert + \left( \sup_{(t_i)_i \in \mathcal{D}([0,T])} \sum_i \vert \langle x(t_{i+1}) - x(t_i), v \rangle_{Z \times V} \vert^{p'} \right)^\frac{1}{p'} \\
                & \quad\quad + \sup_{s,t \in [0,T]} \vert \langle x(s) - x(t), v \rangle_{Z \times V} \vert \\
                & \leq \vert \langle x(0), v \rangle_{Z \times V} \vert + 2 \left( \sup_{(t_i)_i \in \mathcal{D}([0,T])} \sum_i \vert \langle x(t_{i+1}) - x(t_i), v \rangle_{Z \times V} \vert^{p'} \right)^\frac{1}{p'} \\
                & \leq 2 \Vert x \Vert_{p'-var,v}.
            \end{aligned}
        \end{equation*}
        Since the seminorms $\Vert \cdot \Vert_{p'-var,v}$ and $\Vert \cdot \Vert_{p'-var,0,v}$, $v \in V$, generate the topologies $\tau_{p'-var}$ and $\tau_{p'-var,0}$, respectively, we obtain $\tau_{p'-var} = \tau_{p'-var,0}$ on $C^{p-var,\alpha}([0,T];Z)$.

        Now, we show that $\tau_{p'-var,\alpha'} \subseteq \tau_{p-var,\alpha}$ on $C^{p-var,\alpha}([0,T];Z)$. To this end, we use \cite[Proposition~5.3]{friz10} (still holding true in our setting with seminorms) to conclude for every $x \in C^{p-var,\alpha}([0,T];Z)$ and $v \in V$ that
        \begin{equation*}
            \begin{aligned}
                \Vert x \Vert_{p'-var,\alpha',v} & \leq \vert \langle x(0), v \rangle_{Z \times V} \vert + \sup_{s,t \in [0,T] \atop s \neq t} \frac{\vert \langle x(s) - x(t), v \rangle_{Z \times V} \vert}{\vert s-t \vert^{\alpha'}} + \vert x \vert_{p'-var,v} \\
                & \leq \vert \langle x(0), v \rangle_{Z \times V} \vert + T^{\alpha-\alpha'} \sup_{s,t \in [0,T] \atop s \neq t} \frac{\vert \langle x(s) - x(t), v \rangle_{Z \times V} \vert}{\vert s-t \vert^\alpha} + \vert x \vert_{p-var,v} \\
                & \leq \left( 1 + T^{\alpha-\alpha'} \right) \Vert x \Vert_{p-var,\alpha,v}
            \end{aligned}
        \end{equation*}
        Since the seminorms $\Vert \cdot \Vert_{p'-var,\alpha',v}$ and $\Vert \cdot \Vert_{p-var,\alpha,v}$, $v \in V$, generate the topologies $\tau_{p'-var,\alpha'}$ and $\tau_{p-var,\alpha}$, respectively, we obtain $\tau_{p'-var,\alpha'} \subseteq \tau_{p-var,\alpha}$ on $C^{p-var,\alpha}([0,T];Z)$.

        Finally, we show that $\tau_\infty = \tau_{p'-var,\alpha'}$ on every fixed $\Vert \cdot \Vert_{p-var,\alpha}$-bounded subset $B \subseteq C^{p-var,\alpha}([0,T];Z)$. To this end, we use the constant $M := \sup_{x \in B} \Vert x \Vert_{p-var,\alpha} < \infty$ and \cite[Proposition~5.5]{friz10} (still holding true in our setting with seminorms), that $\vert x \vert_{p-var,v} \leq \Vert v \Vert_V \vert x \vert_{p-var} \leq \Vert v \Vert_V \Vert x \Vert_{p-var,\alpha}$, and that $\vert x \vert_{\alpha,v} \leq \Vert v \Vert_V \vert x \vert_\alpha \leq \Vert v \Vert_V \Vert x \Vert_{p-var,\alpha}$ to conclude for every $x \in B$ and $v \in V$ that
        \begin{equation*}
            \begin{aligned}
                \Vert x \Vert_{p'-var,\alpha',v} & = \vert \langle x(0), v \rangle_{Z \times V} \vert + \vert x \vert_{p'-var,v} + \vert x \vert_{\alpha',v} \\
                & \leq \vert \langle x(0), v \rangle_{Z \times V} \vert + \vert x \vert_{p-var,v}^\frac{p}{p'} \vert x \vert_{0,v}^{1-\frac{p}{p'}} + \vert x \vert_{\alpha,v}^\frac{\alpha'}{\alpha} \vert x \vert_{0,v}^{1-\frac{\alpha'}{\alpha}} \\
                & \leq \vert \langle x(0), v \rangle_{Z \times V} \vert + (M \Vert v \Vert_V)^\frac{p}{p'} \vert x \vert_{0,v}^{1-\frac{p}{p'}} + (M \Vert v \Vert_V)^\frac{\alpha'}{\alpha} \vert x \vert_{0,v}^{1-\frac{\alpha'}{\alpha}} \\
                & \leq \left( 1 + (M \Vert v \Vert_V)^\frac{p}{p'} + (M \Vert v \Vert_V)^\frac{\alpha'}{\alpha} \right) \max\left( \vert x \vert_{0,v}, \vert x \vert_{0,v}^{1-\frac{p}{p'}}, \vert x \vert_{0,v}^{1-\frac{\alpha'}{\alpha}} \right).
            \end{aligned}
        \end{equation*}
        Since the seminorms $\Vert \cdot \Vert_{p'-var,\alpha',v}$ and $\Vert \cdot \Vert_{0,v}$, $v \in V$, generate the topologies $\tau_{p'-var,\alpha'}$ and $\tau_0$ on $B$, respectively, we obtain $\tau_{p'-var,\alpha'} \subseteq \tau_0$ on $B$. This together with \ref{LemmaPVarTop1} and Lemma~\ref{LemmaHoelderTop}~\ref{LemmaHoelderTop1} shows that $\tau_\infty = \tau_0 = \tau_{p'-var,\alpha'}$ on the $\Vert \cdot \Vert_{p-var,\alpha}$-bounded subset $B \subseteq C^{p-var,\alpha}([0,T];Z)$.
    \end{proof}
    
    In the following, we now show that the embeddings $(C^{p-var,\alpha}([0,T];Z),\Vert \cdot \Vert_{p-var,\alpha}) \hookrightarrow (C^{p'-var,\alpha'}([0,T];Z),\tau_{p'-var,\alpha'})$ are compact, where $(p',\alpha') \in (p,\infty] \times [0,\alpha)$ with $p' \alpha' < 1$, where $C^{\infty-var,\alpha'}([0,T];Z) := C^{\alpha'}([0,T];Z)$ and $C^{p'-var,0}([0,T];Z) := C^{p'-var}([0,T];Z)$.

    \begin{theorem}
        \label{ThmPVarEmbedding}
        For $(p,\alpha) \in [1,\infty) \times (0,1)$ with $p \alpha < 1$ and $(p',\alpha') \in (p,\infty] \times [0,\alpha)$ with $p' \alpha' < 1$, the embedding $(C^{p-var,\alpha}([0,T];Z),\Vert \cdot \Vert_{p-var,\alpha}) \hookrightarrow (C^{p'-var,\alpha'}([0,T];Z),\tau_{p'-var,\alpha'})$ is compact.
    \end{theorem}
    \begin{proof}
        For some fixed $(p',\alpha') \in (p,\infty] \times [0,\alpha)$ with $p' \alpha' < 1$, we use \cite[Proposition~5.3]{friz10} (still holding true in our setting with seminorms) to conclude for every $x \in C^{p-var,\alpha}([0,T];Z)$ and $v \in V$ that
        \begin{equation*}
            \begin{aligned}
                \Vert x \Vert_{p'-var,\alpha',v} & \leq \vert \langle x(0), v \rangle_{Z \times V} \vert + \sup_{s,t \in [0,T] \atop s \neq t} \frac{\vert \langle x(s) - x(t), v \rangle_{Z \times V} \vert}{\vert s-t \vert^{\alpha'}} + \vert x \vert_{p'-var,v} \\
                & \leq \vert \langle x(0), v \rangle_{Z \times V} \vert + T^{\alpha-\alpha'} \sup_{s,t \in [0,T] \atop s \neq t} \frac{\vert \langle x(s) - x(t), v \rangle_{Z \times V} \vert}{\vert s-t \vert^\alpha} + \vert x \vert_{p-var,v} \\
                & \leq \left( 1 + T^{\alpha-\alpha'} \right) \Vert x \Vert_{p-var,\alpha,v} \\
                & \leq \left( 1 + T^{\alpha-\alpha'} \right) \Vert v \Vert_V \Vert x \Vert_{p-var,\alpha}.
            \end{aligned}
        \end{equation*}
        Since the seminorms $\Vert \cdot \Vert_{p'-var,\alpha',v}$, $v \in V$, generate the topology $\tau_{p'-var,\alpha'}$, we can apply \cite[Theorem~III.1.1]{schaefer99} to conclude that the embedding $(C^{p-var,\alpha}([0,T];Z),\Vert \cdot \Vert_{p-var,\alpha}) \hookrightarrow (C^{p'-var,\alpha'}([0,T];Z),\tau_{p'-var,\alpha'})$ is continuous.

        Next, we fix an $\Vert \cdot \Vert_{p-var,\alpha}$-bounded subset $B \subseteq \overline{B_R(0)} \subseteq C^{p-var,\alpha}([0,T];Z)$, for some $R > 0$. First, we show that $\overline{B_R(0)}$ is pointwise closed, i.e.~closed in the topology of pointwise convergence, where $Z$ is equipped with the weak-$*$-topology. Let $(x_\gamma)_\gamma \subseteq \overline{B_R(0)}$ be a net converging pointwise to some $x \in C^\alpha(S;Z)$. Then, for arbitrary $s,t \in S$ and $0 \leq t_0 < t_1 < \ldots < t_n \leq T$, there exists some $v_0,v_1,\ldots,v_{n+2} \in V$ with $\Vert v_i \Vert_V \leq 1$ such that $\Vert x(0) \Vert_Z = \langle x(0), v_0 \rangle_{Z \times V}$, $\Vert x(t_{i+1}) - x(t_i) \Vert_Z = \langle x(t_{i+1}) - x(t_i), v_{i+1} \rangle_{Z \times V}$ for all $i = 1,\ldots,n+1$, and $\Vert x(s) - x(t) \Vert_Z = \langle x(s) - x(t), v_{n+2} \rangle_{Z \times V}$, which implies that
        \begin{equation*}
            \begin{aligned}
                & \Vert x(0) \Vert_Z + \left( \sum_{i=1}^n \Vert x(t_{i+1}) - x(t_i) \Vert_Z^p \right)^\frac{1}{p} + \frac{\Vert x(s) - x(t) \Vert_Z}{d_S(s,t)^\alpha} \\
                & = \vert \langle x(0), v_0 \rangle_{Z \times V} \vert + \left( \sum_{i=1}^n \vert \langle x(t_{i+1}) - x(t_i), v_{i+1} \rangle_{Z \times V} \vert^p \right)^\frac{1}{p} + \frac{\vert \langle x(s) - x(t), v_{n+2} \rangle_{Z \times V} \vert}{d_S(s,t)^\alpha} \\
                & = \lim_\gamma \left( \langle x_\gamma(0), v_0 \rangle_{Z \times V} \!+\! \left( \sum_{i=1}^n \vert \langle x_\gamma(t_{i+1}) \!-\! x_\gamma(t_i), v_{i+1} \rangle_{Z \times V} \vert^p \right)^\frac{1}{p} \!+\!\frac{\langle x_\gamma(s) \!-\! x_\gamma(t), v_{n+2} \rangle_{Z \times V}}{d_S(s,t)^\alpha} \right) \\
                & \leq \lim_\gamma \left( \Vert x_\gamma(0) \Vert_Z + \left( \sum_{i=1}^n \Vert x_\gamma(t_{i+1}) - x_\gamma(t_i) \Vert_Z^p \right)^\frac{1}{p} + \frac{\Vert x_\gamma(s) - x_\gamma(t) \Vert_Z}{d_S(s,t)^\alpha} \right) \\
                & \leq \lim_\gamma \left( \Vert x_\gamma(0) \Vert_Z + \vert x_\gamma \vert_{p-var} + \vert x_\gamma \vert_\alpha \right) = \lim_\gamma \Vert x_\gamma \Vert_{p-var,\alpha} \leq R.
            \end{aligned}
        \end{equation*}
        By taking the supremum over $s,t \in S$ with $s \neq t$ and the supremum over $(t_i)_i \in \mathcal{D}([0,T])$, we have $\Vert x \Vert_{p-var,\alpha} \leq R$, showing that $\overline{B_R(0)}$ is pointwise closed. Moreover, $\overline{B_R(0)}$ is equicontinuous and pointwise bounded, thus by the Banach-Alaoglu theorem pointwise relatively compact with respect to the weak-$*$-topology of $Z$. Hence, we can apply the Arzel\`a-Ascoli theorem in \cite[Theorem~43.15]{willard04} to conclude that $\overline{B_R(0)}$ is relatively compact with respect to $\tau_\infty$. Thus, by using that the identity $(C^{p-var,\alpha}([0,T];Z),\tau_\infty) \hookrightarrow (C^{p-var,\alpha}([0,T];Z),\tau_{p'-var,\alpha'})$ is continuous (see Lemma~\ref{LemmaPVarTop}~\ref{LemmaPVarTop1}), it follows that $B \subseteq \overline{B_R(0)}$ is relatively compact with respect to $\tau_{p'-var,\alpha'}$.
    \end{proof}

    \begin{theorem}
        \label{ThmPVarPredual}
        For every $(p,\alpha) \in [1,\infty) \times (0,1)$ with $p \alpha < 1$, the Banach space $(C^{p-var,\alpha}([0,T];Z),\Vert \cdot \Vert_{p-var,\alpha})$ is a dual Banach space. Moreover, $\tau_\infty = \tau_{p'-var} = \tau_{p'-var,\alpha'}$ on every $\Vert \cdot \Vert_{p-var,\alpha}$-bounded subset of $C^{p-var}([0,T];Z)$, for all $(p',\alpha') \in (p,\infty] \times [0,\alpha)$ with $p' \alpha' < 1$.
    \end{theorem}
    \begin{proof}
        Let $(V,\Vert \cdot \Vert_V)$ be a predual for $(Z,\Vert \cdot \Vert_Z)$. We want to apply Theorem~\ref{ThmBanachPredual} with
        \begin{equation*}
            E_0 := \left\lbrace C^{p-var,\alpha}([0,T];Z) \ni x \mapsto \langle x(s), v \rangle_{Z \times V} \in \mathbb{R}: \,
            \begin{matrix}
                s \in [0,T] \\
                v \in V
            \end{matrix}
            \right\rbrace \subseteq C^{p-var,\alpha}([0,T];Z)^*.
        \end{equation*}
     	Then, we follow the proof of Theorem~\ref{ThmHoelderPredual} to conclude that the closure of $\linspan(E_0)$ with respect to $\Vert \cdot \Vert_{C^{p-var,\alpha}([0,T];Z)^*}$ is a predual for $C^{p-var,\alpha}([0,T];Z)$. Moreover, by using the same arguments as in the proof of Theorem~\ref{ThmHoelderPredual}, we conclude that the above mentioned topologies coincide on each $\Vert \cdot \Vert_{p-var,\alpha}$-bounded subset of $C^{p-var,\alpha}([0,T];Z)$.
    \end{proof}
    
    For the weight function $C^{p-var,\alpha}([0,T];Z) \ni x \mapsto \psi(x) := \eta\left( \Vert x \Vert_{p-var,\alpha} \right) \in (0,\infty)$, with continuous and non-decreasing function $\eta: [0,\infty) \rightarrow (0,\infty)$, we consider the weighted space $(C^{p-var,\alpha}([0,T];Z),\psi)$ either equipped with the weak-$*$-topology as in Example~\ref{ExWeightedSpaces}~\ref{ExWeightedSpacePVarPredual} or with the $C^{p'-var,\alpha'}$-topology as in Example~\ref{ExWeightedSpaces}~\ref{ExWeightedSpacePVarWeaker}. Then, by using that these topologies coincide on the pre-image $K_R = \psi^{-1}((0,R])$ (see Theorem~\ref{ThmPVarPredual}), $\mathcal{B}_\psi(C^{p-var,\alpha}([0,T];Z))$ is the same for both choices of topologies on $C^{p-var,\alpha}([0,T];Z)$.

    \begin{proposition}
        \label{PropPVarBPsi}
        Let $C^{p-var,\alpha}_\tau([0,T];Z)$ denote $C^{p-var,\alpha}([0,T];Z)$ equipped with the topology $\tau \in \mathcal{T} := \lbrace \tau_{w^*}, \tau_\infty \rbrace \cup \left\lbrace \tau_{p'-var,\alpha'}: (p',\alpha') \in (p,\infty] \times [0,\alpha), \, p' \alpha' < 1 \right\rbrace$. Moreover, for some continuous non-decreasing unbounded $\eta: [0,\infty) \rightarrow (0,\infty)$, let $C^{p-var,\alpha}_\tau([0,T];Z) \ni x \mapsto \psi(x) := \eta\left( \Vert x \Vert_{p-var,\alpha} \right) \in (0,\infty)$ be an admissible weight function, for all $\tau \in \mathcal{T}$. Then, $\mathcal{B}_\psi(C^{p-var,\alpha}_\tau([0,T];Z)) = \mathcal{B}_\psi(C^{p-var,\alpha}_{\tau'}([0,T];Z))$, for all $\tau,\tau' \in \mathcal{T}$.
    \end{proposition}
    \begin{proof}
        The proof follows along the lines of the proof for Proposition~\ref{PropHoelderBPsi}.
    \end{proof}
    
    \begin{remark}
        \label{RemPVar0}
       For the closed vector subspace $C^{p-var,\alpha}_0([0,T];Z) \subseteq C^{p-var,\alpha}([0,T];Z)$ of $\alpha$-H\"older continuous paths $x: [0,T] \rightarrow Z$ with finite $p$-variation preserving the origin, we obtain the analogous results as in Remark~\ref{RemHoelder0}.
    \end{remark}

	\section{Proof of Proposition~\ref{PropAct}}
	\label{sec:proof}
	
	\begin{proof}[Proof of Proposition~\ref{PropAct}]
	    First, we show that $\mathcal{NN}^\rho_{\mathbb{R},\mathbb{R}} \subseteq \mathcal{B}_{\psi_1}(\mathbb{R})$. Since $\mathcal{NN}^\rho_{\mathbb{R},\mathbb{R}}$ is defined as the linear span of functions of the form $\mathbb{R} \ni z \mapsto w \rho(az + b) \in \mathbb{R}$, it suffices to show that $\left( z \mapsto w \rho(az + b) \right) \in \mathcal{B}_{\psi_1}(\mathbb{R})$, for all $a \in \mathbb{N}_0$ and $b,w \in \mathbb{R}$. Moreover, we can assume without loss of generality that $w = 1$. Now, we fix some $a \in \mathbb{N}_0$, $b \in \mathbb{R}$, and $\varepsilon > 0$. Then, by using that $\lim_{z \rightarrow \pm \infty} \frac{\vert \rho(az+b) \vert}{\psi_1(z)} = 0$, there exists some $\widetilde{R} > 0$ such that $\sup_{\vert z \vert > \widetilde{R}} \frac{\vert \rho(az+b) \vert}{\psi_1(z)} \leq \varepsilon/2$. Moreover, we define the constant $C := \sup_{\vert z \vert \leq \widetilde{R}} \vert \rho(az+b) \vert \geq 0$ which is finite as $\rho \in C^0(\mathbb{R})$ is continuous. Hence, by choosing $R \geq 2C/\varepsilon$, it follows that
	    \begin{equation*}
	    	\begin{aligned}
	    		\sup_{z \in \mathbb{R} \setminus K_{1,R}} \frac{\vert \rho(az+b) \vert}{\psi_1(z)} & \leq \sup_{z \in \mathbb{R} \setminus K_{1,R} \atop \vert z \vert \leq \widetilde{R}} \frac{\vert \rho(az+b) \vert}{\psi_1(z)} + \sup_{z \in \mathbb{R} \setminus K_{1,R} \atop \vert z \vert > \widetilde{R}} \frac{\vert \rho(az+b) \vert}{\psi_1(z)} \\
	    		& \leq \frac{\sup_{\vert z \vert \leq \widetilde{R}} \vert \rho(az+b) \vert}{R} + \sup_{z \in \mathbb{R} \atop \vert z \vert > \widetilde{R}} \frac{\vert \rho(az+b) \vert}{\psi_1(z)} \\
	    		& \leq \frac{C}{R} + \frac{\varepsilon}{2} \leq \varepsilon,
	    	\end{aligned}
	    \end{equation*}
	    where $K_{1,R} := \psi_1^{-1}((0,R])$. Since $\varepsilon > 0$ was chosen arbitrarily and $K_{1,R} \ni z \mapsto \rho(az+b) \in \mathbb{R}$ is continuous, for all $R > 0$, Lemma~\ref{LemmaBpsiEquivChar}~\ref{LemmaBpsiEquivChar2} shows that $\left( z \mapsto \rho(az + b) \right) \in \mathcal{B}_{\psi_1}(\mathbb{R})$.
	    
	    Now, we show that each of the conditions \ref{A1}-\ref{A3} implies that $\mathcal{NN}^\rho_{\mathbb{R},\mathbb{R}}$ is dense in $\mathcal{B}_{\psi_1}(\mathbb{R})$:
	    \begin{enumerate}
	        \item[\ref{A1}] We adapt the original proof of George Cybenko in \cite[Theorem~1]{cybenko89} to this weighted setting, where we use the Riesz representation theorem in \cite[Theorem~2.4]{doersek10} for $\mathcal{B}_{\psi_1}(\mathbb{R})$. To this end, let us assume by contradiction that $\rho \in C^0(\mathbb{R})$ is not $\psi_1$-activating, which means that $\mathcal{NN}^\rho_{\mathbb{R},\mathbb{R}}$ is not dense in $\mathcal{B}_{\psi_1}(\mathbb{R})$. Then, by the Hahn-Banach theorem, there exists a non-zero continuous linear functional $l: \mathcal{B}_{\psi_1}(\mathbb{R}) \rightarrow \mathbb{R}$ such that $l(\varphi) = 0$, for all $\varphi \in \mathcal{NN}^\rho_{\mathbb{R},\mathbb{R}}$. Moreover, by using the Riesz representation theorem in \cite[Theorem~2.4]{doersek10} there exists a signed Radon measure $\mu \in \mathcal{M}_{\psi_1}(\mathbb{R})$ such that $l(f) = \int_{\mathbb{R}} f(z) \mu(dz)$, for all $f \in \mathcal{B}_{\psi_1}(\mathbb{R})$, which implies that $\int_{\mathbb{R}} \rho(a z + b) \mu(dz) = 0$, for all $a \in \mathbb{N}_0$ and $b \in \mathbb{R}$. However, since $\rho \in C^0(\mathbb{R})$ is $\psi_1$-discriminatory, it follows that $\mu = 0$, which contradicts the assumption that $l: \mathcal{B}_{\psi_1}(\mathbb{R}) \rightarrow \mathbb{R}$ is non-zero. Hence, it follows that $\mathcal{NN}^\rho_{\mathbb{R},\mathbb{R}}$ is dense in $\mathcal{B}_{\psi_1}(\mathbb{R})$.
	        
	        \item[\ref{A2}] We adapt the original proof of George Cybenko in \cite[Lemma~1]{cybenko89} to this weighted setting to show that $\rho \in C^0(\mathbb{R})$ is $\psi_1$-discriminatory. Then, it follows from \ref{A1} that $\mathcal{NN}^\rho_{\mathbb{R},\mathbb{R}}$ is dense in $\mathcal{B}_{\psi_1}(\mathbb{R})$.
         
            For this purpose, let $\mu \in \mathcal{M}_{\psi_1}(\mathbb{R})$ be a signed Radon measure satisfying for every $a \in \mathbb{N}_0$ and $b \in \mathbb{R}$ that
            \begin{equation}
                \label{EqPropActProof1}
                \int_{\mathbb{R}} \rho(a z + b) \mu(dz) = 0.
            \end{equation}
            Now, we choose $a = \lambda$ and $b = -\lambda \theta + \xi$, for some $\lambda \in \mathbb{N}_0$ and $\theta, \xi \in \mathbb{R}$. Then, by using that $\rho \in C^0(\mathbb{R})$ is sigmoidal, it follows for every $\theta, \xi, z \in \mathbb{R}$ that
    		\begin{equation*}
    			\rho\left( \lambda \left( z - \theta \right) + \xi \right) \quad \overset{\lambda \rightarrow \infty}{\underset{\lambda \in \mathbb{N}_0}{\longrightarrow}} \quad
    			\begin{cases}
    				1 & \text{if } z \in (\theta,\infty), \\
    				\rho(\xi) & \text{if } z = \theta, \\
    				0 & \text{if } z \in (-\infty,\theta).
    			\end{cases}
    		\end{equation*}
    		Moreover, we observe that $\vert \rho(\lambda (z-\theta) + \xi) \vert \leq C_\rho$ for any $\lambda \in \mathbb{N}_0$ and $\theta, \xi \in \mathbb{R}$, where $C_\rho = \sup_{z \in \mathbb{R}} \vert \rho(z) \vert < \infty$ (as $\rho \in C^0(\mathbb{R})$ is continuous and sigmoidal), and that $\mu \in \mathcal{M}_{\psi_1}(\mathbb{R})$ is a finite signed measure as $\int_{\mathbb{R}} \vert \mu \vert(dz) \leq M \int_{\mathbb{R}} \psi_1(z) \vert \mu \vert(dz) < \infty$, where $M = \left( \inf_{x \in X} \psi_1(x) \right)^{-1} < \infty$. Hence, we can apply the dominated convergence theorem to conclude from \eqref{EqPropActProof1} that
    		\begin{equation}
                \label{EqPropActProof2}
                \begin{aligned}
                	0 = \int_{\mathbb{R}} \rho\left( \lambda \left( z - \theta \right) + \xi \right) \mu(dz) \quad \overset{\lambda \rightarrow \infty}{\underset{\lambda \in \mathbb{N}_0}{\longrightarrow}} \quad & \rho(\xi) \int_{\mathbb{R}} \mathds{1}_{\lbrace \theta \rbrace}(z) \mu(dz) + \int_{\mathbb{R}} \mathds{1}_{(\theta,\infty)}(z) \mu(dz) \\
                	= \, & \rho(\xi) \mu\left( \left\lbrace \theta \right\rbrace \right) + \mu\left( (\theta,\infty) \right).
                \end{aligned}
    		\end{equation}
    		Hereby, we use that $\lim_{\xi \rightarrow -\infty} \rho(\xi) = 0$ (as $\rho \in C^0(\mathbb{R})$ is sigmoidal) to conclude that the first term on the right-hand side of \eqref{EqPropActProof2} can be omitted. Now, we decompose $\mu \in \mathcal{M}_{\psi_1}(\mathbb{R})$ into its positive and negative part $\mu_+, \mu_- \in \mathcal{M}_{\psi_1}(\mathbb{R})$ such that $\mu = \mu_+ - \mu_-$. Then, \eqref{EqPropActProof2} (without the first term on the right-hand side) implies that
    		\begin{equation*}
    			\mu_+\left( (\theta,\infty) \right) - \mu_-\left( (\theta,\infty) \right) = \mu\left( (\theta,\infty) \right) = 0,
    		\end{equation*}
    		and thus $\mu_+\left( (\theta,\infty) \right) = \mu_-\left( (\theta,\infty) \right)$. Moreover, by letting $\theta \rightarrow -\infty$, we obtain $\mu_+(\mathbb{R}) = \mu_-(\mathbb{R})$. Hence, we conclude for every $\theta \in \mathbb{R}$ that
    		\begin{equation*}
    			\mu_-\left( (-\infty,\theta] \right) = \mu_-(\mathbb{R}) - \mu_-\left( (\theta,\infty) \right) = \mu_+(\mathbb{R}) - \mu_+\left( (\theta,\infty) \right) = \mu_+\left( (-\infty,\theta] \right).
    		\end{equation*}
    		Thus, by using that $\mathcal{B}(\mathbb{R})$ is generated by the family $\left\lbrace (-\infty,b]: b \in \mathbb{R} \right\rbrace$, which is closed under finite intersections, it follows from \cite[Theorem~5.4]{bauer01} that $\mu_- = \mu_+$. This shows that $\mu = 0 \in \mathcal{M}_{\psi_1}(\mathbb{R})$ and that $\rho \in C^0(\mathbb{R})$ is $\psi_1$-discriminatory.

            \item[\ref{A3}] We generalize Norbert Wiener's Tauberian theorem in \cite{wiener32} to the current setting and we use the argument of Jacob Koorevar's distributional extension in \cite{korevaar65}. In short, \cite{korevaar65} shows that $f * g =0 $, where $f \in L^1(\mathbb{R})$ with nowhere vanishing Fourier transform and some $g \in L^\infty(\mathbb{R})$, leads to $g=0$. We prove with a similar argument that $f * g = 0$, where $g$ has a Fourier transform with $0$ in the interior of its support and $f \in L^1(\mathbb{R})$, leads to local assertions on the Fourier transform of $f$. Then, as we will show, it follows from \ref{A1} that $\mathcal{NN}^\rho_{\mathbb{R},\mathbb{R}}$ is dense in $\mathcal{B}_{\psi_1}(\mathbb{R})$.

            In the following, we argue by contradiction and assume that $\rho \in C^0(\mathbb{R})$ is not $\psi_1$-discriminatory. Then, as in the proof of \ref{A1}, there exists a non-zero finite signed Radon measure $\mu \in \mathcal{M}_{\psi_1}(\mathbb{R})$ such that for every fixed $a \in \mathbb{N}_0$ and $b \in \mathbb{R}$ we have
            \begin{equation}
                \label{EqPropActProof3}
                \int_{\mathbb{R}} \rho(az+ab) \mu(dz) = 0.
            \end{equation}
            Moreover, we identify the one-point compactification $\overline{\mathbb{R}} = \mathbb{R} \cup \lbrace \infty \rbrace$ with the unit circle $S^1 := \lbrace (u,v) \in \mathbb{R}^2: u^2 + v^2 = 1 \rbrace$ via the mapping
            \begin{equation*}
                S^1 \ni (u,v) \quad \mapsto \quad
                \begin{cases}
                    \frac{u}{1-v} & \text{if } v \neq 1, \\
                    \infty & \text{if } v = 1.
                \end{cases}
            \end{equation*}
            Hence, by using that $\rho \in C^0(\mathbb{R})$, that $\psi_1: \mathbb{R} \rightarrow (0,\infty)$ is lower semicontinuous (see Remark~\ref{RemWeight}~\ref{RemWeight1}), and that $\lim_{(z,b) \rightarrow (\pm \infty,b_0)} \frac{\vert \rho(az+ab) \vert}{\psi_1(z)} = 0$ for any $b_0 \in \mathbb{R}$, the function
            \begin{equation*}
                S^1 \times \mathbb{R} \ni (u,v,b) \quad \mapsto \quad g_a(u,v,b) :=
                \begin{cases}
                    \frac{\left\vert \rho\left( a \frac{u}{1-v} + ab \right) \right\vert}{\psi_1\left( \frac{u}{1-v} \right)} & \text{if } v \neq 1, \\
                    0 & \text{if } v = 1,
                \end{cases}
            \end{equation*}
            is upper semicontinuous. Thus, for every compact subset $K \subsetneq \mathbb{R}$, we obtain that
            \begin{equation*}
                \sup_{(z,b) \in \mathbb{R} \times K} \frac{\vert \rho(az+ab) \vert}{\psi_1(z)} = \sup_{(u,v,b) \in S^1 \times K} g_a(u,v,b) < \infty.
            \end{equation*}
            Hence, by using this together with $\int_{\mathbb{R}} \psi_1(z) \vert \mu \vert(dz) < \infty$, we can apply the dominated convergence theorem to conclude that $\mathbb{R} \ni b \mapsto \int_{\mathbb{R}} \vert \rho(az+ab) \vert \vert \mu \vert(dz) \in \mathbb{R}$ is continuous. Therefore, for every compactly supported $h \in C^\infty_c(\mathbb{R};\mathbb{C})$, it follows that
             \begin{equation}
                \label{EqPropActProof4}
                \begin{aligned}
                    & \int_{\mathbb{R}} \int_{\mathbb{R}} \vert h(s-t-z) \vert \vert \rho(as) \vert \, ds \vert \mu \vert(dz) = \int_{\mathbb{R}} \int_{\mathbb{R}} \vert h(b-t) \vert \vert \rho(az+ab) \vert \, db \vert \mu \vert(dz) \\
                    & \quad\quad = \int_{\mathbb{R}} \int_{\mathbb{R}} \vert h(b-t) \vert \vert \rho(az+ab) \vert \vert \mu \vert(dz) db < \infty,
                \end{aligned}
            \end{equation}
            where we have used the substitution $b \mapsto s-z$.

            Next, by using that $\mu \in \mathcal{M}_{\psi_1}(\mathbb{R})$ is non-zero, there exists some $h \in C^\infty_c(\mathbb{R};\mathbb{C})$ such that $\big( s \mapsto f(s) := (h * \mu)(-s) := \int_{\mathbb{R}} h(-s-z) \mu(dz) \big) \in L^1(\mathbb{R};\mathbb{C})$ is also non-zero. Indeed, since every $h \in C^\infty_c(\mathbb{R};\mathbb{C})$ is uniformly bounded and $\mu \in \mathcal{M}_{\psi_1}(\mathbb{R})$ is finite, the function $\mathbb{R} \ni s \mapsto f(s) := (h * \mu)(-s) := \int_{\mathbb{R}} h(-s-z) \mu(dz) \in \mathbb{C}$ is by the dominated convergence theorem continuous, whence we actually only need one fixed $s \in \mathbb{R}$ with $f(s) \neq 0$. Now, by using that $f \in L^1(\mathbb{R};\mathbb{C})$ is non-zero, the injectivity of the Fourier transform implies that $\widehat{f}: \mathbb{R} \rightarrow \mathbb{C}$ is non-zero, too, i.e.~there exists some $\zeta \in \mathbb{R}$ such that $\widehat{f}(\zeta) \neq 0$. Thus, $\left( s \mapsto f_\zeta(s) := f(s) e^{-\mathbf{i} \zeta s} \right) \in L^1(\mathbb{R};\mathbb{C})$ satisfies $\widehat{f_\zeta}(0) = \widehat{f}(\zeta) \neq 0$. Moreover, for every $a \in \mathbb{N}$, we define $\left( z \mapsto \rho_{a,\zeta}(z) := \rho(az) e^{-\mathbf{i} \zeta z} \right) \in C^0(\mathbb{R};\mathbb{C})$. Then, by using the Fubini theorem (justified by \eqref{EqPropActProof4}), the substitution $b \mapsto s-z$, again Fubini's theorem with signed measures (justified by \eqref{EqPropActProof4}), and the identity \eqref{EqPropActProof3}, we conclude for every $a \in \mathbb{N}$ and $t \in \mathbb{R}$ that
            \begin{equation}
                \label{EqPropActProof5}
                \begin{aligned}
                    \left( f_\zeta * \rho_{a,\zeta} \right)(t) & = \int_{\mathbb{R}} f_\zeta(t-s) \rho_{a,\zeta}(s) ds \\
                    & = \int_{\mathbb{R}} f(t-s) e^{-\mathbf{i} \zeta (t-s)} \rho(as) e^{-\mathbf{i} \zeta s} ds \\
                    & = e^{-\mathbf{i} \zeta t} \int_{\mathbb{R}} (h * \mu)(s-t) \rho(as) ds \\
                    & = e^{-\mathbf{i} \zeta t} \int_{\mathbb{R}} \left( \int_{\mathbb{R}} h(s-t-z) \mu(dz) \right) \rho(as) ds \\
                    & = e^{-\mathbf{i} \zeta t} \int_{\mathbb{R}} \left( \int_{\mathbb{R}} h(s-t-z) \rho(as) ds \right) \mu(dz) \\   
                    & = e^{-\mathbf{i} \zeta t} \int_{\mathbb{R}} \left( \int_{\mathbb{R}}h(b-t) \rho(az+ab) db \right) \mu(dz) \\
                    & = e^{-\mathbf{i} \zeta t} \int_{\mathbb{R}} h(b-t) \left( \int_{\mathbb{R}} \rho(az+ab) \mu(dz) \right) db = 0.
                \end{aligned}
            \end{equation}            
            Let $\phi \in \mathcal{S}(\mathbb{R};\mathbb{C})$ be a Schwartz function with $\widehat{\phi}(\xi) = 1$, for all $\xi \in [-1,1]$, and $\widehat{\phi}(\xi) = 0$, for all $\xi \in \mathbb{R} \setminus [-2,2]$. From this, we define for every $n \in \mathbb{N}$ the Schwartz function $\left( s \mapsto \phi_n(s) := \frac{1}{n} \phi\left( \frac{s}{n} \right) \right) \in \mathcal{S}(\mathbb{R};\mathbb{C})$. Then, by following the proof of \cite[Theorem~A]{korevaar65}, there exists some $n \in \mathbb{N}$ and $w \in L^1(\mathbb{R};\mathbb{C})$ such that $w * f_\zeta = \phi_{2n} \in \mathcal{S}(\mathbb{R};\mathbb{C})$. Hence, by using \eqref{EqPropActProof4}+\eqref{EqPropActProof5}, we conclude for every $a \in \mathbb{N}$ and $s \in \mathbb{R}$ that
            \begin{equation*}
                (\phi_{2n} * T_{\rho_{a,\zeta}})(s) = (w * f_\zeta * \rho_{a,\zeta})(s) = \int_{\mathbb{R}} w(s-t) \left( f_\zeta * \rho_{a,\zeta} \right)(t) dt = 0,
            \end{equation*}
            which implies by using \cite[Equation~9.32]{folland92} that $\widehat{\phi_{2n}} \widehat{T_{\rho_{a,\zeta}}} = 0 \in \mathcal{S}'(\mathbb{R};\mathbb{C})$, for all $a \in \mathbb{N}$. Since $\widehat{\phi_{2n}}(\xi) = \widehat{\phi}(2n \xi) = 1$ for all $\xi \in [-\frac{1}{2n},\frac{1}{2n}]$, it follows for every $a \in \mathbb{N}$ that $\widehat{T_{\rho_{a,\zeta}}}$ vanishes on $(-\frac{1}{2n},\frac{1}{2n})$, i.e.~$\widehat{T_{\rho_{a,\zeta}}}(g) = 0$ for all $g \in C^\infty_c((-\frac{1}{2n},\frac{1}{2n}))$.
            
            Finally, we define for every fixed $a \in \mathbb{N}$ and $g \in C^\infty_c((\frac{\zeta-1/(2n)}{a},\frac{\zeta+1/(2n)}{a}))$, the function $\left( z \mapsto g_{a,\zeta}(z) := g\left( \frac{z+\zeta}{a} \right) \right) \in C^\infty_c((-\frac{1}{2n},\frac{1}{2n}))$. Then, the previous step shows that
            \begin{equation*}
                \begin{aligned}
                    \widehat{T_\rho}(g) & = T_\rho(\widehat{g}) = \int_{\mathbb{R}} \rho(s) \widehat{g}(s) ds = a \int_{\mathbb{R}} \rho(az) \widehat{g}(az) dz = \int_{\mathbb{R}} \rho_{a,\zeta}(z) e^{\mathbf{i} \zeta z} \widehat{g\left( \frac{\cdot}{a} \right)}(z) dz \\
                    & = \int_{\mathbb{R}} \rho_{a,\zeta}(z) \widehat{g_{a,\zeta}}(z) dz = T_{\rho_{a,\zeta}}(\widehat{g_{a,\zeta}}) = \widehat{T_{\rho_{a,\zeta}}}(g_{a,\zeta}) = 0.
                \end{aligned}
            \end{equation*}
            Since $a \in \mathbb{N}$ and $g \in C^\infty_c((\frac{\zeta-1/(2n)}{a},\frac{\zeta+1/(2n)}{a}))$ were chosen arbitrarily, it follows that $\widehat{T_\rho} \in \mathcal{S}'(\mathbb{R};\mathbb{C})$ vanishes on the set $\bigcup_{a \in \mathbb{N}} (\frac{\zeta-1/(2n)}{a},\frac{\zeta+1/(2n)}{a})$, whose closure contains the point $0 \in \mathbb{R}$. This, however, contradicts the assumption that $\widehat{T_\rho} \in \mathcal{S}'(\mathbb{R};\mathbb{C})$ has a support with $0 \in \mathbb{R}$ as inner point, whence $\rho \in C^0(\mathbb{R})$ is $\psi_1$-discriminatory.
	    \end{enumerate}
	\end{proof}
	
	\bibliographystyle{abbrv}
	\bibliography{mybib}
\end{document}